\documentclass{article}
\usepackage{arxiv}
\usepackage[utf8]{inputenc} 
\usepackage[T1]{fontenc}    
\usepackage{hyperref}       
\usepackage{url}            
\usepackage{booktabs}       
\usepackage{amsfonts}       
\usepackage{nicefrac}       
\usepackage{microtype}      
\usepackage{lipsum}
\usepackage{graphicx}
\newif\ifshowrevisions
\showrevisionsfalse
\usepackage{xcolor}
\newcommand{\old}[1]{}
\newcommand{\revised}[1]{\ifshowrevisions{\color{blue}#1}\else#1\fi}

\usepackage{subcaption}
\usepackage{amsthm,amsmath,amsfonts,amssymb}
\usepackage[numbers]{natbib}
\usepackage{makecell}
\theoremstyle{plain}
\newtheorem{axiom}{Axiom}
\newtheorem{claim}[axiom]{Claim}
\newtheorem{theorem}{Theorem}[section]
\newtheorem{lemma}[theorem]{Lemma}

\newtheorem{proposition}[theorem]{Proposition}

\theoremstyle{remark}
\newtheorem{definition}[theorem]{Definition}

\newtheorem{assumption}{Assumption}
\newtheorem{remark}[theorem]{Remark}

\newcommand*{\bm}{\boldsymbol}
\newcommand*{\dx}{\mathrm{d}}

\newcommand*{\NTK}{\mathtt{NT}}

\newcommand*{\inner}{\mathtt{in}}

\newcommand{\bbE}{\mathbb{E}}

\newcommand{\bbR}{\mathbb{R}}
\newcommand{\bbS}{\mathbb{S}}

\newcommand{\calB}{\mathcal{B}}

\newcommand{\calD}{\mathcal{D}}
\newcommand{\calE}{\mathcal{E}}
\newcommand{\calF}{\mathcal{F}}

\newcommand{\calH}{\mathcal{H}}

\newcommand{\calX}{\mathcal{X}}

\newcommand{\calZ}{\mathcal{Z}}

\title{Optimal Rate of Kernel Regression\\
in Large Dimensions}
\author{
  Weihao Lu, Haobo Zhang\thanks{Co-first author}, Yicheng Li, Manyun Xu \\
  Center for Statistical Science, Department of Industrial Engineering, Tsinghua University\\
  100084, Beijing, China\\
  \texttt{\{luwh19, zhang-hb21, liyc22, xumy20\}@mails.tsinghua.edu.cn} \\
  \And
  Qian Lin\thanks{Corresponding author} \\
  Center for Statistical Science, Department of Industrial Engineering, Tsinghua University\\
  100084, Beijing, China\\
  \texttt{qianlin@tsinghua.edu.cn}
}
\begin{document}
\maketitle
\begin{abstract}
\old{We study data-dependent early stopping for kernel regression with high-dimensional data (where the sample size $n$ is { polynomially dependent} on the dimension $d$ of the samples, i.e., $n\asymp d^{\gamma}$ for some $\gamma >0$ ).
 We develop general tools for upper and minimax lower bounds { for} kernel regression for high-dimensional data through the Mendelson complexity $\varepsilon_{n}^{2}$ and the metric entropy $\bar{\varepsilon}_{n}^{2}$  respectively. 
When the target function belongs to the RKHS associated with an inner-product kernel defined on $\bbS^{d}$, we derive an $n^{-1/2}$ lower bound for expected minimax risk and a corresponding high-probability upper bound for early-stopped kernel regression when $n\asymp d^{\gamma}$ and $\gamma=2,4,6,\ldots$. 
For all $\gamma>0$, we identify the polynomial exponents of the expected minimax lower bounds and high-probability upper bounds, with the stated logarithmic or arbitrarily small polynomial qualifications, and describe the resulting {\it multiple descent behavior} and {\it periodic plateau behavior}.
}\revised{We study early-stopped kernel regression in the high-dimensional regime $n\asymp d^\gamma$ for fixed $\gamma>0$. General spectral-complexity and metric-entropy tools yield high-probability prediction-error upper bounds and expected minimax lower bounds. For smooth inner-product kernels on the sphere with uniform design and targets in the unit RKHS ball, we establish matching $n^{-1/2}$ bounds at positive even values of $\gamma$. For other values of $\gamma$, our bounds identify matching polynomial exponents, subject to the logarithmic and arbitrarily small polynomial qualifications in the theorem statements. The resulting exponents exhibit multiple descent and periodic plateau behavior, reflecting the dimension-dependent spherical-harmonic spectrum.
}
\end{abstract}
\keywords{kernel regression \and high-dimensional statistics \and minimax rates}
\section{Introduction}\label{sec:intro}
Suppose we have observed $n$ i.i.d. samples $(X_{i},Y_{i})$ from a joint distribution $(X,Y)$ supported on $\mathbb{R}^{d+1}\times \mathbb{R}$. The regression problem, one of the most fundamental problems in statistics,  aims to find a function $\hat{f}_{n}$ based on these samples such that the {\it excess risk},
\begin{align*}
\left\|\hat{f}_{n} - f_{\star}\right\|_{L^2}^2
=\mathbb{E}_{X}\left[\left(f_{\star}(X)-\hat{f}_{n}(X)\right)^{2}\right],
\end{align*}
is small, where $f_{\star}(x)=\mathbb{E}[Y\vert x]$ is the {\it regression function}.
Classical nonparametric regression methods include polynomial splines \cite{Stone_Polynomial_1994}, local polynomials \cite{cleveland1979robust, stone1977consistent}, kernel methods \cite{Caponnetto2006OptimalRF, caponnetto2007optimal, caponnetto2010cross}, etc.
When the data dimension $d$ is small, these methods produce reasonable results;
however, when $d$ is relatively large, the excess risk can converge extremely slowly. 
{
However, while assumptions such as low intrinsic dimensionality (that data falls into a subspace with dimension far smaller than $d$) and feature sparsity can improve theoretical performance in nonparametric regression \cite{Laurent_Estimating_2015, Fukunaga_Algorithm_1971}, these structural assumptions do not by themselves explain the empirical success of modern high-dimensional learning systems.
}
On the other hand, neural network methods have achieved tremendous success in many large-dimensional problems, such as computer vision \cite{he2016deep, krizhevsky2017imagenet} and natural language processing \cite{Devlin_BERT_2019}.
For example, the ILSVRC competition \cite{ILSVRC15} has a dataset of 1.2 million samples with a dimensionality of approximately 200K, while
 the pretraining dataset of the well-known language representation model, Bidirectional Encoder Representations from Transformers (BERT) \cite{Devlin_BERT_2019}, consists of 13 million samples with a dimensionality of approximately 400K. 
Several lines of research have sought to explain the superior performance of neural networks on high-dimensional data
from various aspects. 
However, the highly nonlinear dynamics induced by gradient descent or gradient flow in neural-network training make a direct analysis notoriously difficult \cite{goodfellow2016deep, lecun2015deep, pascanu2013difficulty}.
When the width of a neural network is sufficiently large, the training process falls into the `lazy regime', i.e., its parameters/weights stay in a small neighborhood of their initial position during the training process \cite{arora2019fine, Du_gradient_2019_b, Du_gradient_2019_a, li2018learning}.
Since \cite{Jacot_NTK_2018} observed that the time-varying neural network kernel (NNK) converges to a time-invariant neural tangent kernel (NTK) { pointwise} as the width $m$ of neural networks $\rightarrow\infty$, early-stopped NTK regression has been widely viewed as a useful surrogate for studying the generalization behavior of neural networks in the `lazy regime' \cite{Arora_on_2019, Hu_Regularization_2021, Namjoon_Non_2022}.
Recently, a sequence of works \cite{jianfa2022generalization, li2023statistical} further showed that the NNK uniformly converges to the NTK as the width $m\rightarrow\infty$ thereby providing a rigorous justification for this connection. 
Thus, understanding the generalization behavior of NTK regression in high dimensions will help us understand the superior performance of (wide) neural networks.
Kernel regression (or regression over an RKHS), as a classical topic,  has been studied since the 1990s. 
Most work 
imposes the polynomial eigenvalue decay assumption over a kernel $K$ (i.e.,  there exist constants $0<\mathfrak{c} \leq \mathfrak{C}<\infty$, such that the eigenvalues of the kernel satisfy $\mathfrak{c} j^{-\beta} \leq \lambda_j \leq \mathfrak{C} j^{-\beta}$ for some constant $\beta>1$) and assume that the target function $f_{\star}$ belongs to the RKHS associated with $K$ \cite{Caponnetto2006OptimalRF, caponnetto2007optimal, Lin_Optimal_2020, raskutti2014early}. They then showed that the minimax rate of the excess risk of regression over the corresponding RKHS is lower bounded by $n^{-{\beta}/({\beta+1})}$ and that some kernel methods (e.g., the kernel ridge regression and the early-stopped kernel regression) can produce estimators achieving this optimal rate. 
Thus, verifying that if an NTK satisfies the polynomial eigenvalue decay assumption and determining the 
eigenvalue decay rate 
of it becomes a natural strategy to discuss the generalization ability of the NTK  ( or equivalently, wide neural networks ) regression.
When the NTK is defined on sphere $\mathbb S^{d}$, it is an inner-product kernel.
Hence, the eigenvalues of NTK can be obtained through a detailed calculation with the help of the spherical harmonic polynomials.
It is shown in \cite{Bietti_deep_2021, Bietti_on_2019} that when $d$ is fixed, the eigenvalues of the NTK defined on $\mathbb S^{d}$ polynomially decayed at rate $({d+1})/{d}$. 
When the domain is other than a sphere, \cite{jianfa2022generalization, li2023statistical} further illustrated that the eigenvalues decay rate of NTK on any bounded open set in $\bbR^{d}$ is still $({d+1})/{d}$.
Some works then claimed that the optimal tuned neural network on $\bbS^{d}$ or on any bounded open set in $\bbR^{d}$ can achieve the optimal rate $n^{-({d+1})/({2d+1})}$ \cite{Hu_Regularization_2021, jianfa2022generalization, 
li2023statistical}.
When dimension $d$ is large, much less is known about the convergence rate of the excess risk of kernel methods. 
There are several works devoted to the high-dimensional setting where $n \asymp d$. For example, 
motivated by the linear approximation of kernel matrices in high-dimensional data proposed by \cite{Karoui_spectrum_2010},  \cite{Liang_Just_2019} provided an upper bound on the excess risk of kernel interpolation and claimed that kernel interpolation generalizes well in high dimensions. 
Similar results for kernel ridge regression are proven in \cite{Liu_kernel_2021}. 
These results are widely interpreted as evidence of the benign overfitting phenomenon (e.g., \cite{beaglehole2022kernel, Bietti_on_2019, mallinar2022benign, sanyal2020benign}): overfitted  models can still generalize well.
Building on the work of \cite{Liang_Just_2019}, the benign overfitting phenomenon has been
extensively investigated in the literature, and { we refer to} \cite{bartlett2020benign, cao2022benign, 10.1214/21-AOS2133, 9051968, tsigler2020benign} for details.
There is another line of research considering the high-dimensional setting where  $n \asymp d^{\gamma}$ for some ${\gamma}>0$. 
For example, \cite{ghorbani2021linearized} studied the square-integrable function space on the sphere $\bbS^{d}$ and proved that 
when ${\gamma}$ is a non-integer, 
kernel ridge regression is consistent if and only if the regression function is a polynomial with a fixed degree $\leq \gamma$.
Inspired by the techniques presented in \cite{ghorbani2021linearized}, 
several follow-up works extended the results to different settings \cite{aerni2023strong, Ghorbani_When_2021, Ghosh_three_2021,  mei2021learning, mei2022generalization, misiakiewicz_learning_2021}.
Additionally, 
\cite{Donhauser_how_2021} 
established an upper bound { for} kernel methods with specific kernels when ${\gamma}$ is an integer. 
{
Surprisingly, a recent work (\cite{canatar2021spectral}) numerically reported a `periodic plateau behavior' in Figure 5 (b) of their paper: when $\gamma$ varies within certain specific ranges, the excess risk of kernel regression decays very slowly.
Taken together, these works suggest that determining the convergence rate of kernel regression in high dimensions is a hard but fruitful question, and we may uncover many new phenomena if we can determine its convergence rate.
}
In this paper, we study the generalization behavior of kernel regression in high dimensions, with particular emphasis on inner-product kernels defined on the sphere $\bbS^{d}$ when $n\asymp d^{\gamma}$. More precisely, assuming that the target function $f_{\star}\in \calH^{\inner}$, the RKHS associated with an inner-product kernel on $\bbS^{d}$, we derive sharp convergence rates for the excess risk of kernel regression. 
We further show that these rates are (nearly) minimax optimal for any ${\gamma} > 0$.

\subsection{{Related works}}
The generalization behavior of high-dimensional kernel regression has recently attracted increasing { attention}. When $n\asymp d$, \cite{Karoui_spectrum_2010} discovered a linear approximation of the empirical kernel matrix,
$$
K(\bm{X}, \bm{X}) \approx \alpha_1 \bm{X}\bm{X}^\tau + \alpha_2 \mathbf{1}_n \mathbf{1}_n^\tau + \alpha_3\mathbf{I}_n,
$$
where the coefficients $\alpha_1$, $\alpha_2$, and $\alpha_3$ depend on the dimension $d$ and the inner-product kernel $K$. 
Inspired by this approximation, \cite{Liang_Just_2019} subsequently provided an upper bound $\mathbf{V}$ on the excess risk of kernel interpolation when $n\approx d$.
They further demonstrated that $\mathbf{V} \to 0$ when the data exhibits a low-dimensional structure. 
Under the same setting, \cite{Liu_kernel_2021} extends the upper bound { for} the excess risk to the kernel ridge regression with other choice of the regularization parameters. 
Furthermore, \cite{sahraee2022kernel} demonstrated that the fitting function of kernel ridge regression converges { pointwise} to the one of a linear model with two penalized terms when $n \asymp d$.
In the high-dimensional setting where $n \asymp d^{\gamma}$ for some non-integer ${\gamma}>0$, \cite{ghorbani2021linearized} 
develop the higher-order approximation for the empirical kernel matrix in the following forms:
\begin{equation}\label{eqn:approx_kernel_matrix}
\begin{aligned}
    K(\bm{X}, \bm{X}) &\approx
    ~\underbrace{ \sum_{k < r} \mu_k \bm{Y}_k(\bm{X})\bm{Y}_k(\bm{X})^\tau }_{\mathbf{I}} ~ +
    ~\underbrace{ \vphantom{\sum_{k = r}} \mu_r \bm{Y}_r(\bm{X})\bm{Y}_r(\bm{X})^\tau  }_{\mathbf{II}} ~ +
    ~\underbrace{ \sum_{k > r} \mu_k \bm{Y}_k(\bm{X})\bm{Y}_k(\bm{X})^\tau }_{\mathbf{III}},
\end{aligned} 
\end{equation}
where $r$ is an integer $\leq {\gamma}$, $\mu_k$'s are the eigenvalues of $K$, and $\bm{Y}_k(\bm{X})$ consists of spherical harmonic of degree $k$. 
They demonstrated that the term $\mathbf{III}$ in (\ref{eqn:approx_kernel_matrix}) can be approximated by an identity matrix.
By assuming that the regression function $f_{\star}$ is square-integrable on the sphere $\sqrt{d}\mathbb{S}^{d-1}$ with non-vanishing $L^2$ norm as $d \to \infty$, \cite{ghorbani2021linearized} then proved two results: (1) If $f_{\star}$ is a polynomial, then kernel ridge regression is consistent, and (2) If $f_{\star}$ is not a polynomial and if the model is noiseless, then all kernel methods are
inconsistent.
Several follow-up works have extended the results presented in \cite{ghorbani2021linearized}, and all of them adopted the square-integrable function space assumption. 
For example, \cite{Ghorbani_When_2021} consider the low-intrinsic-dimensional case; \cite{mei2022generalization} allows the degrees of the polynomials diverge with $d$; \cite{aerni2023strong, mei2021learning, misiakiewicz_learning_2021} analyze kernel ridge regression with invariance kernels and convolution kernels rather than inner-product kernels; 
\cite{Ghosh_three_2021} discuss the performance of early-stopped kernel regression;
while \cite{xiao2022precise} approximate the term $\mathbf{II}$ in (\ref{eqn:approx_kernel_matrix}) by $\bm{X}\bm{X}^\tau$
using the Marchenko–Pastur law when ${\gamma} \geq 1$ is an integer.
We discuss their assumptions regarding the function space and their results in Section \ref{subsection_4_2}.
In the work of \cite{liang2020multiple}, an upper bound on the convergence rate of the excess risk of kernel interpolation is provided when $n \asymp d^{\gamma}$, assuming ${\gamma}>1$ is fixed.
\cite{liang2020multiple} assume that the regression function can be expressed as $f_{\star}(x) = \langle K(x, \cdot), \rho_{\star}(\cdot) \rangle_{L^2}$, with $\|\rho_{\star}\|_{L^4}^4 \leq C$ for some constant $C>0$.
Then, they obtain the convergence rate $n^{-\beta({\gamma})}$, where $0 \leq \beta({\gamma}) := \min\left\{ \lceil {\gamma}\rceil / {\gamma} - 1, 1-  \lfloor {\gamma}\rfloor / {\gamma} \right\} \leq 1/ (2\lfloor {\gamma}\rfloor+1)$.
{
However, it remains uncertain whether other kernel methods with regularized terms, including early-stopped kernel regression, can achieve significantly better convergence rates than $n^{-\beta({\gamma})}$ in high dimensions.
As is recently reported in \cite{li2023kernel}, kernel interpolation generalizes substantially worse than early-stopped kernel regression in fixed dimensions. Therefore, it cannot be assumed that other kernel methods perform similarly to kernel interpolation in high dimensions.
Moreover, the results provided by \cite{liang2020multiple} are not sufficient to assert that kernel interpolation is optimal due to the absence of a corresponding minimax lower bound.
A detailed comparison of \cite{liang2020multiple} with our results and corresponding experiments are deferred to Section \ref{sec:5.2_comparison_liang_multiple}.
}
\subsection{Our contribution}
\revised{{Theories for kernel regression with polynomial eigenvalue decay rate have been well studied in the last several decades (e.g. \cite{caponnetto2007optimal, 
li2023kernel, 
raskutti2014early, 
steinwart2009optimal, 
zhang2023optimality_2,
zhang2023optimality})}.
When the dimension of data is large, because the eigenvalues of the kernel may depend on $d$ and the polynomial eigendecay property may not hold anymore, few results about the optimality of kernel regression for high-dimensional data have been obtained. 
Our main contributions are summarized below.

\emph{General bounds from spectral complexity.} Section~\ref{sec:main_results} develops an early-stopping upper bound in terms of the population Mendelson complexity and an expected minimax lower bound using metric entropy (Theorems~\ref{theorem:restate_norm_diff} and~\ref{restate_lower_bound_m_complexity}). These tools are formulated for general bounded kernels and make the dimension dependence explicit.

\emph{Refined rates on the sphere.} For the unit RKHS ball associated with a smooth inner-product kernel and uniform spherical design, Theorems~\ref{thm:near_lower_inner_large_d} and~\ref{thm:near_upper_inner_large_d} give lower and upper bounds throughout $n\asymp d^\gamma$. At positive even values of $\gamma$ the bounds match at order $n^{-1/2}$. Elsewhere they identify matching polynomial exponents, with the qualifications specified in the theorem statements. The lower bound uses finite-dimensional testing on spherical-harmonic blocks; the upper bound exploits the empirical spectral gaps and refined bias--variance estimates.

\emph{Learning-curve structure.} Figure~\ref{fig:1} displays the polynomial exponents implied by these bounds. This figure illustrates two phenomena arising in high-dimensional kernel regression.
$i)$ The first phenomenon is referred to as
the multiple descent behavior.
We plot the rate exponent with respect to $n$ for the excess risk of kernel regression. 
This curve achieves its peaks at $\gamma=2,4,6,\cdots$ and its isolated valleys at $\gamma=3,5,7,\cdots$.
$ii)$ We also report another noteworthy phenomenon, `periodic plateau behavior'.
We plot the rate exponent with respect to $d$ for the excess risk of kernel regression.
When $\gamma$ varies within certain specific ranges, we find that the value of this curve does not change.
This indicates that, in order to improve the rate of excess risk, one has to increase the sample size above a certain threshold. 
These phenomena motivate further investigation.
}
\begin{figure*}[t!]
 \begin{center}
    \begin{subfigure}[t]{0.49\textwidth}
           \includegraphics[width=2.8in, keepaspectratio]{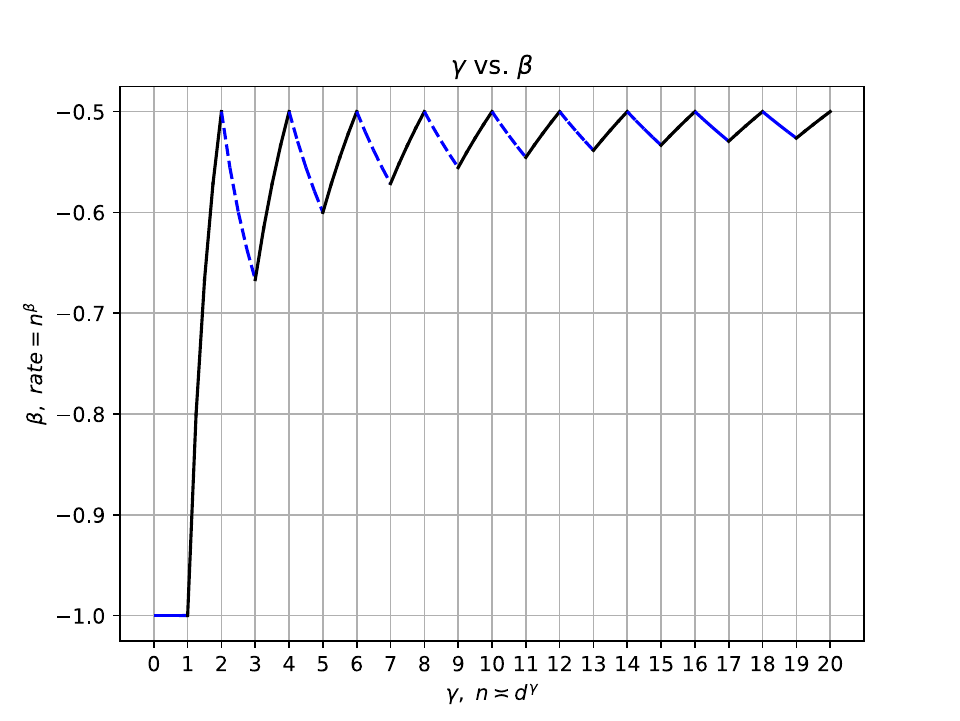}
\caption{Multiple descent behavior
}
    \end{subfigure}
    \begin{subfigure}[t]{0.49\textwidth}
           \includegraphics[width=2.8in, keepaspectratio]{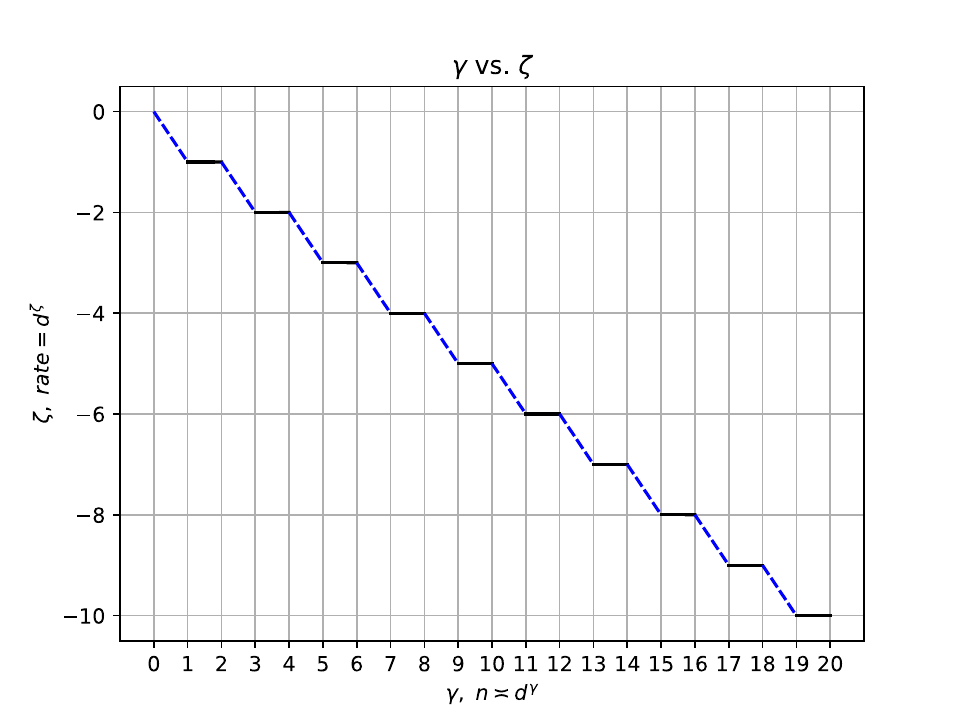}
\caption{Periodic plateau behavior
}
    \end{subfigure}
    \caption{\old{A graphical representation of the minimax optimal rate of the excess risk of kernel regression with inner-product kernels obtained from Theorem \ref{thm:near_lower_inner_large_d}, and Theorem \ref{thm:near_upper_inner_large_d}.}
\revised{Polynomial learning-rate exponents from Theorems~\ref{thm:near_lower_inner_large_d} and~\ref{thm:near_upper_inner_large_d} for the stated spherical-design and inner-product-kernel model.} 
    \old{The solid black line represents the upper bound that matches the minimax lower bound up to a constant factor. The dashed blue line indicates that, for any $\epsilon > 0$, the upper and lower rate exponents agree up to an arbitrarily small loss $\epsilon>0$ (and the logarithmic factor displayed in Theorem~\ref{thm:near_upper_inner_large_d}).}
\revised{The solid curve marks regimes with constant-factor matching bounds. The dashed curve denotes agreement of polynomial exponents, subject to the arbitrarily small exponent loss in the stated lower bound and the logarithmic factor in the upper bound.}
    }
    \label{fig:1}
    \end{center}
\end{figure*}
\subsection{Notations}
For $x\in\mathbb{R}$, let $\lceil x\rceil$ the smallest integer that is greater or equal to $x$ and by $\lfloor x\rfloor$ the greatest integer that is less or equal to $x$. 
For $\bm{v}\in\mathbb{R}^{d}$, denote by $\bm{v}_{(j)}$ the $j$-th component of $\bm{v}$ and denote the $\ell_{2}$ norm and supremum norm of $\bm{v}$ by $\|\bm{v}\|_{2}=(\sum_{j\in[d]}\bm{v}_{(j)}^{2})^{1/2}$ and $\|\bm{v}\|_{\infty}=\max_{j\in [d]}|\bm{v}_{(j)}|$ respectively. 
For a matrix $\bm{A} \in \mathbb{R}^{m\times n}$, denote by $a_{ij}$ the $(i,j)$ entry of $\bm{A}$ and denote the operator norm and the Frobenius norm of $\bm{A}$ by $\|\bm{A}\|_{\mathrm{op}}=\sup_{\bm{v}\in\mathbb{R}^{n}}\|\bm{A}\bm{v}\|_{2}/\|\bm{v}\|_{2}$ and $\|\bm{A}\|_{\mathrm{F}}=(\sum_{i\in[m],j\in[n]}a_{ij}^{2})^{1/2}$ respectively. 
Denote the $j$-th largest eigenvalue of the matrix $\bm{A}$ by $\lambda_j(\bm{A})$.
For a set $A$, let $|A|$ denote its cardinality. 
For a marginal distribution, $\rho_{\calX}$,  on $\mathcal{X}\subset \bbR^{d+1}$, 
we define the space $L^{2}(\mathcal{X}, \rho_{\calX}) = \{f:\mathcal{X} \to \mathbb{R} : \int_{\mathcal{X}} |f(\bm{x})|^2 \dx \rho_{\calX} <\infty \}$, and we denote $L^{2} = L^{2}(\mathcal{X}, \rho_{\calX})$ for simplicity. 
{
Throughout this paper, we will use the symbols $C, C_1, C_2, \dots$ to denote absolute constants, i.e., constants that have a fixed value and do not depend on any other parameters. 
Unless specified, the symbols $\mathfrak{C}, \mathfrak{C}_1, \mathfrak{C}_2, \cdots$ will denote constants that depend only on the variance $\sigma^2$ of the noise in (\ref{equation:true_model}), $\kappa$ defined in Assumption \ref{assu:trace_class}, and the constant in the asymptotic framework (\ref{Asym}), i.e.,  $c_1$, $c_2$, and $\gamma$.  
In different conclusions, we may use the same symbols, such as $C_1$, to represent different constants.
}

\section{Preliminaries}
Traditional technical tools for kernel regression are developed implicitly under the assumption that the dimension $d$ of the domain $\mathcal X$ is fixed or bounded. The recent successes of neural networks in high-dimensional data urge us to investigate the convergence rate of the excess risk of the NTK regression for data with large $d$.
{
Suppose that we have observed $n$ i.i.d. samples $(X_{i}, Y_{i}), i \in [n]$ from the model:
\begin{equation}\label{equation:true_model}
    y=f_{\star}(\bm{x})+\epsilon,
\end{equation}
where $X_{i}$'s are sampled from $\rho_{\calX}$,  $\rho_{\calX}$ is the { uniform distribution on the sphere $\mathcal{X} = \mathbb{S}^{d} \subset \bbR^{d+1}$},
$f_{\star}$ is some function defined on a compact set $\mathcal{X}$, and $\epsilon \mid \bm{x} \sim \mathcal{N}(0,\sigma^{2})$ for some fixed $\sigma>0$. Denote  the $n\times1$ data vector of ${Y_{i}}$'s  and the $n\times(d+1)$ data matrix of $X_{i}$'s by $\bm{y}$ and $\bm{X}$ respectively.
}
Let us make the following assumptions on the kernel $K$ and the candidate function class $\calB$ throughout this paper. 
\begin{assumption}\label{assu:trace_class}
Suppose that $K$ is a continuous positive definite kernel function defined on { the sphere $\mathbb{S}^{d} \subset \bbR^{d+1}$} satisfying
$\max_{x \in \mathcal X} K(x, x) \leq \kappa$ for an absolute constant $\kappa>0$.
\end{assumption}
\begin{assumption}\label{assu:H_norm}
Let us assume that $f_\star$ is in the following family of candidate functions,
\begin{align}
    \calB=\left\{ f\in \calH~\bigg\vert ~\|f\|_{\mathcal H} \leq 1\right\},
\end{align}
where $\calH$ is the RKHS associated with the kernel $K$.
\end{assumption}
\begin{remark}
    \old{The Assumption \ref{assu:trace_class} holds for a large class of kernels ( e.g. , the spherical NTK, Gaussian kernel,  Laplace kernel, etc.).}\revised{Assumption~\ref{assu:trace_class} is satisfied by many bounded continuous kernels on the sphere, including the RBF kernel and the spherical ReLU NTK. This boundedness condition is distinct from the stronger coefficient and smoothness requirements of the inner-product-kernel assumption in Section~\ref{sec:more_intevals}.}
    \old{The Assumption \ref{assu:H_norm} is merely a compact condition that is quite common and necessary regardless of the dimension $d$.}\revised{Assumption~\ref{assu:H_norm} imposes a substantive source-class restriction: the target belongs to the unit ball of the RKHS associated with the working kernel.}
    Both of these two assumptions are commonly assumed in the literature on kernel methods \cite{Caponnetto2006OptimalRF, caponnetto2007optimal, raskutti2014early} when the dimension $d$ of the domain is fixed or bounded. 
\end{remark}
Given a positive definite kernel function $K$ and a positive measure $\rho_{\calX}$ on $\calX$, the integral operator $T_{K}$ defined by
\begin{align*}
    T_{K}(f)(x)=\int K(x, y) f(y) ~\mathsf{d} \rho_{\calX}(y)
\end{align*}
is a self-adjoint compact operator. 
The celebrated Mercer's decomposition theorem further assures  that 
\begin{align}\label{eqn:mercer_decomp}
    K(x,y)=\sum\nolimits_{j}\lambda_{j}\phi_{j}(x)\phi_{j}(y),
\end{align}
where $\{\lambda_{j},j=1,2,...\}$ and ortho-normal eigen-functions $\{\phi_{j}(x),j=1,2,...\}$ are the non-increasing ordered eigenvalues and corresponding eigen-functions of $T_{K}$. After a little bit of abuse of notations, we may call  $\{\lambda_{j},j=1,2,...\}$ and $\{\phi_{j}(x),j=1,2,...\}$ the eigenvalues and eigenvectors(or eigen-functions) of the kernel function $K$ as well. 
Suppose that $f_{\star}\in \mathcal{H}$, a {\it { reproducing kernel Hilbert space}} (RKHS) \cite{cucker2002mathematical,
kohler2001nonparametric,  steinwart2008support} associated with a positive definite kernel function $K(\cdot, \cdot)$ defined on $\mathcal{X}$.
The gradient flow of the loss function 
$\mathcal{L}=\frac{1}{2n}\sum_{j}(y_{j}-f(X_{j}))^{2}$
{ induces} a gradient flow in $\mathcal{H}$ which is given by 
\begin{equation}\label{ntk:f:flow}
    \begin{aligned}
    \frac{\mathsf{d}}{\mathsf{d} t}{f}_{t}(\bm{x})=-\frac{1}{n}K(\bm{x},\bm{X})(f_{t}(\bm{X})-\bm{y}),
    \end{aligned}
\end{equation}
where $\bm{X}=(X_1, \cdots, X_n)$, $\bm{y}=(Y_1, \cdots, Y_n)^\tau$. If we further assume that ${f}_{0}(\bm{x})=0$, then we have
\begin{equation}\label{solution:gradient:flow}
f_t(\bm{x})=K(\bm{x},\bm{X})K(\bm{X},\bm{X})^{\dagger}(\bm{I}_n-e^{-\frac{1}{n}K(\bm{X},\bm{X})t})\bm{y},
\end{equation}
Here $A^{\dagger}$ denotes the Moore--Penrose pseudoinverse. The expression is well-defined even if the empirical Gram matrix is singular: equivalently,
\[
f_t(\bm{x})=\frac1n K(\bm{x},\bm{X})\int_0^t e^{-sK(\bm{X},\bm{X})/n}\,\mathsf{d}s\,\bm y.
\]
The equivalence follows from the spectral decomposition of the Gram matrix; its null eigenspace is annihilated by $K(\bm{x},\bm{X})$. At $t=\infty$, the limit is the minimum-RKHS-norm kernel interpolant of the projected response $P_{\operatorname{ran}K(\bm X,\bm X)}\bm y$, and interpolates all observations only when $\bm y$ belongs to the Gram-matrix range. 
This $f_{t}(\bm{x})$ is referred to as the estimator given by kernel regression stopped at time $t$.

\section{Main results: optimality of kernel regression}\label{sec:more_intevals}
\old{The main focus of this section is to develop new tools to determine the optimal rate for all $\gamma>0$.}
\revised{This section states the minimax lower and early-stopping upper bounds for the spherical model. Exact order matching holds at positive even values of $\gamma$; in the remaining regimes the bounds identify the same polynomial exponent, subject to the qualifications below.} 
We first specify the following large-dimensional scenario for kernel regression where we perform our analysis:
\begin{assumption}\label{assump_asymptotic}
    Suppose that there exist three positive constants $c_{1}$, $c_{2}$ and $\gamma$, such that
\begin{align}\label{Asym}
    c_{1}d^{\gamma}\leq n\leq c_{2}d^\gamma,
\end{align}
and our asymptotic statements concern sufficiently large $d$.
\end{assumption}
We consider inner-product kernels defined on the sphere. 
An inner-product kernel  is a kernel function $K$ defined on $\bbS^{d}$ 
such that there exists a function $\Phi:[-1,1] \to \mathbb{R}$
satisfying that for any $x, x^\prime \in \mathbb S^{d}$, we have $K(x, x^\prime) = \Phi(\left\langle x, x^\prime \right\rangle)$.
If the marginal distribution $\rho_{\calX}$ is the uniform distribution on $\mathcal X=\bbS ^{d}$, then Mercer's decomposition for  ${K}$ can be rewritten as
\begin{equation}\label{spherical_decomposition_of_inner_main}
\begin{aligned}
{K}(x,x^\prime) = \sum_{k=0}^{\infty} \mu_{k} \sum_{j=1}^{N(d, k)} Y_{k, j}(x) Y_{k, j}\left(x^\prime\right),
\end{aligned}
\end{equation}
where $Y_{k, j}$ for $j=1, \cdots, N(d, k)$ are spherical harmonic polynomials of degree $k$ and $\mu_{k}$'s are the eigenvalues of   $K$ with multiplicity 
$N(d,0)=1$; $N(d, k) = \frac{2k+d-1}{k} \cdot \frac{(k+d-2)!}{(d-1)!(k-1)!}, k =1,2,\cdots$. For more details of the inner-product kernels, see \cite{gallier2009notes}.
\begin{remark}\label{remark:sphere_data}
We consider the inner-product kernels on the sphere mainly because spherical harmonic analysis provides an explicit spectral decomposition ( e.g., properties of spherical harmonic polynomials are more concise than the orthogonal series on general domains). This makes Mercer's decomposition of the inner product more explicit rather than several abstract assumptions ( e.g., \cite{https://doi.org/10.1002/cpa.22008}).
    We also note that very few results are available for Mercer's decomposition of a kernel defined on the general domain, especially when the dimension of the domain is taking into consideration. e.g., even the eigen-decay rate of the neural tangent kernels is only determined for the spheres. Restricted by this technical reason, most works analyzing the spectral algorithm in large-dimensional settings  
 focus on the inner-product kernels on spheres \citep[etc.]{liang2020multiple, ghorbani2021linearized,  misiakiewicz_spectrum_2022, xiao2022precise}. 
 Though there might be several works that tried to relax the spherical assumption (e.g., \cite{liang2020multiple, aerni2023strong, barzilai2023generalization}, we can find that most of them (i) adopted a near-spherical assumption; (ii) adopted strong assumptions on the regression function, e.g., $f_{\star}(x) = x[1]x[2]\cdots x[L]$ for an integer $L>0$, where $x[i]$ denotes the $i$-th component of $x$;
 or (iii) cannot determine the convergence rate on the excess risk of the spectral algorithm.
\end{remark}
To avoid unnecessary notation, let us make the following assumption on the inner-product kernel $K$.
\begin{assumption}\label{assu:coef_of_inner_prod_kernel} 
$\Phi(t) \in \mathcal{C}^{\infty} \left([-1,1]\right)$ is a fixed function independent of $d$ and there exists a sequence of absolute constants $\{a_j > 0\}_{j \geq 0}$, such that we have
    \begin{displaymath}
        \Phi(t) = \sum_{j=0}^\infty a_j t^j.
    \end{displaymath}
\end{assumption}
The purpose of Assumption \ref{assu:coef_of_inner_prod_kernel} is to keep the main results and proofs clean. 
Notice that, by Theorem 1.b in \cite{gneiting2013strictly}, the inner-product kernel $K$ on the sphere is { positive semidefinite} for all dimensions if and only if all coefficients $\{a_{j},j=0,1,2,...\}$ are non-negative. 
The case of vanishing coefficients is not covered by the positive-coefficient results stated here.
Under this assumption, the following lemma follows from  \cite{ghorbani2021linearized}.
\begin{lemma}\label{lemma:inner_edr}
    Suppose that Assumptions ~\ref{assu:trace_class}-\ref{assu:coef_of_inner_prod_kernel} hold.
    Suppose that $p \geq 0$ is any integer. There exist positive constants $\mathfrak{C}_1$, $\mathfrak{C}_2$, $\mathfrak{C}_3$, and $\mathfrak{C}_4$, such that for any $d \geq \mathfrak{C}$, we have
\begin{equation}
\begin{aligned}
{\mathfrak{C}_1}{d^{-k}} &\leq \mu_{k} \leq {\mathfrak{C}_2}{d^{-k}}, \quad k  = 0, 1, 2, \cdots, p, p+1\\
{\mathfrak{C}_3}{d^{k}} &\leq N(d, k) \leq {\mathfrak{C}_4}{d^{k}}, \quad k  = 0, 1, 2, \cdots, p, p+1.
\end{aligned}
\end{equation}
\end{lemma}
\old{We first state the minimax lower bounds.}
\revised{We first state lower bounds for the expected minimax risk over the unit RKHS ball. These are compared below with high-probability upper bounds for a specific early-stopped estimator; the two probability statements are deliberately kept distinct.}
\begin{theorem}
\label{thm:near_lower_inner_large_d}
Let $\rho_{f_{\star}}$ be the joint law of $x, y$ given by (\ref{equation:true_model}) with $f_{\star}\in \mathcal B = \{ f\in \calH^{\inner} ~ \mid ~\|f\|_{\calH^{\inner}} \leq 1 \}$.
Let $\gamma > 0$ be a fixed real number and $p = \lfloor \gamma/2 \rfloor$.
Assume the large-dimensional sampling regime in Assumption~\ref{assump_asymptotic}, namely $c_1d^\gamma\le n\le c_2d^\gamma$ for fixed $0<c_1\le c_2<\infty$.
{
Suppose that Assumptions ~\ref{assu:trace_class}-\ref{assu:coef_of_inner_prod_kernel} hold.
}
Then we have the following statements.
\begin{itemize}
    \item[(i)] If $\gamma \in \{2, 4, 6, \cdots\}$, then, there exist constants $\mathfrak{C}_1>0$ and $\mathfrak{C}$, such that for any $d \geq \mathfrak{C}$, we have:
\begin{equation}
\inf _{\hat{f}} \sup _{f_{\star} \in \mathcal B} \mathbb{E}_{(\bold{X}, \bold{y}) \sim \rho_{f_{\star}}^{\otimes n}}
\left\|\hat{f} - f_{\star}\right\|_{L^2}^2
\geq \mathfrak{C}_1 n^{-\frac{1}{2}};
\end{equation}
    \item[(ii)] If $\gamma \in \bigcup_{j=0}^{\infty} (2j, 2j+1]$, then, for any $\epsilon>0$, there exist constants $\mathfrak{C}_1>0$ and $\mathfrak{C}$ depending on $c_1$, $c_2$, $\gamma$, $\epsilon$, and the fixed kernel and noise parameters, such that for any $d \geq \mathfrak{C}$, we have:
\begin{equation}
\inf _{\hat{f}} \sup _{f_{\star} \in \mathcal B} \mathbb{E}_{(\bold{X}, \bold{y}) \sim \rho_{f_{\star}}^{\otimes n}}
\left\|\hat{f} - f_{\star}\right\|_{L^2}^2
\geq \mathfrak{C}_1 n^{-\left(\frac{\gamma - p}{\gamma} + \epsilon\right)};
\end{equation}
    \item[(iii)] If $\gamma \in  \bigcup_{j=0}^{\infty} (2j+1, 2j+2)$, then, there exist constants $\mathfrak{C}_1>0$ and $\mathfrak{C}$, such that for any $d \geq \mathfrak{C}$, we have:  
\begin{equation}
\inf _{\hat{f}} \sup _{f_{\star} \in \mathcal B} \mathbb{E}_{(\bold{X}, \bold{y}) \sim \rho_{f_{\star}}^{\otimes n}}
\left\|\hat{f} - f_{\star}\right\|_{L^2}^2
\geq \mathfrak{C}_1 n^{-\frac{p+1}{\gamma}};
\end{equation}
\end{itemize}
where $\hat{f}$ is any estimator of $f_{\star}$, measurable with respect to the observed data set $(\bm{X}, \bm{Y})$.
\end{theorem}
\old{We next present high-dimensional upper bounds on excess risk that nearly match the minimax lower bounds given in Theorem \ref{thm:near_lower_inner_large_d}.}
\revised{We next give high-probability prediction-error bounds for early-stopped kernel regression. The resulting polynomial exponents agree with the minimax lower-bound exponents, with exact order matching only in the regimes stated explicitly.}
\begin{theorem}
\label{thm:near_upper_inner_large_d}
Suppose that $\calH^{\inner}$ is an RKHS associated with $K^{\inner}$ defined on $\mathbb S^{d}$. 
    Let $f_{\widehat{T}}^{\inner}$ be the function defined in \eqref{solution:gradient:flow} where $\widehat{T}^{-1}=\widehat{\varepsilon}_n^{2}$ defined in (\ref{eqn:def_empirical_mendelson_complexity}) and $K=K^{\inner}$.
    Let $\gamma > 0$ be a fixed real number and $p = \lfloor \gamma/2 \rfloor$.
    Assume the large-dimensional sampling regime in Assumption~\ref{assump_asymptotic}, namely $c_1d^\gamma\le n\le c_2d^\gamma$ for fixed $0<c_1\le c_2<\infty$.
    Suppose further that Assumptions ~\ref{assu:trace_class}-\ref{assu:coef_of_inner_prod_kernel} hold with $\calH = \calH^{\inner}$.
Then, we have the following statements:
\begin{itemize}
    \item[(i)] If $\gamma \in \{2, 4, 6, \cdots\}$, then, there exist constants $\mathfrak{C}$ and $\mathfrak{C}_{i}$, where $i=1,2,3$, such that for any $d \geq \mathfrak{C}$,  we have
\begin{equation}
\begin{aligned}
\left\|{f}_{\widehat{T}}^{\inner} - f_{\star}\right\|_{L^2}^2
\leq \mathfrak{C}_1 n^{-\frac{1}{2}},
\end{aligned}
\end{equation}
holds with probability at least $1 -\mathfrak{C}_2\exp\{-\mathfrak{C}_3 n^{1/2} \}$.
    \item[(ii)] If $\gamma \in \bigcup_{j=0}^{\infty} (2j, 2j+1]$, then, for any $\delta>0$, there exist constants $\mathfrak{C}$ and $\mathfrak{C}_{i}$, where $i=1,2,3$, depending on $\gamma$, $\delta$, $c_1$, $c_2$, and the fixed kernel and noise parameters, such that for any $d \geq \mathfrak{C}$,  we have 
\begin{equation}
\begin{aligned}
\left\|{f}_{\widehat{T}}^{\inner} - f_{\star}\right\|_{L^2}^2
\leq \mathfrak{C}_1 n^{-\frac{\gamma - p}{\gamma}} \log(n),
\end{aligned}
\end{equation}
holds with probability at least $1-\delta-\mathfrak{C}_2\exp\{-\mathfrak{C}_3 n^{p/\gamma} \log(n)\}$.
    \item[(iii)] If $\gamma \in \bigcup_{j=0}^{\infty} (2j+1, 2j+2)$, then, for any $\delta>0$, there exist constants $\mathfrak{C}$ and $\mathfrak{C}_{i}$, where $i=1,2,3$, depending on $\gamma$, $\delta$, $c_1$, $c_2$, and the fixed kernel and noise parameters, such that for any $d \geq \mathfrak{C}$,  we have  
\begin{equation}
\begin{aligned}
\left\|{f}_{\widehat{T}}^{\inner} - f_{\star}\right\|_{L^2}^2
\leq \mathfrak{C}_1 n^{-\frac{p+1}{\gamma}},
\end{aligned}
\end{equation}
holds with probability at least $1-\delta-\mathfrak{C}_2\exp\{-\mathfrak{C}_3 n^{1 - (p+1)/\gamma}\}$.
\end{itemize}   
\end{theorem}
{
We note that the above periodic behavior is very much linked to the spectral structure of inner-product kernels for uniform data on a large-dimensional sphere.
Notice that from Lemma \ref{lemma:inner_edr}, eigenvalues of order $d^{-k}$ is of  multiplicity $d^{k}$, $k=0, \cdots, p+1$.
Such a strong block structure of the spectrum makes both the bias and variance terms of the (empirical) excess risk decrease with gaps when $\gamma$ increases
(see, e.g., Lemma \ref{lemma_B_t:thm:empirical_loss} and \ref{lemma_V_t:thm:empirical_loss},
and their modified version Lemma \ref{lemma_B_t:thm:empirical_loss_modified} and \ref{lemma_V_t:thm:empirical_loss_modified}).
}
{
\begin{remark}\label{remark_limitation_existing_technique}
    Existing results already imply Theorem \ref{thm:near_lower_inner_large_d} and Theorem \ref{thm:near_upper_inner_large_d} for the special cases where $n\asymp d^{\gamma}, \gamma=2, 4, 6,\cdots$, see Appendix \ref{section_4_ntk_certain_asymp_frameworks} for details.
    For non-even $\gamma$, the generic entropy-critical-radius lower bound need not match the refined high-dimensional minimax lower bound of Theorem \ref{thm:near_lower_inner_large_d}. The latter follows from the spectral block structure and a direct finite-dimensional hypercube testing argument, without relying on an entropy-critical-radius comparison.
A key change in Theorem \ref{thm:near_upper_inner_large_d} appears to lie in the treatment of the bias term $\boldsymbol{B}_t$ and the variance term $\boldsymbol{V}_t$ (specifically, the improvement from Lemmas \ref{lemma_B_t:thm:empirical_loss} and \ref{lemma_V_t:thm:empirical_loss} to Lemmas \ref{lemma_B_t:thm:empirical_loss_modified} and \ref{lemma_V_t:thm:empirical_loss_modified}).
\end{remark}
}
{
\begin{remark}
    Denote $\mathrm{P}_{>0}$ as the projection onto
    linear space of spherical harmonics with degree $>0$.
    In the region $n \ll d$ (i.e. $\gamma \in(0,1)$), since $\left\|\mathrm{P}_{>0} f_*\right\|_{L^2}^2 \leq \mu_1\left\|f_*\right\|_{\mathcal{H}}^2 \lesssim 1 / d$, one is essentially fitting a trivial predictor (a constant), which has excess risk indeed $O(1 / n)$. Hence, such a region is in fact not very interesting, though the convergence rate is faster (as a function of $n$ ).
\end{remark}
\begin{remark}
    \old{In Theorem \ref{thm:near_upper_inner_large_d}, a data-driven optimal stopping time is given, and we show that kernel regression stopped at $\widehat{T}$ attains the minimax rate up to the logarithmic and arbitrarily small polynomial gaps stated in Theorems \ref{thm:near_lower_inner_large_d} and \ref{thm:near_upper_inner_large_d}.}\revised{Theorem~\ref{thm:near_upper_inner_large_d} analyzes a data-dependent stopping time. Together with Theorem~\ref{thm:near_lower_inner_large_d}, it establishes constant-factor rate matching for positive even $\gamma$ and matching polynomial exponents, with the stated qualifications, in the remaining regimes.} 
    Moreover, the order of $\widehat{T}$ is $\Theta(n^{1/2})$ (see, e.g., Lemma \ref{lemma:bound_empirical_and_expected_mandelson_complexities} and Lemma \ref{lemma:inner_mendelson_interval_control}), which is independent of $\gamma$.
\end{remark}
}
{
\begin{remark}
    \old{Theorem \ref{thm:near_lower_inner_large_d} and Theorem \ref{thm:near_upper_inner_large_d} establish that the excess risk of kernel  gradient flow achieves the minimax rate over the RKHS ball $\mathcal{B}$.}\revised{Theorems~\ref{thm:near_lower_inner_large_d} and~\ref{thm:near_upper_inner_large_d} characterize the minimax polynomial exponent over the RKHS ball $\mathcal{B}$, with exact order matching in the specified even-$\gamma$ regimes.} However, it is worth noting that minimax optimality can be conservative: the worst-case functions that attain the lower bound may be highly irregular or oscillatory, and for "typical" or smoother target functions $f_{\star}$, the actual excess risk could be significantly smaller than the minimax rate. 
    Nevertheless, the minimax perspective provides a rigorous benchmark for understanding the fundamental difficulty of the estimation problem over the entire function class.
\end{remark}
}
\old{Figure \ref{fig:1} illustrates the results obtained by Theorem \ref{thm:near_lower_inner_large_d} and Theorem \ref{thm:near_upper_inner_large_d}. From these figures, we observe some interesting phenomena.}\revised{Figure~\ref{fig:1} summarizes the polynomial risk exponents identified by Theorems~\ref{thm:near_lower_inner_large_d} and~\ref{thm:near_upper_inner_large_d}. The curves concern asymptotic rates as $d$ and $n$ grow together, rather than the finite-sample risk at a fixed dimension.} 
\section{Bounds for high-dimensional kernel regression}\label{sec:main_results}
{
In this section, we present technical results that pertain to general kernel learning rate bounds and are applicable to a continuous kernel $K$ defined on a compact space $\mathcal{X}$ (not necessarily $\mathbb{S}^d$ ).
\old{We believe these results may hold independent interest.}\revised{These results may also be useful beyond the spherical setting considered in Section~\ref{sec:more_intevals}.}
}
\old{We first introduce the (population and empirical) Mendelson complexity, the key quantities in determining the minimax rate of regression over $\calB$.}\revised{We begin with localized spectral complexities that control the risk of early-stopped kernel regression. We then introduce metric entropy and a sufficient condition for a minimax lower bound. In the spherical model, these general tools recover the even-$\gamma$ regime; the sharper bounds for other $\gamma$ in Section~\ref{sec:more_intevals} require the spectral-block analysis.} 
\begin{definition}[Mendelson complexity]\label{def:pop_men_complexity}
Suppose that $K$ is a kernel function satisfying the Assumption \ref{assu:trace_class}, we then  introduce:
\begin{itemize}
    \item[i)]({\it Population Mendelson complexity}) Let $\lambda_i$'s be the eigenvalues of $K$ given in (\ref{eqn:mercer_decomp}) and 
$\mathcal{R}_{K}(\varepsilon):=\left[\frac{1}{n} \sum_{j=1}^{\infty} \min \left\{\lambda_j, \varepsilon^2\right\}\right]^{1 / 2}
$.
The population Mendelson complexity is given by
\begin{equation}\label{eqn:def_population_mendelson_complexity}
 \begin{aligned}
{\varepsilon}_n :=\inf\left\{\varepsilon>0:\ {\mathcal{R}}_{{K}}(\varepsilon) \leq \varepsilon^2 /(2e \sigma)\right\}.
\end{aligned}   
\end{equation}
\item[ii)]({\it Empirical Mendelson complexity})
Let $\widehat \lambda_i$'s be the eigenvalues of $\frac{1}{n} K(\bm X,\bm X)$ and 
$\widehat R_{K}(\varepsilon) =\left[\frac{1}{n}\sum_{j=1}^{n} \min\{\widehat \lambda_{j},\varepsilon^{2}\}\right]^{1/2}$.
The empirical Mendelson complexity is given by
\begin{equation}\label{eqn:def_empirical_mendelson_complexity}
 \begin{aligned}
\widehat{\varepsilon}_n :=\inf\left\{\varepsilon>0:\ \widehat{\mathcal{R}}_{{K}}(\varepsilon) \leq \varepsilon^2 /(2e \sigma)\right\}.
\end{aligned}   
\end{equation}
If the empirical Gram matrix is identically zero, the above infimum is zero. In this degenerate case we define the stopping time to be $\widehat T=0$ (so $f_{\widehat T}=f_0$). Otherwise the empirical critical radius is strictly positive and the stopping time is $\widehat T=\widehat\varepsilon_n^{-2}$. All subsequent uses of the positive critical-root identity are on the nondegenerate event; in particular, the empirical/population critical-radius comparison event excludes the zero-Gram case whenever $\varepsilon_n>0$.
\end{itemize}
\end{definition}
{
\begin{remark}
    The Mendelson complexity can be interpreted as the optimal bias‑variance trade‑off.
Specifically, recall that the Empirical Mendelson complexity is defined as $\widehat{\varepsilon}_n :=\inf\left\{\varepsilon>0:\ \widehat{\mathcal{R}}_{{K}}(\varepsilon) \leq \varepsilon^2 /(2e \sigma)\right\}$. 
Denote $\varepsilon = t^{-1/2}$, then from Lemma \ref{lemma_B_t:thm:empirical_loss} and the proof of Lemma  \ref{lemma_V_t:thm:empirical_loss} we have $Bias^2 \lesssim \varepsilon^2$ and $Variance \lesssim {\varepsilon}^2/(2e \sigma)^2 + {\varepsilon}^{-2}\widehat R_{K}(\varepsilon)^2$ with high probability, hence a balance of these terms gives the definition of the Empirical Mendelson complexity.
\end{remark}
}
{
\begin{remark}
    For a nonzero population (respectively empirical) spectrum, the positive critical radius is well-defined and is the unique positive crossing of the corresponding complexity inequality. For an identically zero empirical Gram matrix, we use the zero-radius convention specified above; no positive empirical critical radius exists.
\end{remark}
}
\vspace{4mm}
\noindent {\it Upper bound of the excess risk of kernel regression.} \quad The Mendelson complexity is closely related to the upper bound of the excess risk of  kernel regression. 
\begin{theorem}[Upper bound]
\label{theorem:restate_norm_diff}
Suppose that Assumptions ~\ref{assu:trace_class} and~\ref{assu:H_norm} hold.
Let $f_{\widehat{T}}$ be the function defined in \eqref{solution:gradient:flow} where $\widehat{T}=1/\widehat{\varepsilon}_n^{2}$ on the positive empirical critical-radius event (and $\widehat T=0$ for the identically zero empirical Gram matrix).  
Suppose for any absolute constant $C$, there exists a constant $\mathfrak{C}$, such that for any $n \geq \mathfrak{C}$, we have $n\varepsilon_n^{2} \geq C$.
  Then
there exist positive constants $C_1$, $C_2$, and $C_3$, depending at most on the fixed model parameters, and a constant $\mathfrak{C}_0$, such that for any $n \geq \mathfrak{C}_0$, we have
\begin{equation}
\begin{aligned}
\left\|{f}_{\widehat{T}} - f_{\star}\right\|_{L^2}^2
 \leq C_1 \varepsilon_n^2,
\end{aligned}
\end{equation}
with probability at least $1-C_2\exp \left\{ - C_3 n  \varepsilon_n^2  \right\}$.
\end{theorem}
\old{Related results have been established in \cite{raskutti2014early} for fixed $d$ (see, e.g., Theorem 2 of \cite{raskutti2014early}),
 the contribution here is an upper bound whose constants are uniform in $n$ and $d$ when the kernel envelope and noise scale are fixed. 
Thus, we could apply it to the large-dimensional scenario.}\revised{The bound is related to the fixed-dimensional early-stopping analysis of \cite{raskutti2014early}. Here the constants are controlled uniformly over $n$ and $d$ under the stated kernel-envelope and noise assumptions, which permits its use in the polynomial high-dimensional regime.} 
{
\begin{remark}
    The proofs for results in Theorem \ref{theorem:restate_norm_diff} and for results in \cite{raskutti2014early} rely on Mendelson complexity. Comparing to results in Theorem \ref{thm:near_upper_inner_large_d}, it can be shown that the upper bounds provided by Theorem \ref{theorem:restate_norm_diff} become loose in high dimensions due to the degeneracy of eigenspaces of inner‑product kernels on the sphere (i.e., $d^k$ eigenvalues of order $d^{-k}$).
\end{remark}
}
\begin{remark}
Since $\varepsilon_n$ should be much slower than the typical parametric rate $n^{-1/2}$ \cite{Chen_Convergence_1988, Heckman_Spline_1986},
previous works have commonly assumed the existence of constants $\mathfrak{C}$ and $C$, such that for any $n \geq \mathfrak{C}$, we have $n\varepsilon_n^2 \geq C$ (e.g., \cite{raskutti2014early}).
However, most of these works implicitly assumed that $d$ is bounded and   $\{\lambda_{j}\}$ are polynomially decayed and ignored the dependence of the constant $\mathfrak{C}$ on $\{\lambda_{j}\}$ and $d$.
 Theorem \ref{theorem:restate_norm_diff} explicitly requires that $\mathfrak{C}$ only depends on $c_{1},c_{2}$, and $\gamma$.
\end{remark}
\noindent {\it Lower bound of the excess risk of kernel regression.}
Suppose that $(\calZ,d)$ is a topological space with a compatible loss function $d$, which are mappings from $ \calZ\times  \calZ$ to  $\bbR_{\geq 0}$ with $d(f, f)=0$ and $d(f, f^{\prime}) >0$ for $f \neq f^{\prime}$. 
We introduce the packing entropy and covering entropy below:
\begin{definition}[Packing entropy]
A finite set $N_{\varepsilon} \subset \mathcal Z$ is said to be an $\varepsilon$-packing set in $\mathcal Z$ with separation $\varepsilon>0$, if for any $f, f^{\prime} \in N_{\varepsilon}, f \neq f^{\prime}$, we have $d\left(f, f^{\prime}\right)>\varepsilon$. The logarithm of the maximum cardinality of $\varepsilon$-packing set is called the $\varepsilon$-packing entropy or Kolmogorov capacity of $\calZ$ with distance $d$ and is denoted by 
$M_{d}(\varepsilon,\calZ)$.
\end{definition}
\begin{definition}[Covering entropy]\label{def:covering_entropy}
A set $G_{\varepsilon} \subset \mathcal Z$ is said to be an $\varepsilon$-net for $\mathcal Z$ if for any $\tilde{f} \in \mathcal Z$, there exists an $f_0 \in G_{\varepsilon}$ such that $d(\tilde{f}, f_0) \leq \varepsilon$. The logarithm of the minimum cardinality of $\varepsilon$-net is called the $\varepsilon$-covering entropy of $\mathcal Z$ and is denoted by
$V_{d}(\varepsilon,\calZ)$.
\end{definition}
Let $M_{2}(\varepsilon,\calB)$ be the $\varepsilon$-packing entropy of $(\calB, d^2=\|\cdot\|_{L^2}^2)$ and $V_{2}(\varepsilon,\calB)$ be the $\varepsilon$-covering entropy of $(\calB, d^2=\|\cdot\|_{L^2}^2)$.
Direct calculation shows that 
$
    M_{2}(2\varepsilon, \calB) \leq V_{2}(\varepsilon, \calB) \leq M_{2}(\varepsilon, \calB)
$ (see, e.g., Lemma \ref{lemma_M_2_and_V_2} ).
If we further introduce
\begin{align}\label{def_calD}
    \calD=\left\{ \rho_{f}~\bigg\vert~ \mbox{ joint distribution of $(y,x$) where } x\sim \rho_{\calX}, y=f(x)+\epsilon, \epsilon \mid \bm{x} \sim N(0,\sigma^{2}),
    f\in \calB \right\},
\end{align}
and let
 $V_{K}(\varepsilon,\calD)$ be the $\varepsilon$-covering entropy of $(\calD, d^2=\text{ KL divergence })$. Then it is easy to verify that $V_{2}(\varepsilon, \calB ) = V_{K}({\varepsilon}/{(\sqrt{2}\sigma)}, \calD)$ (see, e.g., Lemma \ref{claim:d_K_and_d_2} ).
The classical entropy approach to minimax estimation, including the work of Yang and Barron \citep{Yang_Density_1999}, compares covering and packing entropies and studies conditions under which their critical scales agree. We recall the terminology relevant to that comparison, without assuming that an entropy critical radius satisfies an exact equality.
\begin{remark}\label{remark:classical_entropy_radii}
For a class $\calB$, a covering critical radius may be defined by the generalized inverse
\[
\bar\varepsilon_n^{\mathrm{crit}}:=\inf\{r>0:V_K(r,\calD)\le nr^2\}.
\]
Because covering entropy need not be continuous, the infimum need not solve $V_K(r,\calD)=nr^2$ or even belong to the displayed set. We use an explicitly admissible radius whenever the minimax theorem below requires one. The classical \emph{richness condition} refers to a fixed-dimensional entropy-growth property such as
\[
\liminf_{\varepsilon\downarrow0}
\frac{M_2(\alpha\varepsilon,\calB)}{M_2(\varepsilon,\calB)}>1
\qquad\text{for some }0<\alpha<1.
\]
No claim about constants uniform in the ambient dimension follows from this condition alone.
\end{remark}
The next proposition shows that no fixed entropy-growth gap can be inferred uniformly along a suitable high-dimensional sequence of radii:
\begin{proposition}[The problem of the richness condition in high dimensions]\label{prop:rc_violated_in_large_dimensions}
    Let $
    \calB=\{ f_{\star}\in \calH^{\inner}~\vert ~\|f_{\star}\|_{\mathcal H^{\inner}} \leq 1\}$,
where
$\calH^{\inner}$ is the RKHS associated with $K^{\inner}$.
    For any $0<\alpha<1$ and any $\delta>0$,
    there exists a sequence $\{\tilde{\varepsilon}_d\}_{d=1}^\infty$, such that $\liminf_{d \rightarrow \infty} \tilde{\varepsilon}_d = 0$, and we have
    $$
    \liminf _{d \rightarrow \infty} \frac{M_2(\alpha \tilde{\varepsilon}_d, \calB)}{M_2(\tilde{\varepsilon}_d, \calB)} \leq 1+\delta.
    $$
\end{proposition}
Proposition \ref{prop:rc_violated_in_large_dimensions} concerns a sequence of dimensions and radii; it does not disprove the fixed-dimensional richness condition. It shows why a dimension-uniform richness constant cannot be taken for granted when using fixed-dimensional entropy-growth arguments. 
To avoid potential confusion, we specify the high-dimensional scenario for kernel regression where we perform our analysis as in Assumption \ref{assump_asymptotic}. 
\old{The following theorem provides a minimax lower bound of kernel regression in high dimensions by replacing the richness condition with a more general
condition involving the $\varepsilon$-covering entropy.}\revised{The next theorem states a sufficient minimax lower-bound condition directly in terms of covering entropy. Its hypotheses are checked at an admissible radius, without assuming a dimension-uniform richness condition.}
\begin{theorem}[Minimax lower bound]
\label{restate_lower_bound_m_complexity}
Let $\bar{\varepsilon}_{n}>0$ be any radius satisfying $V_K(\bar\varepsilon_n,\mathcal D)\le n\bar\varepsilon_n^2$. 
Assume that there exists a constant $\mathfrak{C}$, such that for any $n \geq \mathfrak{C}$, we have $n\bar\varepsilon_n^2 \geq 2\log 2$.
Then for any constant $\mathfrak{c}_2 > 0$
such that the inequality 
\begin{equation}\label{eqn:lower_condition_24}
\begin{aligned}
        5n\bar\varepsilon_n^2 \leq V_{2}(\mathfrak{c}_{2} \bar\varepsilon_n, \mathcal B)
    \end{aligned}
\end{equation}
holds for any $n \geq \mathfrak{C}$,  we have
\begin{equation}\label{eqn:lower_bound_main_thm}
\inf _{\hat{f}} \sup _{f_{\star} \in \mathcal B} \mathbb{E}_{(\bold{X}, \bold{y}) \sim \rho_{f_{\star}}^{\otimes n}}
\left\|\hat{f} - f_{\star}\right\|_{L^2}^2
\geq \frac{1}{2}\left(\frac{\mathfrak{c}_2}{12}\right)^2 \bar\varepsilon_n^2,
\end{equation}
where $\hat{f}$ is any estimator of $f_{\star}$, measurable with respect to the observed data set $(\bm{X}, \bm{Y})$ and $\rho_{f_{\star}}$ is the joint law of $(x,y)$ specified by (\ref{equation:true_model}) with $f=f_{\star}$.
\end{theorem}
\begin{remark}\label{remark:rc_to_18}
    {
The entropy condition in Theorem \ref{restate_lower_bound_m_complexity} is a sufficient condition stated directly in terms of the covering entropy at the chosen radius. The classical richness condition controls ratios of packing entropies at nearby radii, but does not by itself establish the strengthened inequality \eqref{eqn:lower_condition_24} for a generalized-inverse critical radius. Accordingly, we verify \eqref{eqn:lower_condition_24} directly in the applications below; no implication from richness is used in our main results.
}
\end{remark}
{
    \begin{remark}
        Theorem \ref{theorem:restate_norm_diff} and Theorem \ref{restate_lower_bound_m_complexity} can be applied to high-dimensional kernel regression with inner-product kernels defined on the sphere: in the cases where $n\asymp d^{\gamma}, \gamma=2, 4, 6,\cdots$, we can use the above theorems to recover results in Theorem \ref{thm:near_lower_inner_large_d} and Theorem \ref{thm:near_upper_inner_large_d}, see Appendix \ref{section_4_ntk_certain_asymp_frameworks} for details. However, when $\gamma \neq 2, 4, 6,\cdots$, these results do not yield tight bounds as given in Theorem \ref{thm:near_lower_inner_large_d} and Theorem \ref{thm:near_upper_inner_large_d}, see Remark \ref{remark_limitation_existing_technique} for details.
    \end{remark}
}

\section{Kernel Regression in High Dimensions: Interpretation and Comparison}\label{sec:discuss}
\old{Since \cite{Jacot_NTK_2018} introduced the NTK, studying the generalization performance of kernel methods has become a natural surrogate for studying the generalization performance of neural networks.}\revised{This section places the spherical-RKHS learning-rate results in the context of high-dimensional kernel regression. We first compare assumptions on the target function and the resulting risk guarantees, then contrast the available bounds for interpolation and early stopping, and finally present numerical illustrations. The comparisons concern distinct function classes or risk functionals where indicated.} 
In recent years, substantial work has been done in kernel regression with fixed-dimensional
(e.g. \cite{caponnetto2007optimal, 
li2023kernel, 
raskutti2014early, 
steinwart2009optimal, 
zhang2023optimality_2,
zhang2023optimality}).
These results have advanced our understanding of kernel regression while motivating further questions in high dimensions. 
For example, \cite{li2023kernel} showed that \old{fixed-dimensionalal kernel interpolation}\revised{fixed-dimensional kernel interpolation} generalized poorly, which conflicts with the widely observed `benign overfitting' phenomenon. 
Some researchers then speculated that in certain scenarios, the `benign overfitting phenomenon' might be due to the high dimensionality of data. 
This motivates the study of kernel regression with
high-dimensional data (i.e., $n \asymp d^{\gamma}$ for some $\gamma>0$) 
(see, e.g., 
\cite{Donhauser_how_2021,ghorbani2021linearized,
Liang_Just_2019, liang2020multiple, Liu_kernel_2021, sahraee2022kernel,  xiao2022precise}). 
In this section, we review related findings and compare them with Theorem \ref{thm:near_lower_inner_large_d} and Theorem \ref{thm:near_upper_inner_large_d}. These results, together with our findings, suggest that there might be other deeper structures hidden in kernel regression on high-dimensional data.
\subsection{Consistency of kernel regression when $n \asymp d^{\gamma}$, $\gamma > 0$}\label{subsection_4_2}
\old{We call a nonparametric regression method {\it consistent} if its estimator's excess risk converges to zero as $n \to \infty$, and {\it inconsistent} otherwise.}\revised{We call an estimator sequence consistent when its excess risk converges to zero as $n\to\infty$ in the specified mode of convergence. For high-dimensional kernel methods, consistency depends on both the target class and the relation between $n$ and $d$.} 
Several studies have discussed the inconsistency of kernel methods with inner-product kernels when $n \asymp d^{\gamma}$ for some non-integer $\gamma$ (\cite{Ghorbani_When_2021,
ghorbani2021linearized,
Ghosh_three_2021,
mei2022generalization,  misiakiewicz_spectrum_2022}).
We first adopt notation from \cite{ghorbani2021linearized}.
Denote 
$R_{\mathrm{KR}}\left(f_{\star}, \boldsymbol{X}, \lambda\right)$ as the excess risk of kernel ridge regression with regularization parameter $\lambda$, and
$R_K\left(g\right):= \min _a \mathbb{E}_x\left\{\left(g(x)-\sum_{i=1}^n a_i K\left( x_i, x\right)\right)^2\right\}$ as a lower bound on the prediction error of general kernel methods with regression function $g$, $\mathrm{P}_{\leq \ell}$ as the projection onto polynomials with degree $\leq \ell$, and $\mathrm{P}_{>\ell}$ as the projection onto polynomials with degree $>\ell$.
\begin{remark}
    For functions defined on $\mathbb{S}^{d}$, 
    $\mathrm{P}_{\leq \ell}$ is the projection onto
    linear space of spherical harmonics with degree $\leq \ell$ (see, e.g., Definition 1.1.1 in \cite{dai2013approximation}). These spherical harmonics form an orthonormal basis of $L^2\left(\mathbb{S}^d, \rho_{\calX}\right)$, and thus can represent functions in any RKHS $\calH \subset L^2\left(\mathbb{S}^d, \rho_{\calX}\right)$. For example, $\mathcal{H}^{\NTK}$ is spanned by spherical polynomials with degree $\ell=0, 1, 2, 4, \cdots$ (\cite{Bietti_on_2019}).
\end{remark}
The following two propositions restate results from \cite{ghorbani2021linearized}:
\begin{proposition}[Restate Theorem 3 in \cite{ghorbani2021linearized}]\label{restate_thm_3_in_linear}
Suppose there exists an integer $\ell \in \{0, 1, \cdots\}$, and a constant $0<\delta < 1$, such that $n \asymp d^{\ell+1-\delta}$.
Assume that $f_{\star}$ is square-integrable on $\sqrt d\,\mathbb S^{d-1}$ with uniformly bounded $L^2$ norm.
Assume the rotationally invariant kernel satisfies the square-integrability condition in Theorem 3 of \cite{ghorbani2021linearized}. The result concerns the best approximation in the sample-dependent kernel span; it does not require a noise assumption.
Then, 
    for any $\varepsilon>0$, with high probability, we have
\begin{equation}
    \left|R_K\left(f_{\star}\right)-R_K\left(\mathrm{P}_{\leq \ell} f_{\star}\right)-\left\|\mathrm{P}_{>\ell} f_{\star}\right\|_{L^2}^2\right| \leq \varepsilon\left\|f_{\star}\right\|_{L^2}\left\|\mathrm{P}_{>\ell} f_{\star}\right\|_{L^2}.
\end{equation}
\end{proposition}
\begin{proposition}[Restate Theorem 4 in \cite{ghorbani2021linearized}]\label{restate_thm_4_in_linear}
Suppose there exists an integer $\ell \geq 0$ and a constant $0<\delta<1$ such that $d^\ell\log d=o(n)$ and $n=O(d^{\ell+1-\delta})$.
Assume that $f_{\star}$ is square-integrable on $\sqrt{d}\,\mathbb S^{d-1}$ and that Assumption 3 in \cite{ghorbani2021linearized} holds for the kernel $K$. Assume also the Gaussian-noise model of that reference, and, if the kernel has zero mean, its additional centering condition on $f_\star$.
Then, 
    for any $\varepsilon>0$ and any regularization parameter $0<\lambda<\lambda^*$ with high probability, we have
\begin{equation}
    \left|R_{\mathrm{KR}}\left(f_{\star}, \boldsymbol{X}, \lambda\right)-\left\|\mathrm{P}_{>\ell} f_{\star}\right\|_{L^2}^2\right| \leq \varepsilon\left(\left\|f_{\star}\right\|_{L^2}^2+\sigma^2\right),
\end{equation}
where $\lambda^*$ is defined as (20) in \cite{ghorbani2021linearized}.
\end{proposition}
By assuming that the regression function falls into the square-integrable function
space, we summarize the relevant implications in three points:
\begin{itemize}
    \item[(1)] 
    When $f_{\star}$ is a polynomial with a degree at most $\ell \geq 0$,  Proposition \ref{restate_thm_4_in_linear} demonstrates that under specific regularization parameters, kernel ridge regression is consistent when $n \asymp d^{\gamma}$ for some non-integer $\gamma>\ell$. 
    \item[(2)] 
    If the high-degree energy $\|\mathrm P_{>\ell}f_\star\|_{L^2}^2$ stays bounded away from zero as $d\to\infty$, Proposition~\ref{restate_thm_3_in_linear} implies a nonvanishing best-approximation error for estimators restricted to the sample-dependent kernel span. Merely being non-polynomial at each dimension does not imply such a uniform lower bound.
\item[(3)] 
They claimed that "kernel methods can fit at most a degree-$\ell$ polynomial".
\end{itemize}
We note that they merely assume the regression function falls into the square-integrable function space, which is broader than the smoothness classes commonly studied in nonparametric regression. 
In statistical applications, one often considers subspaces of the square-integrable function space that possess better properties.
For instance, \cite{Stone_Additive_1985} and \cite{Stone_Polynomial_1994} prove the optimality of additive regression and polynomial splines by assuming that the regression functions are square-integrable with specific smoothness conditions.
Moreover, when dealing with kernel methods, researchers often assume that the regression function falls into the RKHS associated with the kernel \cite{Caponnetto2006OptimalRF, caponnetto2007optimal, caponnetto2010cross}, rather than imposing only square-integrability.
Here, we instead assume that the regression function falls into the RKHS $\mathcal{H}^{\inner}$.
By adapting empirical-process tools and bounding the covering numbers of $\mathcal{H}^{\inner}$, we obtain convergence-rate guarantees for early-stopped kernel regression when $n \asymp d^{\gamma}$ for $\gamma>0$.
\old{In contrast, tools of the empirical process do not apply to the square-integrable function class, since the covering number of the square-integrable function class is unbounded.
Therefore, for the square-integrable function class, it is difficult to attain optimality results of kernel regression in high dimensions.}\revised{The full $L^2$ unit ball lacks finite covering entropy in the $L^2$ metric, so the entropy arguments used for an RKHS ball cannot be transferred to that class without further restrictions. This difference in function classes is essential when comparing rates.}
\begin{remark}
    For functions in the unit ball of $[\mathcal H]^s$ with $s>0$, the spherical-harmonic expansion gives
    $\|\mathrm P_{>\ell}f_\star\|_{L^2}^2\leq
    (\sup_{k>\ell}\mu_k^s)\|f_\star\|_{[\mathcal H]^s}^2$.
    In particular, parity monotonicity implies
    $\sup_{k>\ell}\mu_k\leq\max\{\mu_{\ell+1},\mu_{\ell+2}\}$;
    replacing this maximum by $\mu_{\ell+1}$ alone would require an additional cross-parity ordering assumption.
    For fixed $\ell$ and $s>0$, this spectral tail bound vanishes as
    $d\to\infty$ under the fixed-degree eigenvalue asymptotics and parity
    monotonicity. Theorem~4 of \cite{ghorbani2021linearized} provides an
    asymptotic characterization of the prediction risk in terms of the
    high-degree energy $\|\mathrm P_{>\ell}f_\star\|_{L^2}^2$.
    Its proof also contains quantitative estimates for the kernel-matrix
    approximation and related error terms. In the polynomial sample-size
    regime, these estimates can be used to extract an upper bound with the
    characteristic dependence
    $O\!\left(\log d\,[n/d^{\lfloor\gamma\rfloor}
    +d^{\lceil\gamma\rceil}/n]\right)$ for the corresponding
    approximation remainder (under the assumptions and normalization of
    that work). This quantitative information should be distinguished from
    the asymptotic $o_{d,\mathbb P}(1)$ formulation of its main theorem.
    Thus, it would be inaccurate to
    suggest that this work yields no convergence-rate information. The
    displayed theorem, however, controls the remainder by an
    $o_{d,\mathbb P}(1)$ term (for bounded signal and noise levels); this
    statement alone does not yield a sharp rate for the vanishing risk over
    the RKHS or interpolation-space balls considered here. Obtaining such
    rates requires tracking the dimension-dependent remainder together with
    the source-class spectral tail, and matching the resulting upper bounds
    with minimax lower bounds. Our results address this latter question for
    early-stopped kernel regression under the stated spherical-design and
    inner-product-kernel assumptions.

\end{remark}
\begin{figure*}[t!]
        \begin{subfigure}[t]{0.32\textwidth}
        \includegraphics[width=1.8in]{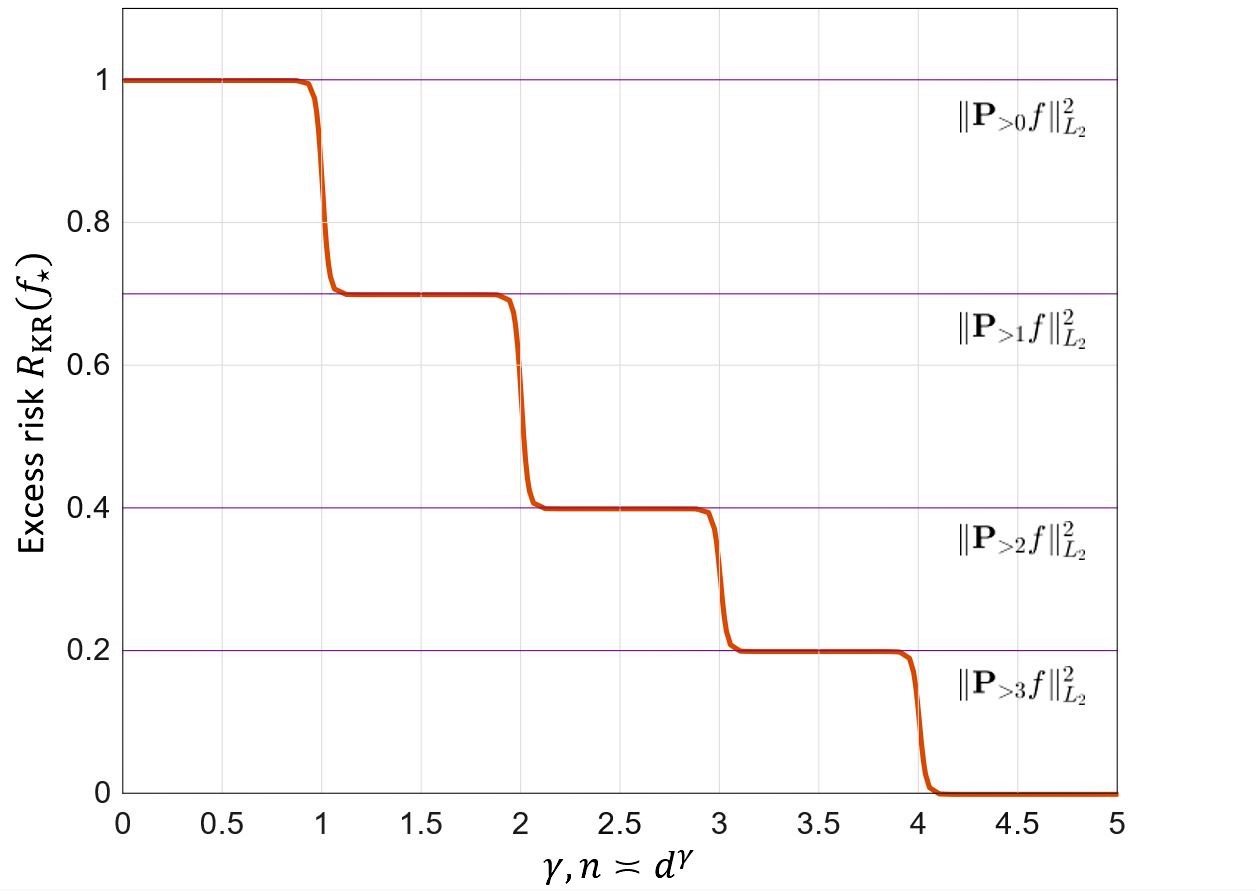}
        \captionsetup{width=0.7\textwidth}
        \caption*{(a)}
   \label{fig:montanari_cartoon}
    \end{subfigure}
    \begin{subfigure}[t]{0.32\textwidth}
           \includegraphics[width=1.8in, keepaspectratio]{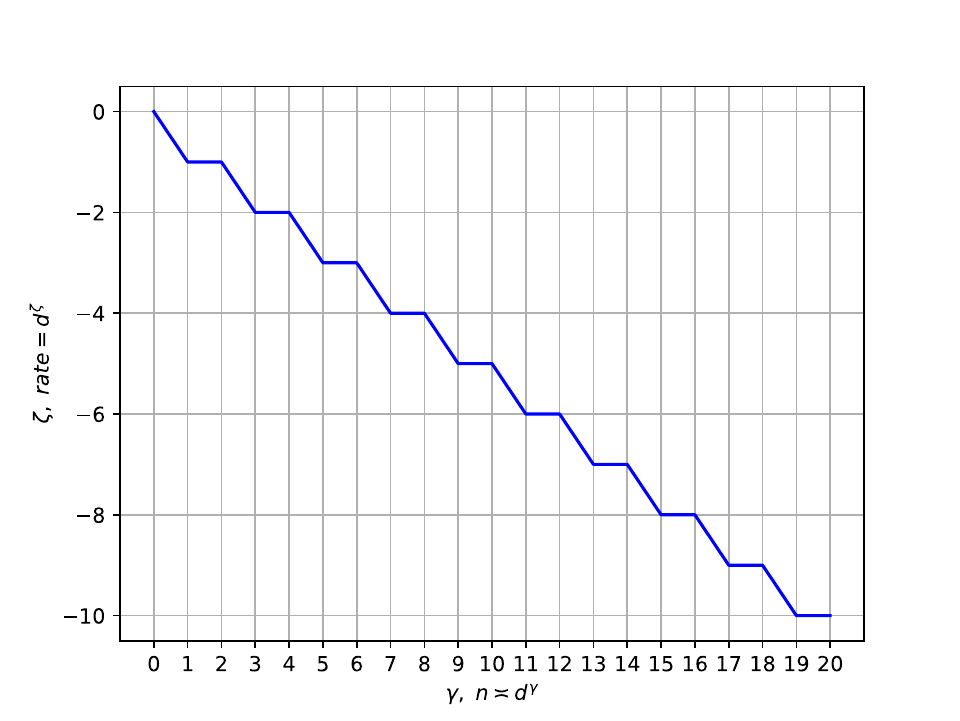}
    \captionsetup{width=0.7\textwidth} 
        \caption*{(b)}
    \end{subfigure}
    \begin{subfigure}[t]{0.32\textwidth}
           \includegraphics[width=1.8in, keepaspectratio]{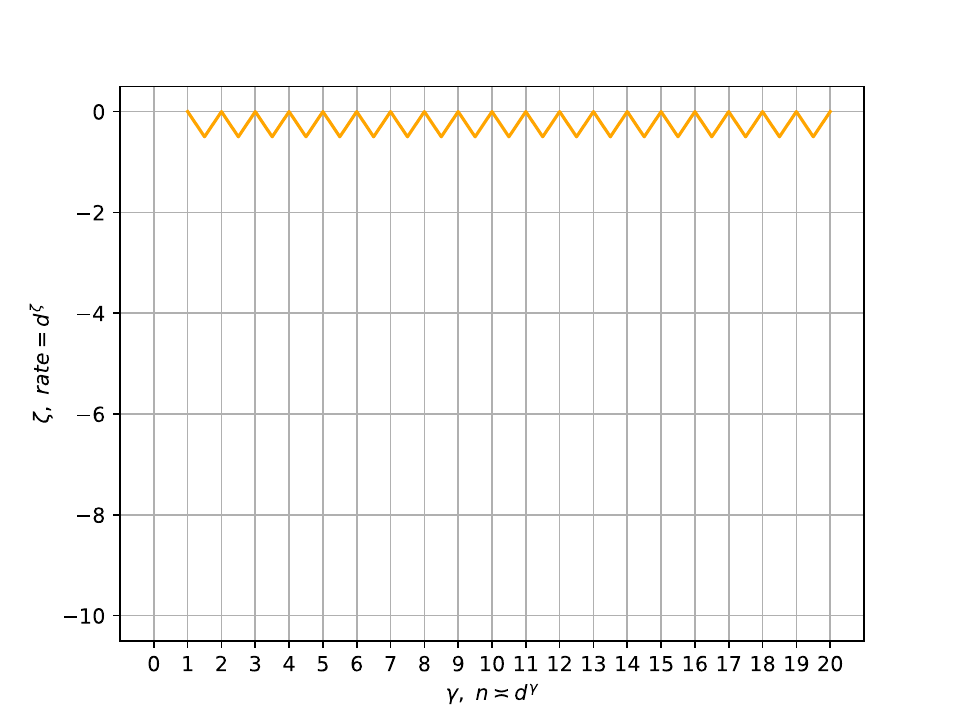}
    \captionsetup{width=0.7\textwidth} 
    \caption*{(c)}
    \end{subfigure}
\caption{
{\bf (a)} A cartoon of the excess risk of kernel ridge regression when $f_{\star}$ is square-integrable. Borrowed from \cite{ghorbani2021linearized}.
{\bf (b)} The excess risk of early-stopped kernel regression when $f_{\star} \in \calH^{\inner}$. Obtained from Theorem \ref{thm:near_lower_inner_large_d} and Theorem \ref{thm:near_upper_inner_large_d}.
{\bf (c)} A conditional expected-risk upper-bound exponent for kernel interpolation under the assumptions of \cite{liang2020multiple}; it is not a risk lower bound and does not use the same design assumptions as panel (b).
    } 
    \label{fig:3_comparison}
\end{figure*}
\subsection{Comparison with kernel interpolation in high dimensions}\label{sec:5.2_comparison_liang_multiple}
Recent findings reported in \cite{li2023kernel} indicate that kernel interpolation exhibits poorer generalization compared to early-stopped kernel regression in fixed dimensions.
\old{In this subsection, we provide evidence that kernel interpolation can generalize less effectively than kernel regression in high dimensions.}\revised{In this subsection, we compare published risk upper bounds for minimum-norm kernel interpolation with our bounds for early-stopped regression.}
The following proposition summarizes the conditional expected-risk upper bound from the informal Theorem 1 of \cite{liang2020multiple}, whose assumptions are specified through Theorems 2 and 3 of that paper. In particular, it is not a bound for arbitrary inner-product kernels or arbitrary design distributions.
\begin{proposition}[Consequence of Theorems 2--3 of \cite{liang2020multiple}]\label{restate_thm_1_in_liang}
Let $\ell\geq1$ be fixed. Consider the minimum-norm kernel interpolant under the design, kernel, noise, and source-function assumptions of Theorems 2 and 3 of \cite{liang2020multiple}; in particular, their design is an i.i.d. product distribution with sub-Gaussian coordinates, the kernel is $k(x,z)=h(x^\top z/d)$ with the required positive Taylor coefficients (including a positive coefficient of degree at least $2\ell+3$), the conditional noise variance is uniformly bounded, and $f_\star$ admits the integral representation with an $L^4$-bounded source density and bounded kernel diagonal. Assume $d^\ell\log d\lesssim n\lesssim d^{\ell+1}$. For any fixed $\delta\in(0,1)$, there is a constant $C_\delta$, independent of $n,d$, such that, for sufficiently large $n$, with probability at least $1-\delta-\exp\{-\Omega(n/d^\ell)\}$ over the design,
\[
\mathbb E\bigl[\|f_\infty-f_\star\|_{L^2}^2\mid\boldsymbol X\bigr]
\le C_\delta\left(\frac{d^\ell}{n}+\frac{n}{d^{\ell+1}}\right).
\]
In the polynomial scaling $n\asymp d^\gamma$ with $\ell<\gamma<\ell+1$, the displayed right-hand side is $O(n^{-\eta(\gamma)})$, where
$\eta(\gamma)=\min\{(\ell+1)/\gamma-1,\ 1-\ell/\gamma\}$.
\end{proposition}
\old{Let's compare the results presented in Proposition \ref{restate_thm_1_in_liang} with the findings stated in Theorem \ref{thm:near_upper_inner_large_d}:}\revised{The exponents in Proposition~\ref{restate_thm_1_in_liang} and Theorem~\ref{thm:near_upper_inner_large_d} can be compared as follows:}
\begin{itemize}
    \item[(1)] For fixed $\ell\geq1$ and $\ell<\gamma<\ell+1$, the exponent satisfies $\eta(\gamma)\leq 1/(2\ell+1)\leq1/3$. Therefore, when $n\asymp d^{\gamma}$ for $\gamma>1$, the exponent in the quoted conditional expected-risk upper bound for kernel interpolation is smaller than the exponent in our high-probability upper bound for early-stopped kernel regression. Since the bounds concern different risk functionals and are both upper bounds, this comparison alone does not order the actual risks.
    \item[(2)] For $0< \gamma \leq 1$, the proved high-probability upper bound for early-stopped kernel regression is of order $(\log n)/n$, while the quoted interpolation theorem does not provide a corresponding rate in this regime.
\end{itemize}
\old{Moreover, in Figure \ref{fig:3_comparison}, we plot the upper bound results in \cite{liang2020multiple} for kernel interpolation, represented by the orange line, together with the upper and lower bound results in Theorem \ref{thm:near_lower_inner_large_d} and Theorem \ref{thm:near_upper_inner_large_d} for kernel regression, represented by the blue line. We can observe that the displayed upper-bound exponent for kernel regression is more favorable than the available upper-bound exponent for interpolation for all $\gamma > 1$. These upper bounds alone do not establish a lower bound on the risk of interpolation. Motivated by the comparison, we { may conjecture} that kernel interpolation ($t = \infty$) generalizes substantially worse than
early-stopped kernel regression ($t = \widehat{T}<\infty$) in high dimensions.}\revised{Figure~\ref{fig:3_comparison} juxtaposes the available rate descriptions. The interpolation curve is an upper-bound exponent under the assumptions of \cite{liang2020multiple}, whereas our curve reflects minimax lower bounds and early-stopping upper bounds for a spherical RKHS ball. The plots therefore compare guarantees, not the actual risks of the two procedures under a common model. They motivate, but do not establish, a performance separation between interpolation and early stopping.}
{
\begin{remark}
    In subsequent work \cite{zhang2024optimal, zhang2024phase}, we determine the optimal rate on the excess risk of kernel ridge regression and kernel ridgeless regression. These works support our conjecture that early-stopped kernel regression is minimax optimal while kernel ridge(-less) regression might not be.
\end{remark}
}
\subsection{Numerical Experiments}
\old{In this subsection, we numerically illustrate the rates suggested by our theory when $n \asymp d^{\gamma}$ for some fixed $\gamma>0$, and considering functions $f_{\star}$ in $\calH$ with bounded norms, 
the early-stopping
kernel regression algorithm, defined in (\ref{solution:gradient:flow}), exhibits related convergence behavior, and compare its empirical prediction error with that of kernel interpolation when $\gamma>1$.}
\revised{We examine the numerical behavior of early-stopped kernel regression in the polynomial scaling regime $n\asymp d^\gamma$ and compare it with kernel interpolation. The experiments use two kernels with different relationships to Assumption~\ref{assu:coef_of_inner_prod_kernel}: the RBF kernel satisfies the required positive-coefficient and smoothness conditions, whereas the two-layer ReLU NTK does not. The NTK experiments are therefore exploratory and are not presented as a numerical verification of Theorems~\ref{thm:near_lower_inner_large_d} and~\ref{thm:near_upper_inner_large_d}.}
We consider the following two inner-product kernels:
\begin{itemize}
    \item The neural tangent kernel of a two-layer ReLU neural network:
    \begin{displaymath}
        K^{\NTK}(x, x^\prime) := \Phi(\langle x, x^{\prime} \rangle), ~~x, x^{\prime} \sim \mathbb{S}^{d}.
    \end{displaymath}
    where $\Phi(t)=\left[\sin{(\arccos t)}+2(\pi-\arccos t)t\right]/ (2 \pi)$.
    \item The RBF kernel with a fixed bandwidth:
    \begin{displaymath}
        K^{\mathrm{rbf}}(x,x^{\prime}) = \exp{\left(-\frac{\|x-x^{\prime}\|_{2}^{2}}{2}\right)}, ~~x, x^{\prime} \sim \mathbb{S}^{d}.
    \end{displaymath}
\end{itemize}
\old{The two kernels are treated on the same footing as numerical illustrations.}
\revised{On the unit sphere, the RBF kernel equals $K^{\mathrm{rbf}}(x,x')=e^{-1}e^{\langle x,x'\rangle}$, so its Taylor coefficients are $a_j=e^{-1}/j!>0$ and its profile is smooth on $[-1,1]$. Thus it satisfies Assumption~\ref{assu:coef_of_inner_prod_kernel} (and the bounded-diagonal condition). By contrast, the ReLU NTK profile has odd part $[\Phi(t)-\Phi(-t)]/2=t/2$, so its coefficients of odd degree $j\ge3$ vanish; it also fails the $C^\infty([-1,1])$ requirement at the endpoints. The NTK curves illustrate whether similar behavior occurs beyond the theorem's assumptions. In both experiments, the target is a finite sum of kernel sections and therefore belongs to the corresponding RKHS, but its norm is not normalized to the unit ball; accordingly, we compare slopes rather than claim a direct verification of the unit-ball risk bounds.}
For any dimension $d$, let $\rho_{\calX}$ be the uniform distribution on $\mathcal X=\mathbb S^d$.
We construct a function $f_{\star}$ in $\calH$ as follows:
\begin{equation}\label{experiment true function}
    f^{*}(x) = k(x, u_{1}) + k(x, u_{2}) + k(x, u_{3}),
\end{equation}
where $k \in \{K^{\NTK}, K^{\mathrm{rbf}}\}$, and
$u_1$, $u_2$, and $u_3$ are sampled from $\rho_{\calX}$.
Then, we consider the data generation process with the model given by Equation (\ref{equation:true_model}), which can be expressed as:
\begin{equation}
    y=f_{\star}(\bm{x})+\epsilon,
\end{equation}
where $\epsilon \sim \mathcal{N}(0, 1)$.
We construct the estimators of the kernel 
 regression and kernel interpolation (KI) ${f}_{T}$ and ${f}_{t_\infty}$ using Equation (\ref{solution:gradient:flow}), where the numerical stopping time is chosen deterministically as $T=Cn^{1/2}$ for a constant $C$, and $t_\infty=\infty$. This is not the data-dependent stopping rule $\widehat T$ in Theorem~\ref{thm:near_upper_inner_large_d}; accordingly, these experiments illustrate the predicted scaling but are not a direct numerical verification of that stopping rule.
We consider four different settings to simulate results under different asymptotic frameworks of $n\asymp d^{\gamma}$, where $\gamma>0$:
\begin{itemize}
    \item $\gamma = 0.5:$ $n$ from 100 to 200, with intervals 5, $d = n^{2}$.
    \item $\gamma = 0.8:$ $n$ from 500 to 1000, with intervals 10, $d = n^{5/4}$.
    \item $\gamma = 1.5:$ $n$ from 1000 to 5000, with intervals 200, $d = n^{2/3}$.
    \item $\gamma = 1.8:$ $n$ from 1000 to 5000, with intervals 200, $d = n^{5/9}$.
\end{itemize}
We numerically approximate the excess risk $\|{f}_{t} - f_{\star}\|_{L^2}^2$ by $\sum_{i=1}^{N}({f}_{t}(z_i) - f_{\star}(z_i))^2/N$, where $N=1000$ and $z_i$'s are test data drawn i.i.d. from $\rho_{\calX}$.
For each combination of $(n, d)$, we repeat the experiments $20$ times and compute the average excess risk.
To visualize the convergence rate $r$, we perform
logarithmic least-squares $\log \text{risk} = r \log n + b$ to fit the excess risk with respect to
the sample size and display the value of $r$.
\begin{figure*}[h!]
    \begin{subfigure}[t]{0.48\textwidth}
           \includegraphics[width=2.8in, keepaspectratio]{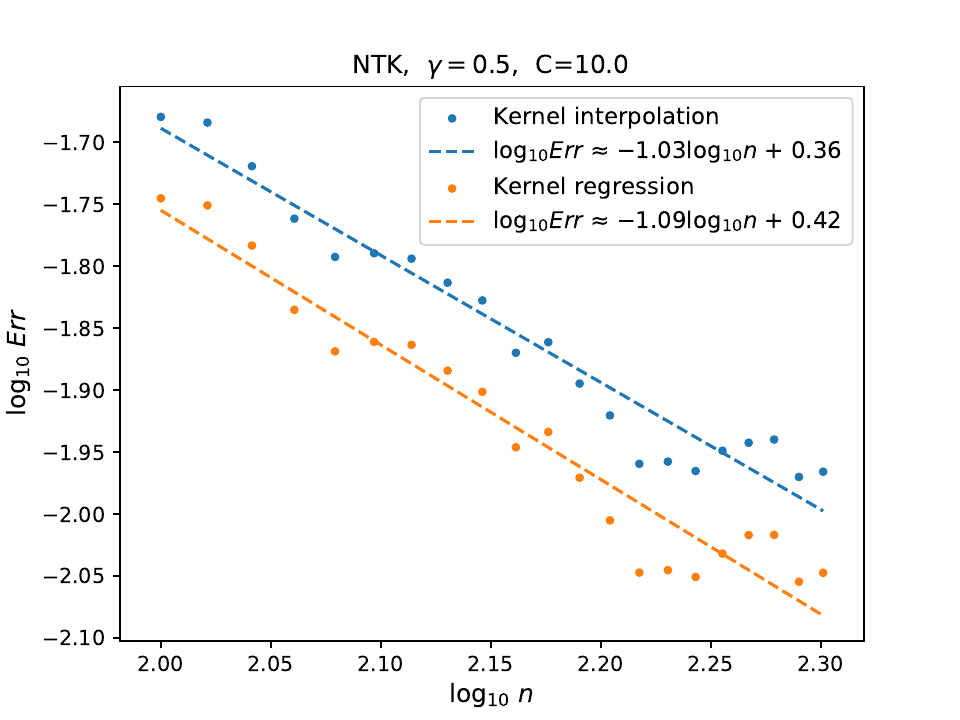}
\old{\caption{$n = d^{0.5}$, kernel regression theoretical rate $=n^{-1}$}}\revised{\caption{$n = d^{0.5}$; reference polynomial exponent $n^{-1}$}}
    \end{subfigure}
    \begin{subfigure}[t]{0.48\textwidth}
            \includegraphics[width=2.8in, keepaspectratio]{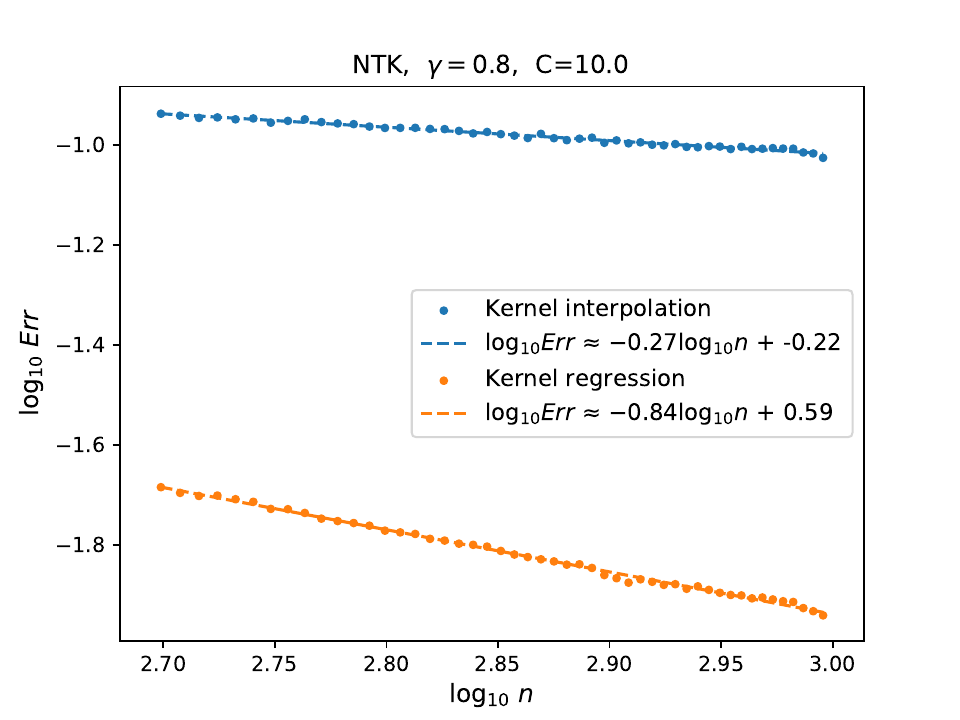}
\old{\caption{$n = d^{0.8}$, kernel regression theoretical rate $=n^{-1}$}}\revised{\caption{$n = d^{0.8}$; reference polynomial exponent $n^{-1}$}}
    \end{subfigure}
        \begin{subfigure}[t]{0.48\textwidth}
            \includegraphics[width=2.8in, keepaspectratio]{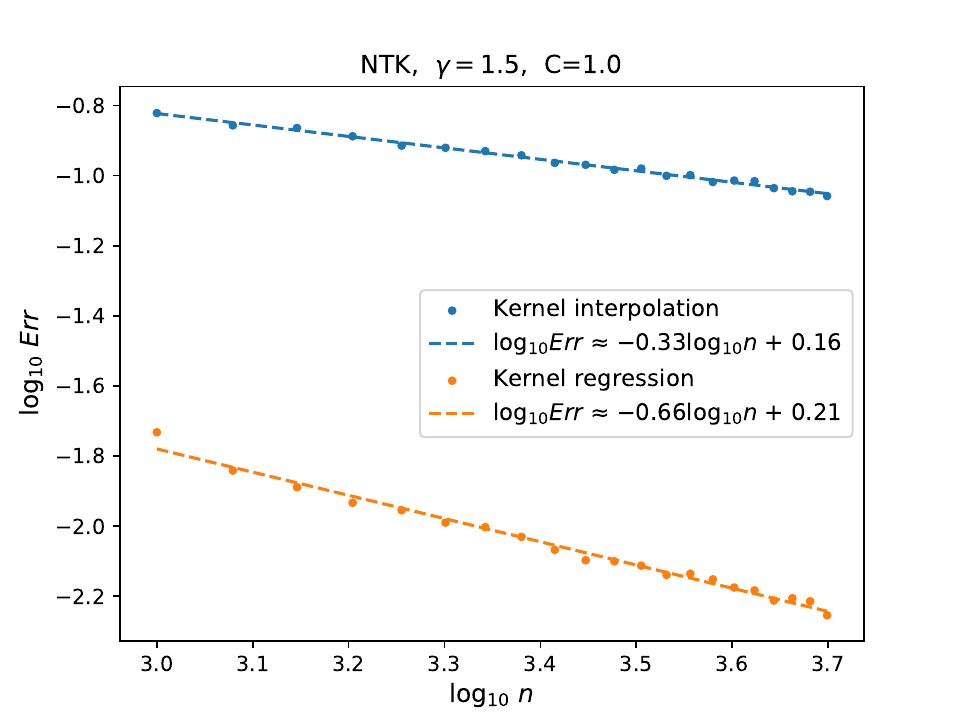}
\old{\caption{$n = d^{1.5}$, kernel regression theoretical rate $=n^{-2/3}$}}\revised{\caption{$n = d^{1.5}$; reference polynomial exponent $n^{-2/3}$}}
    \end{subfigure}
        \begin{subfigure}[t]{0.48\textwidth}
            \includegraphics[width=2.8in, keepaspectratio]{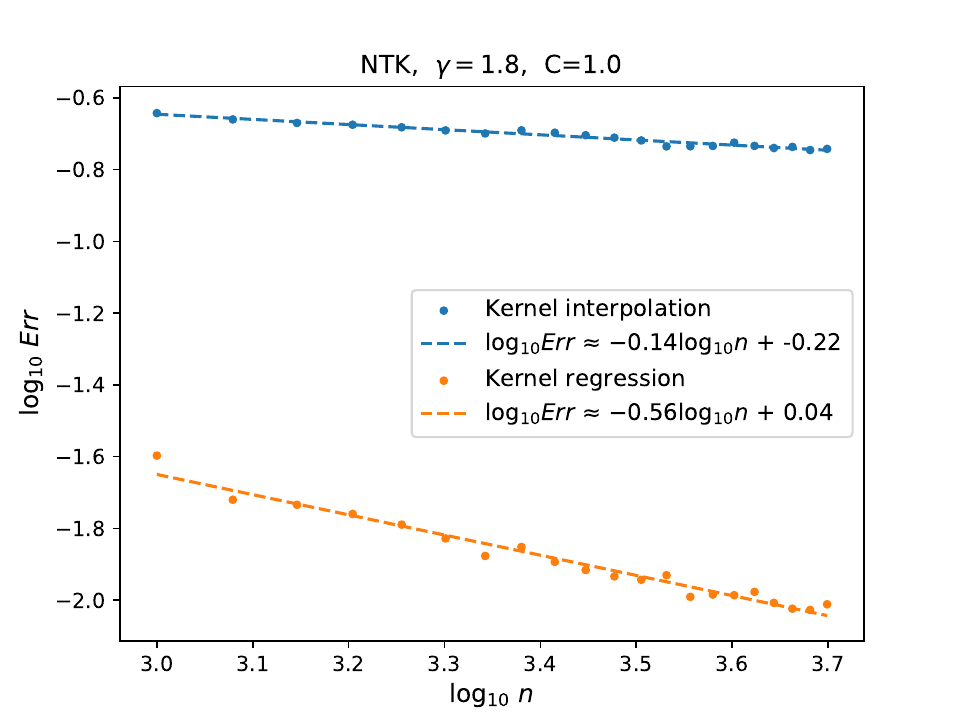}
\old{\caption{$n = d^{1.8}$, kernel regression theoretical rate $=n^{-5/9}$}}\revised{\caption{$n = d^{1.8}$; reference polynomial exponent $n^{-5/9}$}}
    \end{subfigure}
    \caption{
    Log excess risk decay curves for the ReLU NTK (outside Assumption~\ref{assu:coef_of_inner_prod_kernel}) under different asymptotic frameworks $n \asymp d^\gamma$.
    The points represent the average excess risks computed from 20 trials.
    The dashed lines are obtained through logarithmic least-squares regression, with the slopes indicating the convergence rates denoted as $r$.
    The four sub-figures from left to right and from top to bottom correspond to the settings where $n$ is set to be equal to $d^{0.5}$, $d^{0.8}$, $d^{1.5}$, and $d^{1.8}$ respectively. In each setting, the constant $C$ is chosen from $\{0.001, 0.01, 0.1, 1, 10, 100, 1000\}$, and we report our numerical results under the best choice of $C$.
    }
    \label{fig:3_1}
\end{figure*}
\begin{figure*}[h!]
    \begin{subfigure}[t]{0.48\textwidth}
           \includegraphics[width=2.8in, keepaspectratio]{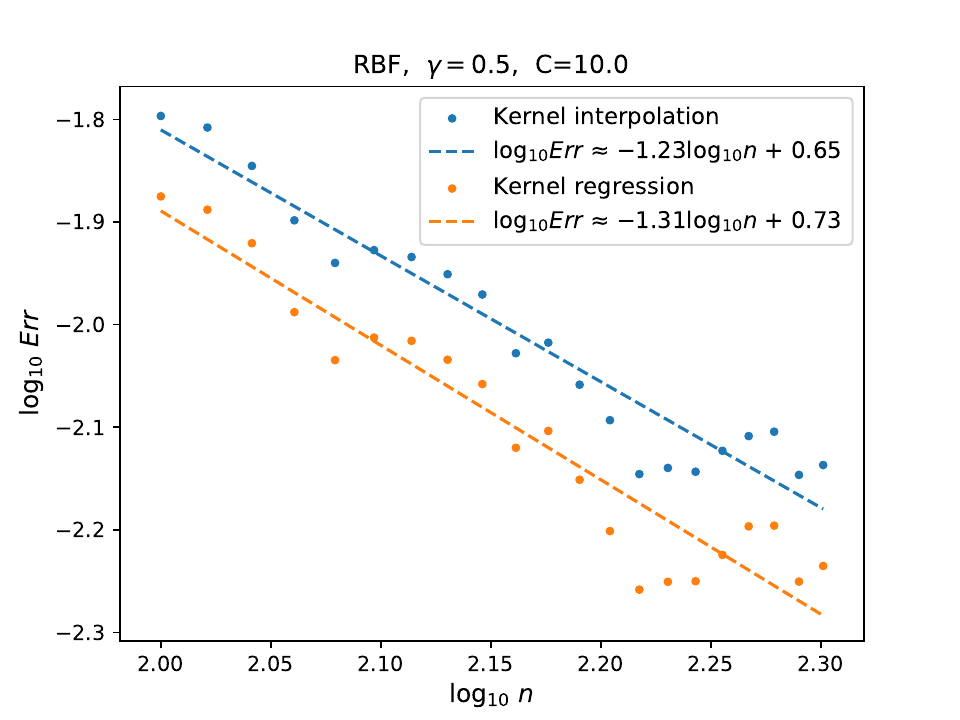}
\old{\caption{$n = d^{0.5}$, kernel regression theoretical rate $=n^{-1}$}}\revised{\caption{$n = d^{0.5}$; reference polynomial exponent $n^{-1}$}}
    \end{subfigure}
    \begin{subfigure}[t]{0.48\textwidth}
            \includegraphics[width=2.8in, keepaspectratio]{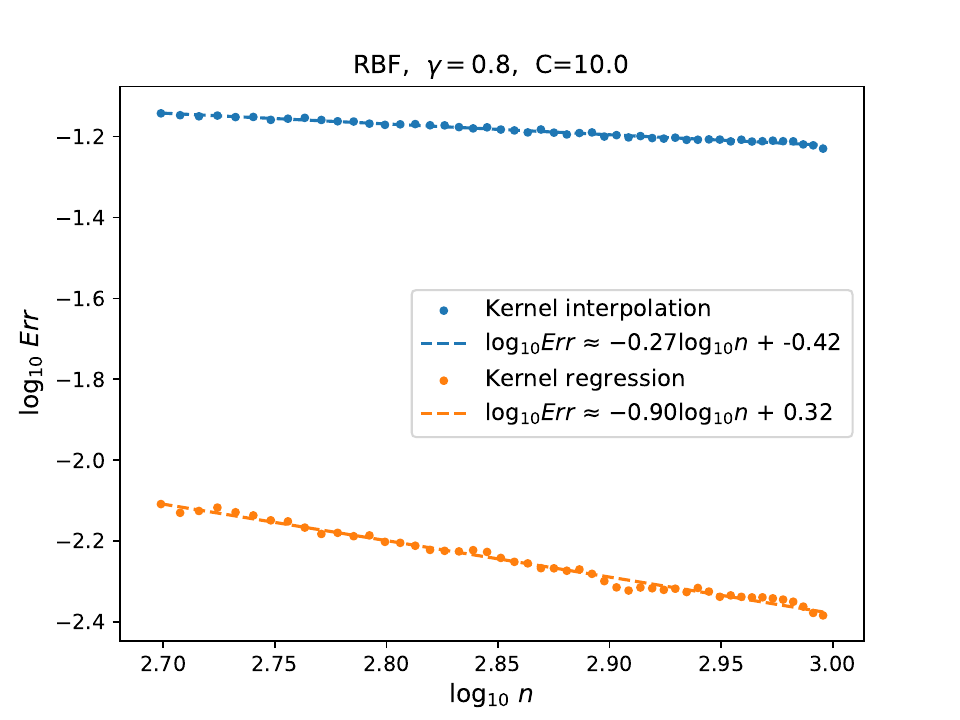}
\old{\caption{$n = d^{0.8}$, kernel regression theoretical rate $=n^{-1}$}}\revised{\caption{$n = d^{0.8}$; reference polynomial exponent $n^{-1}$}}
    \end{subfigure}
        \begin{subfigure}[t]{0.48\textwidth}
            \includegraphics[width=2.8in, keepaspectratio]{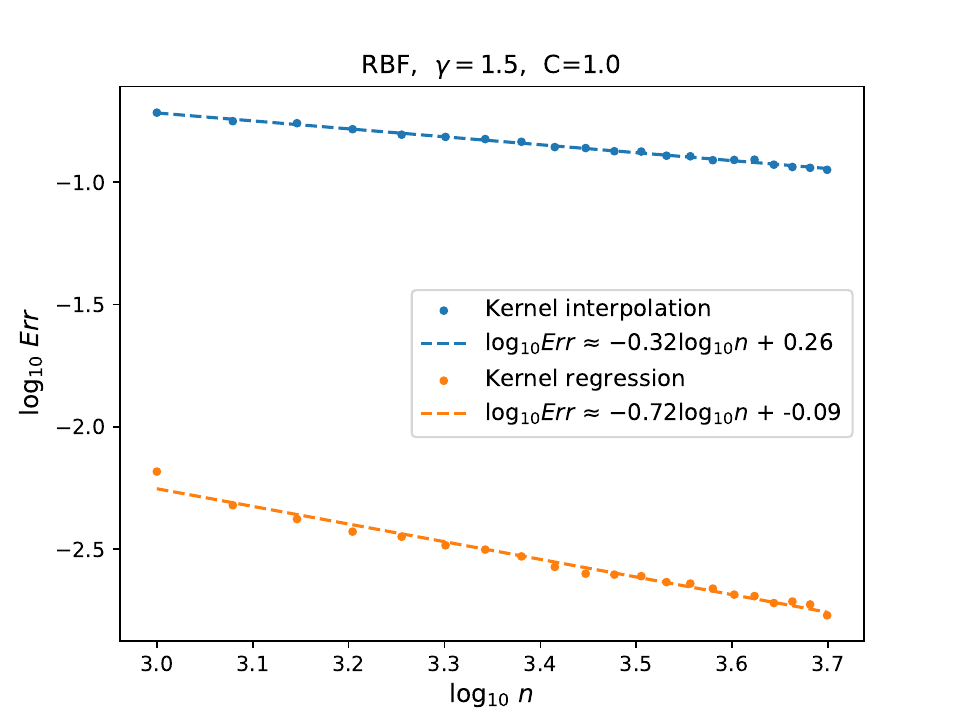}
\old{\caption{$n = d^{1.5}$, kernel regression theoretical rate $=n^{-2/3}$}}\revised{\caption{$n = d^{1.5}$; reference polynomial exponent $n^{-2/3}$}}
    \end{subfigure}
        \begin{subfigure}[t]{0.48\textwidth}
            \includegraphics[width=2.8in, keepaspectratio]{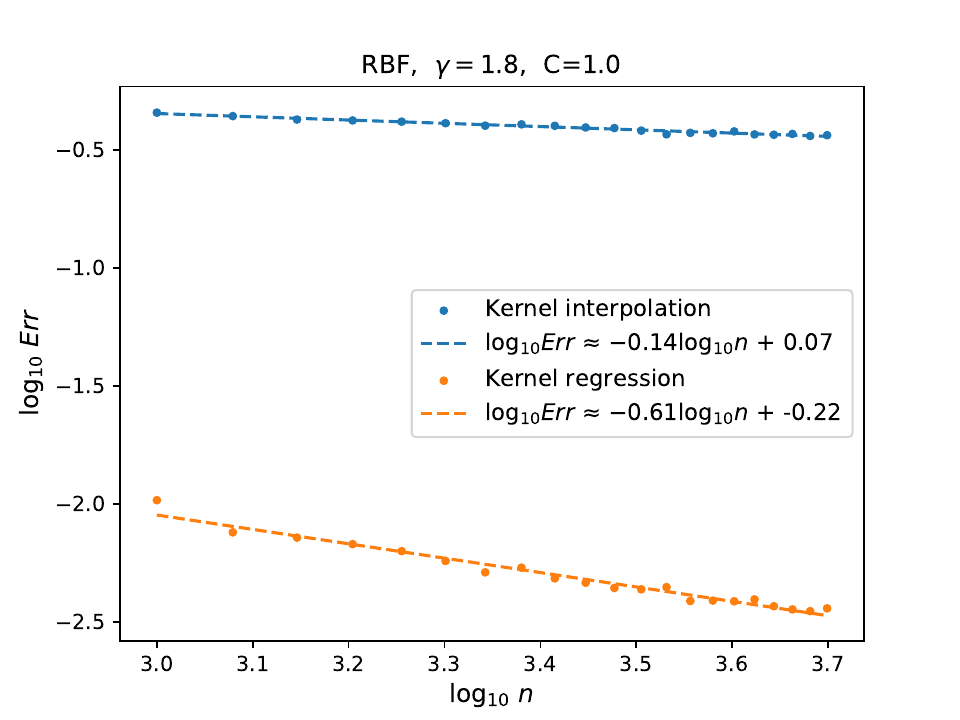}
\old{\caption{$n = d^{1.8}$, kernel regression theoretical rate $=n^{-5/9}$}}\revised{\caption{$n = d^{1.8}$; reference polynomial exponent $n^{-5/9}$}}
    \end{subfigure}
    \caption{
    \old{A similar plot as Figure \ref{fig:3_1}, but with the RBF kernel.}\revised{The same numerical comparison for the RBF kernel, which satisfies Assumption~\ref{assu:coef_of_inner_prod_kernel}. The displayed reference rates indicate polynomial exponents; logarithmic factors are omitted.}
    }
    \label{fig:3_2}
\end{figure*}
We try different values of the constant 
$C \in \{0.001, 0.01, 0.1, 1, 10, 100, 1000\}$ 
for the numerical stopping time $T=C\sqrt n$, and we report our numerical results in Figure \ref{fig:3_1} and Figure \ref{fig:3_2} under the best choice of $C$. \old{For each setting, we observe that the convergence rates of the excess risk in kernel regression algorithms are consistently close to the theoretical rate as given in Theorem \ref{thm:near_lower_inner_large_d} and Theorem \ref{thm:near_upper_inner_large_d}.}\revised{For the RBF kernel, the fitted slopes are broadly consistent with the polynomial exponents suggested by Theorems~\ref{thm:near_lower_inner_large_d} and~\ref{thm:near_upper_inner_large_d}; the logarithmic factor and the use of a deterministic stopping time preclude a direct test of the high-probability theorem. For the ReLU NTK, we report analogous empirical trends without claiming that the theorems apply.}
Moreover, we find that KI is comparable to kernel regression when $\gamma = 0.5$, and is worse than kernel regression when $\gamma = 0.8, 1.5$, or $1.8$.

\section{Conclusion and Future Work}\label{sec:conclusion}
\revised{We studied early-stopped kernel regression in the high-dimensional regime $n\asymp d^\gamma$. General complexity and entropy arguments yield an upper bound for early stopping and an expected minimax lower bound. For smooth inner-product kernels under uniform spherical design, the refined analysis identifies the polynomial exponents of the risk bounds across all $\gamma>0$. The bounds match at order $n^{-1/2}$ for positive even $\gamma$; in the other regimes, the polynomial exponents agree subject to the logarithmic and arbitrarily small polynomial qualifications in Theorems~\ref{thm:near_lower_inner_large_d} and~\ref{thm:near_upper_inner_large_d}.

The resulting exponents exhibit multiple descents and periodic plateaus, reflecting the dimension-dependent eigenvalues and multiplicities of spherical-harmonic blocks. These conclusions rely on the stated kernel and design assumptions. The RBF experiments illustrate behavior within the smooth positive-coefficient kernel class, while the ReLU NTK experiments explore behavior outside that class and do not directly verify the theorems. Extending the analysis to kernels with missing spectral blocks, other data distributions, and broader source classes remains an interesting direction for future work.}

\newpage
% Empty historical appendix environment removed.
\section*{Acknowledgements}
The authors gratefully acknowledge {\it the National Natural Science Foundation of China (Grant 11971257), Beijing Natural Science Foundation (Grant Z190001), National Key R\&D Program of China (2020AAA0105200), and Beijing
Academy of Artificial Intelligence.}
Part of the work in this paper was done while the authors visited the Center of Statistical Research, School of Statistics, Southwestern University of Finance and Economics.
The authors would like to thank the anonymous referees, the Associate Editor, and the Editor for their constructive comments that improved the quality of this paper. 
\revised{AI-assisted tools were used to assist with mathematical proof development, including identifying and correcting errors in proofs, developing alternative proof arguments, and verifying mathematical derivations. AI-assisted tools were also used for language polishing and checking the correctness and consistency of mathematical proofs.}

\bibliographystyle{plain} 
\bibliography{reference} 

@misc{jianfa2022generalization,
  doi = {10.48550/ARXIV.2302.05933},
  url = {https://arxiv.org/abs/2302.05933},
  author = {Lai, Jianfa and Xu, Manyun and Chen, Rui and Lin, Qian},
  title = {Generalization Ability of Wide Neural Networks on $\mathbb{R}$},
  publisher = {arXiv},
  year = {2023},
  copyright = {Creative Commons Attribution 4.0 International}
}

@article{raskutti2014early,
  author  = {Garvesh Raskutti and Martin J. Wainwright and Bin Yu},
  title   = {Early Stopping and Non-parametric Regression: An Optimal Data-dependent Stopping Rule},
  journal = {Journal of Machine Learning Research},
  year    = {2014},
  volume  = {15},
  number  = {11},
  pages   = {335--366},
  url     = {http://jmlr.org/papers/v15/raskutti14a.html}
}

@article{Jacot_NTK_2018,
  title={Neural tangent kernel: Convergence and generalization in neural networks},
  author={Jacot, Arthur and Gabriel, Franck and Hongler, Cl{\'e}ment},
  journal={Advances in Neural Information Processing Systems},
  volume={31},
  year={2018}
}

@article{Bietti_deep_2021,
  title={Deep equals shallow for ReLU networks in kernel regimes},
  author={Bietti, Alberto and Bach, Francis},
  journal={arXiv preprint arXiv:2009.14397},
  year={2020}
}

@article{Du_gradient_2019_a,
  title={Gradient descent provably optimizes over-parameterized neural networks},
  author={Du, Simon S and Zhai, Xiyu and Poczos, Barnabas and Singh, Aarti},
  journal={arXiv preprint arXiv:1810.02054},
  year={2018}
}

@inproceedings{Du_gradient_2019_b,
  title={Gradient descent finds global minima of deep neural networks},
  author={Du, Simon and Lee, Jason and Li, Haochuan and Wang, Liwei and Zhai, Xiyu},
  booktitle={International Conference on Machine Learning},
  pages={1675--1685},
  year={2019},
  organization={PMLR}
}

@article{Wright_bound_1973,
author = {F. T. Wright},
title = {{A Bound on Tail Probabilities for Quadratic Forms in Independent Random Variables Whose Distributions are not Necessarily Symmetric}},
volume = {1},
journal = {The Annals of Probability},
number = {6},
publisher = {Institute of Mathematical Statistics},
pages = {1068 -- 1070},
year = {1973},
doi = {10.1214/aop/1176996815},
URL = {https://doi.org/10.1214/aop/1176996815}
}

@inproceedings{Hu_Regularization_2021,
  title={Regularization matters: A nonparametric perspective on overparametrized neural network},
  author={Hu, Tianyang and Wang, Wenjia and Lin, Cong and Cheng, Guang},
  booktitle={International Conference on Artificial Intelligence and Statistics},
  pages={829--837},
  year={2021},
  organization={PMLR}
}

@inproceedings{Mendelson_Geometric_2002,
  author = {Mendelson, Shahar},
  title = {Geometric parameters of kernel machines},
  booktitle = {Computational Learning Theory},
  series = {Lecture Notes in Artificial Intelligence},
  volume = {2375},
  pages = {29--43},
  publisher = {Springer},
  year = {2002},
  address = {Berlin},
}

@article{Koltchinskii_Local_2006,
author = {Vladimir Koltchinskii},
title = {{Local Rademacher complexities and oracle inequalities in risk minimization}},
volume = {34},
journal = {The Annals of Statistics},
number = {6},
publisher = {Institute of Mathematical Statistics},
pages = {2593 -- 2656},
year = {2006},
doi = {10.1214/009053606000001019},
URL = {https://doi.org/10.1214/009053606000001019}
}

@article{Yang_Density_1999,
author = {Yuhong Yang and Andrew Barron},
title = {{Information-theoretic determination of minimax rates of convergence}},
volume = {27},
journal = {The Annals of Statistics},
number = {5},
publisher = {Institute of Mathematical Statistics},
pages = {1564 -- 1599},
year = {1999},
doi = {10.1214/aos/1017939142},
URL = {https://doi.org/10.1214/aos/1017939142}
}

@inproceedings{Laurent_Estimating_2015,
    title={Estimating local intrinsic dimensionality},
  author={Amsaleg, Laurent and Chelly, Oussama and Furon, Teddy and Girard, St{\'e}phane and Houle, Michael E and Kawarabayashi, Ken-ichi and Nett, Michael},
  booktitle={Proceedings of the 21th ACM SIGKDD International Conference on Knowledge Discovery and Data Mining},
  pages={29--38},
  year={2015}
}

@article{Fukunaga_Algorithm_1971,
  title={An algorithm for finding intrinsic dimensionality of data},
  author={Fukunaga, Keinosuke and Olsen, David R},
  journal={IEEE Transactions on Computers},
  volume={100},
  number={2},
  pages={176--183},
  year={1971},
  publisher={IEEE}
}

@article{Bietti_on_2019,
  title={On the inductive bias of neural tangent kernels},
  author={Bietti, Alberto and Mairal, Julien},
  journal={Advances in Neural Information Processing Systems},
  volume={32},
  year={2019}
}

@article{ILSVRC15,
  title={Imagenet large scale visual recognition challenge},
  author={Russakovsky, Olga and Deng, Jia and Su, Hao and Krause, Jonathan and Satheesh, Sanjeev and Ma, Sean and Huang, Zhiheng and Karpathy, Andrej and Khosla, Aditya and Bernstein, Michael and others},
  journal={International Journal of Computer Vision},
  volume={115},
  pages={211--252},
  year={2015},
  publisher={Springer}
}

@inproceedings{Devlin_BERT_2019,
  title = "{BERT}: Pre-training of Deep Bidirectional Transformers for Language Understanding",
    author = "Devlin, Jacob  and
      Chang, Ming-Wei  and
      Lee, Kenton  and
      Toutanova, Kristina",
    booktitle = "Proceedings of the 2019 Conference of the North {A}merican Chapter of the Association for Computational Linguistics: Human Language Technologies, Volume 1 (Long and Short Papers)",
    month = jun,
    year = "2019",
    address = "Minneapolis, Minnesota",
    publisher = "Association for Computational Linguistics",
    url = "https://aclanthology.org/N19-1423",
    doi = "10.18653/v1/N19-1423",
    pages = "4171--4186",
}

@article{Arora_on_2019,
  title={On exact computation with an infinitely wide neural net},
  author={Arora, Sanjeev and Du, Simon S and Hu, Wei and Li, Zhiyuan and Salakhutdinov, Russ R and Wang, Ruosong},
  journal={Advances in Neural Information Processing Systems},
  volume={32},
  year={2019}
}

@article{misiakiewicz_learning_2021,
  title={Learning with convolution and pooling operations in kernel methods},
  author={Misiakiewicz, Theodor and Mei, Song},
  journal={Advances in Neural Information Processing Systems},
  volume={35},
  pages={29014--29025},
  year={2022}
}

@article{Karoui_spectrum_2010,
author = {Noureddine El Karoui},
title = {{The spectrum of kernel random matrices}},
volume = {38},
journal = {The Annals of Statistics},
number = {1},
publisher = {Institute of Mathematical Statistics},
pages = {1 -- 50},
year = {2010},
doi = {10.1214/08-AOS648},
URL = {https://doi.org/10.1214/08-AOS648}
}

@article{misiakiewicz_spectrum_2022,
  title={Spectrum of inner-product kernel matrices in the polynomial regime and multiple descent phenomenon in kernel ridge regression},
  author={Misiakiewicz, Theodor},
  journal={arXiv preprint arXiv:2204.10425},
  year={2022}
}

@InProceedings{Donhauser_how_2021,
  title={How rotational invariance of common kernels prevents generalization in high dimensions},
  author={Donhauser, Konstantin and Wu, Mingqi and Yang, Fanny},
  booktitle={International Conference on Machine Learning},
  pages={2804--2814},
  year={2021},
  organization={PMLR}
}

@article{Liang_Just_2019,
author = {Tengyuan Liang and Alexander Rakhlin},
title = {{Just interpolate: Kernel “Ridgeless” regression can generalize}},
volume = {48},
journal = {The Annals of Statistics},
number = {3},
publisher = {Institute of Mathematical Statistics},
pages = {1329 -- 1347},
year = {2020},
doi = {10.1214/19-AOS1849},
URL = {https://doi.org/10.1214/19-AOS1849}
}

@inproceedings{Liu_kernel_2021,
  title={Kernel regression in high dimensions: Refined analysis beyond double descent},
  author={Liu, Fanghui and Liao, Zhenyu and Suykens, Johan},
  booktitle={International Conference on Artificial Intelligence and Statistics},
  pages={649--657},
  year={2021},
  organization={PMLR}
}

@article{caponnetto2007optimal,
  title={Optimal rates for the regularized least-squares algorithm},
  author={Caponnetto, Andrea and De Vito, Ernesto},
  journal={Foundations of Computational Mathematics},
  volume={7},
  number={3},
  pages={331--368},
  year={2007},
  publisher={Springer}
}

@inproceedings{Namjoon_Non_2022,
  author    = {Namjoon Suh and
               Hyunouk Ko and
               Xiaoming Huo},
  title     = {A Non-Parametric Regression Viewpoint : Generalization of Overparametrized
               Deep {RELU} Network Under Noisy Observations},
  booktitle = {The Tenth International Conference on Learning Representations, {ICLR}
               2022, Virtual Event, April 25-29, 2022},
  publisher = {OpenReview.net},
  year      = {2022},
  url       = {https://openreview.net/forum?id=bZJbzaj\_IlP},
  bibsource = {dblp computer science bibliography, https://dblp.org}
}

@article{Ghorbani_When_2021,
  title={When do neural networks outperform kernel methods?},
  author={Ghorbani, Behrooz and Mei, Song and Misiakiewicz, Theodor and Montanari, Andrea},
  journal={Advances in Neural Information Processing Systems},
  volume={33},
  pages={14820--14830},
  year={2020}
}

@inproceedings{mei2021learning,
  title={Learning with invariances in random features and kernel models},
  author={Mei, Song and Misiakiewicz, Theodor and Montanari, Andrea},
  booktitle={Conference on Learning Theory},
  pages={3351--3418},
  year={2021},
  organization={PMLR}
}

@article{Lin_Optimal_2020,
	doi = {10.1016/j.acha.2018.09.009},
	year = 2020,
	month = {may},
	publisher = {Elsevier {BV}},
	volume = {48},
	number = {3},
	pages = {868--890},
	author = {Junhong Lin and Alessandro Rudi and Lorenzo Rosasco and Volkan Cevher},
	title = {Optimal rates for spectral algorithms with least-squares regression over Hilbert spaces},
	journal = {Applied and Computational Harmonic Analysis}
}

@inproceedings{arora2019fine,
  title={Fine-grained analysis of optimization and generalization for overparameterized two-layer neural networks},
  author={Arora, Sanjeev and Du, Simon and Hu, Wei and Li, Zhiyuan and Wang, Ruosong},
  booktitle={International Conference on Machine Learning},
  pages={322--332},
  year={2019},
  organization={PMLR}
}

@article{10.1214/aop/1019160263,
author = {Pascal Massart},
title = {{About the constants in Talagrand's concentration inequalities for empirical processes}},
volume = {28},
journal = {The Annals of Probability},
number = {2},
publisher = {Institute of Mathematical Statistics},
pages = {863 -- 884},
year = {2000},
doi = {10.1214/aop/1019160263},
URL = {https://doi.org/10.1214/aop/1019160263}
}

@article{gallier2009notes,
  title={Notes on spherical harmonics and linear representations of Lie groups},
  author={Gallier, Jean},
  journal={preprint},
  year={2009}
}

@article{bartlett2005local,
  author = {Peter L. Bartlett and Olivier Bousquet and Shahar Mendelson},
title = {{Local Rademacher complexities}},
volume = {33},
journal = {The Annals of Statistics},
number = {4},
publisher = {Institute of Mathematical Statistics},
pages = {1497 -- 1537},
year = {2005},
doi = {10.1214/009053605000000282},
URL = {https://doi.org/10.1214/009053605000000282}
}

@article{caponnetto2010cross,
  title={Cross-validation based adaptation for regularization operators in learning theory},
  author={Caponnetto, Andrea and Yao, Yuan},
  journal={Analysis and Applications},
  volume={8},
  number={02},
  pages={161--183},
  year={2010},
  publisher={World Scientific}
}

@article{10.1214/21-AOS2133,
author = {Trevor Hastie and Andrea Montanari and Saharon Rosset and Ryan J. Tibshirani},
title = {{Surprises in high-dimensional ridgeless least squares interpolation}},
volume = {50},
journal = {The Annals of Statistics},
number = {2},
publisher = {Institute of Mathematical Statistics},
pages = {949 -- 986},
year = {2022},
doi = {10.1214/21-AOS2133},
URL = {https://doi.org/10.1214/21-AOS2133}
}

@article{Ghosh_three_2021,
  title={The Three Stages of Learning Dynamics in High-dimensional Kernel Methods},
  author={Ghosh, Nikhil and Mei, Song and Yu, Bin},
  journal={arXiv preprint arXiv:2111.07167},
  year={2021}
}

@techreport{Caponnetto2006OptimalRF,
  author = {Caponnetto, Andrea},
  title = {Optimal rates for regularization operators in learning theory},
  institution = {Massachusetts Institute of Technology},
  number = {CBCL Paper \#264/AI Technical Report \#062},
  year = {2006},
  month = {September},
}

@article{smola2000regularization,
  title={Regularization with dot-product kernels},
  author={Smola, Alex and Ov{\'a}ri, Zolt{\'a}n and Williamson, Robert C},
  journal={Advances in Neural Information Processing Systems},
  volume={13},
  year={2000}
}

@article{mei2022generalization,
title = {Generalization error of random feature and kernel methods: Hypercontractivity and kernel matrix concentration},
journal = {Applied and Computational Harmonic Analysis},
volume = {59},
pages = {3-84},
year = {2022},
issn = {1063-5203},
doi = {https://doi.org/10.1016/j.acha.2021.12.003},
url = {https://www.sciencedirect.com/science/article/pii/S1063520321001044},
author = {Song Mei and Theodor Misiakiewicz and Andrea Montanari}
}

@article{ghorbani2021linearized,
author = {Behrooz Ghorbani and Song Mei and Theodor Misiakiewicz and Andrea Montanari},
title = {{Linearized two-layers neural networks in high dimension}},
volume = {49},
journal = {The Annals of Statistics},
number = {2},
publisher = {Institute of Mathematical Statistics},
pages = {1029 -- 1054},
year = {2021},
doi = {10.1214/20-AOS1990},
URL = {https://doi.org/10.1214/20-AOS1990}
}

@article{kohler2001nonparametric,
  title={Nonparametric regression estimation using penalized least squares},
  author={Kohler, Michael and Krzyzak, Adam},
  journal={IEEE Transactions on Information Theory},
  volume={47},
  number={7},
  pages={3054--3058},
  year={2001},
  publisher={IEEE}
}

@article{cucker2002mathematical,
  title={On the mathematical foundations of learning},
  author={Cucker, Felipe and Smale, Steve},
  journal={Bulletin of the American mathematical society},
  volume={39},
  number={1},
  pages={1--49},
  year={2002}
}

@book{steinwart2008support,
  title={Support vector machines},
  author={Steinwart, Ingo and Christmann, Andreas},
  year={2008},
  publisher={Springer Science \& Business Media}
}

@article{sahraee2022kernel,
  title={Kernel methods and multi-layer perceptrons learn linear models in high dimensions},
  author={Sahraee-Ardakan, Mojtaba and Emami, Melikasadat and Pandit, Parthe and Rangan, Sundeep and Fletcher, Alyson K},
  journal={arXiv preprint arXiv:2201.08082},
  year={2022}
}

@inproceedings{he2016deep,
  title={Deep residual learning for image recognition},
  author={He, Kaiming and Zhang, Xiangyu and Ren, Shaoqing and Sun, Jian},
  booktitle={Proceedings of the IEEE conference on computer vision and pattern recognition},
  pages={770--778},
  year={2016}
}

@book{
carl_stephani_1990, 
place={Cambridge}, 
series={Cambridge Tracts in Mathematics}, title={Entropy, Compactness and the Approximation of Operators}, 
DOI={10.1017/CBO9780511897467}, 
publisher={Cambridge University Press}, 
author={Carl, Bernd and Stephani, Irmtraud}, 
year={1990}, 
collection={Cambridge Tracts in Mathematics}
}

@misc{tao_concentration_2010,
  title = {254A, Notes 1: Concentration of Measure},
  author = {Tao, Terence},
  year = {2010},
  howpublished = {\url{https://terrytao.wordpress.com/2010/01/03/254a-notes-1-concentration-of-measure/}}
}

@article{li2023kernel,
  title={Kernel interpolation generalizes poorly},
  author={Li, Yicheng and Zhang, Haobo and Lin, Qian},
  journal={arXiv preprint arXiv:2303.15809},
  year={2023}
}

@article{Stone_Additive_1985,
author = {Charles J. Stone},
title = {{Additive Regression and Other Nonparametric Models}},
volume = {13},
journal = {The Annals of Statistics},
number = {2},
publisher = {Institute of Mathematical Statistics},
pages = {689 -- 705},
year = {1985},
doi = {10.1214/aos/1176349548},
URL = {https://doi.org/10.1214/aos/1176349548}
}

@article{Stone_Polynomial_1994,
author = {Charles J. Stone},
title = {{The Use of Polynomial Splines and Their Tensor Products in Multivariate Function Estimation}},
volume = {22},
journal = {The Annals of Statistics},
number = {1},
publisher = {Institute of Mathematical Statistics},
pages = {118 -- 171},
year = {1994},
doi = {10.1214/aos/1176325361},
URL = {https://doi.org/10.1214/aos/1176325361}
}

@article{steinwart2012mercer,
  title={Mercer’s theorem on general domains: On the interaction between measures, kernels, and RKHSs},
  author={Steinwart, Ingo and Scovel, Clint},
  journal={Constructive Approximation},
  volume={35},
  pages={363--417},
  year={2012},
  publisher={Springer}
}

@article{sanyal2020benign,
  title={How benign is benign overfitting?},
  author={Sanyal, Amartya and Dokania, Puneet K and Kanade, Varun and Torr, Philip HS},
  journal={arXiv preprint arXiv:2007.04028},
  year={2020}
}

@article{mallinar2022benign,
  title={Benign, tempered, or catastrophic: A taxonomy of overfitting},
  author={Mallinar, Neil and Simon, James B and Abedsoltan, Amirhesam and Pandit, Parthe and Belkin, Mikhail and Nakkiran, Preetum},
  journal={arXiv preprint arXiv:2207.06569},
  year={2022}
}

@article{beaglehole2022kernel,
  title={Kernel Ridgeless Regression is Inconsistent for Low Dimensions},
  author={Beaglehole, Daniel and Belkin, Mikhail and Pandit, Parthe},
  journal={arXiv preprint arXiv:2205.13525},
  year={2022}
}

@article{Heckman_Spline_1986,
author = {Heckman, Nancy E.},
title = {Spline Smoothing in a Partly Linear Model},
journal = {Journal of the Royal Statistical Society: Series B (Methodological)},
volume = {48},
number = {2},
pages = {244-248},
doi = {https://doi.org/10.1111/j.2517-6161.1986.tb01407.x},
url = {https://rss.onlinelibrary.wiley.com/doi/abs/10.1111/j.2517-6161.1986.tb01407.x},
eprint = {https://rss.onlinelibrary.wiley.com/doi/pdf/10.1111/j.2517-6161.1986.tb01407.x},
year = {1986}
}

@article{Chen_Convergence_1988,
author = {Hung Chen},
title = {{Convergence Rates for Parametric Components in a Partly Linear Model}},
volume = {16},
journal = {The Annals of Statistics},
number = {1},
publisher = {Institute of Mathematical Statistics},
pages = {136 -- 146},
year = {1988},
doi = {10.1214/aos/1176350695},
URL = {https://doi.org/10.1214/aos/1176350695}
}

@book{dai2013approximation,
  title={Approximation theory and harmonic analysis on spheres and balls},
  author={Dai, Feng and Xu, Yuan},
  volume={23},
  year={2013},
  publisher={Springer}
}

@article{krizhevsky2017imagenet,
  title={Imagenet classification with deep convolutional neural networks},
  author={Krizhevsky, Alex and Sutskever, Ilya and Hinton, Geoffrey E},
  journal={Communications of the ACM},
  volume={60},
  number={6},
  pages={84--90},
  year={2017},
  publisher={AcM New York, NY, USA}
}

@article{stone1977consistent,
author = {Charles J. Stone},
title = {{Consistent Nonparametric Regression}},
volume = {5},
journal = {The Annals of Statistics},
number = {4},
publisher = {Institute of Mathematical Statistics},
pages = {595 -- 620},
year = {1977},
doi = {10.1214/aos/1176343886},
URL = {https://doi.org/10.1214/aos/1176343886}
}

@article{cleveland1979robust,
author = { William S.   Cleveland },
title = {Robust Locally Weighted Regression and Smoothing Scatterplots},
journal = {Journal of the American Statistical Association},
volume = {74},
number = {368},
pages = {829-836},
year  = {1979},
publisher = {Taylor & Francis},
doi = {10.1080/01621459.1979.10481038},

URL = { 
    
    
        https://www.tandfonline.com/doi/abs/10.1080/01621459.1979.10481038
    

},
eprint = { 
    
    
        https://www.tandfonline.com/doi/pdf/10.1080/01621459.1979.10481038
    

}

}

@book{goodfellow2016deep,
  title={Deep learning},
  author={Goodfellow, Ian and Bengio, Yoshua and Courville, Aaron},
  year={2016},
  publisher={MIT press}
}

@inproceedings{pascanu2013difficulty,
  title={On the difficulty of training recurrent neural networks},
  author={Pascanu, Razvan and Mikolov, Tomas and Bengio, Yoshua},
  booktitle={International Conference on Machine Learning},
  pages={1310--1318},
  year={2013},
  organization={PMLR}
}

@article{lecun2015deep,
  title={Deep learning},
  author={LeCun, Yann and Bengio, Yoshua and Hinton, Geoffrey},
  journal={nature},
  volume={521},
  number={7553},
  pages={436--444},
  year={2015},
  publisher={Nature Publishing Group UK London}
}

@article{li2018learning,
  title={Learning overparameterized neural networks via stochastic gradient descent on structured data},
  author={Li, Yuanzhi and Liang, Yingyu},
  journal={Advances in Neural Information Processing Systems},
  volume={31},
  year={2018}
}

@article{li2023statistical,
  author  = {Yicheng Li and Zixiong Yu and Guhan Chen and Qian Lin},
  title   = {On the Eigenvalue Decay Rates of a Class of Neural-Network Related Kernel Functions Defined on General Domains},
  journal = {Journal of Machine Learning Research},
  year    = {2024},
  volume  = {25},
  number  = {82},
  pages   = {1--47},
  url     = {http://jmlr.org/papers/v25/23-0866.html}
}

@ARTICLE{9051968,
  author={Muthukumar, Vidya and Vodrahalli, Kailas and Subramanian, Vignesh and Sahai, Anant},
  journal={IEEE Journal on Selected Areas in Information Theory}, 
  title={Harmless Interpolation of Noisy Data in Regression}, 
  year={2020},
  volume={1},
  number={1},
  pages={67-83},
  doi={10.1109/JSAIT.2020.2984716}}

@article{bartlett2020benign,
  title={Benign overfitting in linear regression},
  author={Bartlett, Peter L and Long, Philip M and Lugosi, G{\'a}bor and Tsigler, Alexander},
  journal={Proceedings of the National Academy of Sciences},
  volume={117},
  number={48},
  pages={30063--30070},
  year={2020},
  publisher={National Acad Sciences}
}

@article{cao2022benign,
  title={Benign overfitting in two-layer convolutional neural networks},
  author={Cao, Yuan and Chen, Zixiang and Belkin, Misha and Gu, Quanquan},
  journal={Advances in Neural Information Processing Systems},
  volume={35},
  pages={25237--25250},
  year={2022}
}

@article{tsigler2020benign,
  title={Benign overfitting in ridge regression},
  author={Tsigler, Alexander and Bartlett, Peter L},
  journal={arXiv preprint arXiv:2009.14286},
  year={2020}
}

@article{aerni2023strong,
  title={Strong inductive biases provably prevent harmless interpolation},
  author={Aerni, Michael and Milanta, Marco and Donhauser, Konstantin and Yang, Fanny},
  journal={arXiv preprint arXiv:2301.07605},
  year={2023}
}

@inproceedings{liang2020multiple,
  title={On the multiple descent of minimum-norm interpolants and restricted lower isometry of kernels},
  author={Liang, Tengyuan and Rakhlin, Alexander and Zhai, Xiyu},
  booktitle={Conference on Learning Theory},
  pages={2683--2711},
  year={2020},
  organization={PMLR}
}

@book{scholkopf2002learning,
  title={Learning with kernels: support vector machines, regularization, optimization, and beyond},
  author={Sch{\"o}lkopf, Bernhard and Smola, Alexander J and Bach, Francis and others},
  year={2002},
  publisher={MIT press}
}

@book{wainwright2019high,
  title={High-dimensional statistics: A non-asymptotic viewpoint},
  author={Wainwright, Martin J},
  volume={48},
  year={2019},
  publisher={Cambridge university press}
}

@inproceedings{steinwart2009optimal,
  title={Optimal Rates for Regularized Least Squares Regression},
  author={Steinwart, Ingo and Hush, Don and Scovel, Clint},
  booktitle={Conference on Learning Theory},
  pages={79--93},
  year={2009},
  organization={PMLR}
}

@article{zhang2023optimality,
  title={On the Optimality of Misspecified Kernel Ridge Regression},
  author={Zhang, Haobo and Li, Yicheng and Lu, Weihao and Lin, Qian},
  journal={arXiv preprint arXiv:2305.07241},
  year={2023}
}

@article{zhang2023optimality_2,
  title={On the optimality of misspecified spectral algorithms},
  author={Zhang, Haobo and Li, Yicheng and Lin, Qian},
  journal={arXiv preprint arXiv:2303.14942},
  year={2023}
}

@article{xiao2022precise,
  title={Precise learning curves and higher-order scaling limits for dot product kernel regression},
  author={Xiao, Lechao and Pennington, Jeffrey},
  journal={arXiv preprint arXiv:2205.14846},
  year={2022}
}

@article{canatar2021spectral,
  title={Spectral bias and task-model alignment explain generalization in kernel regression and infinitely wide neural networks},
  author={Canatar, Abdulkadir and Bordelon, Blake and Pehlevan, Cengiz},
  journal={Nature communications},
  volume={12},
  number={1},
  pages={2914},
  year={2021},
  publisher={Nature Publishing Group UK London}
}

@article{barzilai2023generalization,
  title={Generalization in Kernel Regression Under Realistic Assumptions},
  author={Barzilai, Daniel and Shamir, Ohad},
  journal={arXiv preprint arXiv:2312.15995},
  year={2023}
}

@article{https://doi.org/10.1002/cpa.22008,
author = {Mei, Song and Montanari, Andrea},
title = {The Generalization Error of Random Features Regression: Precise Asymptotics and the Double Descent Curve},
journal = {Communications on Pure and Applied Mathematics},
volume = {75},
number = {4},
pages = {667-766},
doi = {https://doi.org/10.1002/cpa.22008},
url = {https://onlinelibrary.wiley.com/doi/abs/10.1002/cpa.22008},
eprint = {https://onlinelibrary.wiley.com/doi/pdf/10.1002/cpa.22008},
year = {2022}
}

@article{zhang2024phase,
  title={The phase diagram of kernel interpolation in large dimensions},
  author={Zhang, Haobo and Lu, Weihao and Lin, Qian},
  journal={arXiv preprint arXiv:2404.12597},
  year={2024}
}

@article{gneiting2013strictly,
  author = {Tilmann Gneiting},
    title = {{Strictly and non-strictly positive definite functions on spheres}},
    volume = {19},
    journal = {Bernoulli},
    number = {4},
    publisher = {Bernoulli Society for Mathematical Statistics and Probability},
    pages = {1327 -- 1349},
    year = {2013},
    doi = {10.3150/12-BEJSP06},
    URL = {https://doi.org/10.3150/12-BEJSP06}
}

@article{zhang2024optimal,
  title={Optimal Rates of Kernel Ridge Regression under Source Condition in Large Dimensions},
  author={Zhang, Haobo and Li, Yicheng and Lu, Weihao and Lin, Qian},
  journal={arXiv preprint arXiv:2401.01270},
  year={2024}
}
\newpage
\title{Supplement to "Optimal Rate of Kernel Regression
in Large Dimensions"}
\section{Proof of Claims and Theorems in Section \ref{sec:more_intevals}}
From Lemma \ref{lemma:inner_mendelson_point_control_assist_2} and Lemma \ref{lemma:monotone_of_eigenvalues_of_inner_product_kernels}, one can check the following claim:
\begin{claim}\label{claim_3_inner}
Suppose that $\gamma \in(2 p, 2 p+1]$  for some integer $p$. Then, for any $\delta^{\prime}>0$ and 
for any $d \geq \mathfrak{C}$, where $\mathfrak{C}$ is a sufficiently large constant depending on the fixed kernel and noise parameters, $p$, $\gamma$, $\delta^{\prime}$, $c_1$, $c_2$, and the chosen trial-radius constant, we have
\begin{equation*}
\begin{aligned}
&K\left( \sqrt{2}\sigma \tilde\varepsilon_2 / 6 \right) \leq
 \left(1+\frac{\delta^{\prime}}{4}\right) \frac{1}{2} N(d,p)\log\left(\frac{18\mu_p}{\sigma^2\tilde\varepsilon_2^2}\right).
\end{aligned}
\end{equation*}
\end{claim}
\begin{proof}
Put $r^2=\sigma^2\tilde\varepsilon_2^2/18$.  In this regime
$r^2\asymp d^{p-\gamma}\log d$, and hence
$r^2/\mu_{p+1}\to\infty$ (including $\gamma=2p+1$).
The fixed-degree estimates and Lemma
\ref{lemma:monotone_of_eigenvalues_of_inner_product_kernels}
show that $\sup_{k\ge p+1}\mu_k\le C d^{-(p+1)}<r^2$
for all sufficiently large $d$.  Thus no degree above $p$
contributes to $K(r)$.  For each fixed $k<p$, the corresponding
contribution is $O(d^k\log d)$, whose sum is
$o(N(d,p)\log d)$ when $p\ge1$; for $p=0$ this sum is empty.
The remaining degree-$p$ contribution is exactly
$\frac12 N(d,p)\log(\mu_p/r^2)$, since
$\mu_p/r^2\to\infty$.  Its logarithm is
$(\gamma-2p)\log d+O(\log\log d)$, and thus it dominates
the lower-degree contributions.  Absorbing the latter into the
fixed positive slack $\delta'/4$ proves the claim.
\end{proof}
\vspace{10pt}
\subsection{Proof of Theorem \ref{thm:near_lower_inner_large_d}}
We give a direct finite-dimensional lower-bound argument that covers all
three regimes and avoids any generalized-inverse or entropy-critical-radius
assumption.  Fix an integer degree $q\geq0$ for which the population
kernel eigenvalue $\mu_q>0$, and put $M=N(d,q)$.  Let
$Y_{q,1},\ldots,Y_{q,M}$ be orthonormal spherical harmonics in
$L^2(\rho_X)$.  For each $v\in\{-1,1\}^M$, define
\[
 f_v=a\sum_{j=1}^M v_jY_{q,j},\qquad
 a^2=\min\left\{\frac{\mu_q}{M},\frac{\sigma^2}{16n}\right\}.
\]
Since the RKHS norm on the degree-$q$ block satisfies
$\|f_v\|_{\calH^{\inner}}^2=Ma^2/\mu_q\leq1$, every $f_v$
belongs to $\calB$.  For two sign vectors differing in precisely one
coordinate, the Gaussian regression model with random design gives
\[
 D_{\mathrm{KL}}\bigl(\rho_{f_v}^{\otimes n}
 \|\rho_{f_{v'}}^{\otimes n}\bigr)
 =\frac{n}{2\sigma^2}\|f_v-f_{v'}\|_{L^2}^2
 =\frac{2na^2}{\sigma^2}\leq\frac18.
\]
By Pinsker's inequality, the total variation distance between these
neighboring experiments is at most $1/4$.  For any measurable
estimator $\widehat f$, estimate the $j$th sign by the sign of
$\langle\widehat f,Y_{q,j}\rangle_{L^2}$ (with an arbitrary
convention at zero).  If the estimated sign is wrong, the squared
error in that coefficient is at least $a^2$.  Averaging the testing
error over each neighboring pair under the uniform prior on the
hypercube gives a probability of sign error at least
$(1-1/4)/2=3/8$.  Orthogonality and Bessel's inequality thus imply
\[
 \inf_{\widehat f}\sup_{f_\star\in\calB}
 \mathbb E_{\rho_{f_\star}^{\otimes n}}
 \|\widehat f-f_\star\|_{L^2}^2
 \geq\frac38 Ma^2
 =\frac38\min\left\{\mu_q,\frac{\sigma^2N(d,q)}{16n}\right\}.
 \tag{*}
\]
The testing argument applies to arbitrary measurable estimators,
including estimators that do not take values in the RKHS.
For fixed $q$, Lemma \ref{lemma:inner_mendelson_point_control_assist_2}
and the spherical-harmonic multiplicity formula give
$\mu_q\asymp d^{-q}$ and $N(d,q)\asymp d^q$; their
constants may depend on $q$ and the fixed kernel but not on $d,n$.
Also $c_1d^\gamma\leq n\leq c_2d^\gamma$.
If $\gamma=2p$ for an integer $p\geq1$, take $q=p$ in (*).
Both terms are bounded below by a fixed multiple of $d^{-p}$,
which is comparable to $n^{-1/2}$.  This proves (i).
If $\gamma\in(2p,2p+1]$ for an integer $p\geq0$, take $q=p$.
Because $\gamma>2p$, the smaller term in (*) is of order
$d^{p-\gamma}\asymp n^{-(\gamma-p)/\gamma}$.  This proves
(ii), in fact with the stronger exponent $(\gamma-p)/\gamma$
without the arbitrarily small loss $\epsilon$ stated there.
In particular $p=0$ is covered using $N(d,0)=1$ and
$\mu_0\asymp1$, and the endpoint $\gamma=2p+1$ is included.
Finally, if $\gamma\in(2p+1,2p+2)$, take $q=p+1$.
Since $\gamma<2p+2$, the smaller term in (*) is of order
$d^{-(p+1)}\asymp n^{-(p+1)/\gamma}$.  This proves (iii).
All implicit constants and dimension thresholds are uniform over
$n\in[c_1d^\gamma,c_2d^\gamma]$ for each fixed $\gamma$.
\hfill$\square$

\subsection{Proof of Theorem \ref{thm:near_upper_inner_large_d}}\label{append_a2}
Let $\gamma > 0$ be a fixed real number and $p = \lfloor \gamma/2 \rfloor$.
Recall that the empirical eigenvalues $\widehat \lambda_i$'s are defined in Definition \ref{def:pop_men_complexity}.
The following lemma shows that there is a gap between two empirical eigenvalues $\widehat\lambda_{N(p)+1}$ and $\widehat\lambda_{N(p)}$ in high dimensions.
\begin{lemma}\label{condition_inner_3}
Adopt all notations and conditions in Theorem \ref{thm:near_upper_inner_large_d}. Further suppose that $\gamma \neq 2, 4, 6, \cdots$.
    For any constants $0<c_1\leq c_2<\infty$ and any $\delta>0$, 
    there exist constants $\mathfrak{C}^{\prime\prime}$ and  $\mathfrak{C}_1$ depending on $c_1$, $c_2$, $\delta$, $\gamma$, and the fixed kernel parameters, such that
    for any $d \geq \mathfrak{C}^{\prime\prime}$, when $c_1 d^{\gamma} \leq n < c_2 d^{\gamma}$,
     we have
    \begin{align}
        \widehat\lambda_{N(0)+1} 
        &<  
        \frac{\mathfrak{C}_1}{n}, \quad
         \mu_0 / 4
         <
         \widehat\lambda_{N(0)}, \quad \text{ if } \gamma \in (0, 1]
         \label{eqn:empirical_inner_eigenvalues_inner_small_gamma_1}\\
         \widehat\lambda_{N(p)+1} 
        &<  
        4\mu_{p+1}
        <
         \mu_p / 4
         <
         \widehat\lambda_{N(p)}, \quad \text{ if } \gamma > 1,
         \label{eqn:empirical_inner_eigenvalues_gamma_greater_than_4}
    \end{align}
with probability at least $1-\delta$, where $N(p)=\sum_{k=0}^{p} N(d, k)$.
\proof Deferred to the end of this subsection.
\end{lemma}
\begin{lemma}[Blockwise empirical-to-population transfer]\label{lemma:block_transfer_inner}
Adopt the assumptions of Theorem \ref{thm:near_upper_inner_large_d}, let $p=\lfloor\gamma/2\rfloor$, and suppose that $\gamma\notin\{2,4,6,\ldots\}$. For every fixed $M>0$ and $\delta>0$, there is a constant $C>0$ such that, for all sufficiently large $d$, with probability at least $1-\delta$,
\begin{equation}\label{eq:block_transfer_inner}
 \|g\|_{L^2}^2 \le C\left\{\|g\|_n^2+M^2d^{-(p+1)}+\frac{M^2}{n}\right\}
 \qquad\text{for every }g\in\mathcal H\text{ with }\|g\|_{\mathcal H}\le M.
\end{equation}
In particular, the assertion is uniform in $g$ and may therefore be applied to a data-dependent function such as $f_{\widehat T}^{\inner}-f_\star$.
\proof Deferred to the end of this subsection.
\end{lemma}
The proof of Theorem \ref{thm:near_upper_inner_large_d} is mainly based on the proof of Theorem \ref{theorem:restate_norm_diff}.
But we have to update Lemma \ref{corollary:norm_diff}, \ref{lemma_B_t:thm:empirical_loss}, and \ref{lemma_V_t:thm:empirical_loss} into following lemmas, respectively.
\begin{lemma}\label{lemma_B_t:thm:empirical_loss_modified}
    For any $J \geq 1$, if $t^{-1} \in [ \widehat\lambda_{J+1},  \widehat\lambda_{J})$, then we have
\begin{equation}
\begin{aligned}
\mathbf{B}_{t}^2 \leq \frac{1}{t^2 \widehat\lambda_{J}} + 2\widehat\lambda_{J+1}.
\end{aligned}
\end{equation}
\proof Deferred to the end of this subsection.
\end{lemma}
\begin{lemma}\label{lemma_V_t:thm:empirical_loss_modified}
Conditionally on the design, let $t>0$ be deterministic or design-measurable. For any $\delta_2>0$ and any $J \geq 1$, if $t^{-1} \in [ \widehat\lambda_{J+1},  \widehat\lambda_{J})$, then we have
\begin{equation}
\begin{aligned}
\mathbf{V}_{t} \leq 2\sigma^2\frac{t^2}{n} \left(\frac{J}{t^2} +   \kappa\widehat{\lambda}_{J+1}\right) +
\delta_2,
\end{aligned}
\end{equation}
with conditional probability at least $1-\exp \left(-C \min\left\{\frac{n\delta_2}{2}, \frac{n^2\delta_2^2}{ 4 t^2\left(\frac{J}{t^2} +   \kappa\widehat{\lambda}_{J+1}\right)}\right\}\right)$, where $C>0$ depends on the fixed noise variance.
\proof Deferred to the end of this subsection.
\end{lemma}
{
\begin{remark}
    The condition $t^{-1} \in [ \widehat\lambda_{J+1},  \widehat\lambda_{J})$ places the inverse stopping time $t^{-1}$ between two consecutive empirical eigenvalues, which allows us to decompose the bias and variance terms sharply at the index $J$. Specifically, when there exists a gap between $\widehat\lambda_{J+1}$ and $\widehat\lambda_{J}$ (the existence of such gap is ensured by Lemma \ref{lemma:inner_edr} and Lemma \ref{condition_inner_3}), this separation enables us to bound the contributions from low-frequency and high-frequency components differently, leading to the tight estimates in Lemmas \ref{lemma_B_t:thm:empirical_loss_modified} and \ref{lemma_V_t:thm:empirical_loss_modified}.
\end{remark}
}
Now let's begin to prove Theorem \ref{thm:near_upper_inner_large_d}.
The proof will be divided into three parts: 
\begin{itemize}
    \item[(i)] $\gamma \in \{2, 4, 6, \cdots\}$,
    \item[(ii)] $\gamma \in \bigcup_{j=0}^{\infty} (2j, 2j+1]$,
    \item[(iii)] $\gamma \in \bigcup_{j=0}^{\infty} (2j+1, 2j+2)$.
\end{itemize}
\paragraph*{Proof of Theorem \ref{thm:near_upper_inner_large_d} (i)} This is a direct corollary of Theorem \ref{thm:upper_bpund_inner}.
\paragraph*{Proof of Theorem \ref{thm:near_upper_inner_large_d} (ii)}
Suppose that $\gamma \in \bigcup_{j=0}^{\infty} (2j, 2j+1]$. Let $p = \lfloor \gamma/2 \rfloor$.
For $0<\gamma\le1$, the empirical spectral event must be treated separately. Here $p=0$ and $J=N(0)=1$. Lemma~\ref{condition_inner_3} gives $\widehat\lambda_2\le C/n$ and $\widehat\lambda_1\ge\mu_0/4$ with probability at least $1-\delta/4$. The critical-radius bounds give $\widehat T^{-1}=\widehat\varepsilon_n^2\asymp n^{-1/2}$ on an event whose complement has probability $Ce^{-c\sqrt n}$. Thus $\widehat\lambda_2<\widehat T^{-1}<\widehat\lambda_1$ for large $d$. Applying Lemmas~\ref{lemma_B_t:thm:empirical_loss_modified} and~\ref{lemma_V_t:thm:empirical_loss_modified} conditionally on the design, with $\delta_2=A\log(n)/n$ for sufficiently large fixed $A$, gives
\[
 \|f_{\widehat T}^{\inner}-f_\star\|_n^2
 \le C\{\widehat T^{-2}+n^{-1}+\widehat T^2/n^2\}+\delta_2
 \le C'\log(n)/n.
\]
Indeed $\widehat T\asymp\sqrt n$, so each term in braces is $O(n^{-1})$. The conditional Gaussian deviation has probability at most $Ce^{-c n\delta_2}$ after choosing $A$ sufficiently large. Apply the blockwise transfer lemma with failure probability $\delta/4$; the empirical spectral event, transfer event, critical-radius event, RKHS-norm event, and conditional Gaussian deviation event are then combined by a union bound. The RKHS-norm event and the uniform blockwise transfer lemma (with $p=0$ and $M=3$) now yield
\[
 \|f_{\widehat T}^{\inner}-f_\star\|_{L^2}^2
 \le C\{\log(n)/n+d^{-1}+n^{-1}\}
 \le C'\log(n)/n,
\]
since $n\lesssim d$ for $\gamma\le1$. This proves the low-$\gamma$ branch with the asserted fixed-confidence bound. Henceforth assume $\gamma>1$.
For any fixed $\delta>0$, apply Lemma \ref{condition_inner_3} and Lemma \ref{lemma:block_transfer_inner} below with failure probability $\delta/4$ each. The remaining failure budget $\delta/2$ is reserved for the exponentially small critical-radius, RKHS-norm, and conditional Gaussian deviation events; for sufficiently large $d$, their sum is at most $\delta/2$. 
Let $d \geq \mathfrak{C}=\max\{\mathfrak{C}', \mathfrak{C}''\}$, where $\mathfrak{C}'$ is the constant (depending only on $c_{1}$, $c_{2}$ and $\gamma$ ) introduced in Theorem \ref{thm:upper_bpund_inner} and $\mathfrak{C}''$ is the constant (depending only on $c_{1}$, $c_{2}$, $\gamma$ and $\delta$ ) introduced in Lemma \ref{condition_inner_3}. 
Note that  Theorem \ref{thm:upper_bpund_inner}, Lemma \ref{lemma:bound_empirical_and_expected_mandelson_complexities}, and Lemma \ref{lemma:inner_mendelson_interval_control} imply that
$$4\mu_{p+1} \leq \mathfrak{C}_1 n^{-1/2}
\leq \widehat{T}^{-1}=\widehat{\varepsilon}_n^2 \leq 
\mathfrak{C}_2 n^{-1/2} \leq \mu_{p}/4$$ holds with probability at least $1-\mathfrak{C}_3 \exp \left\{ -\mathfrak{C}_4 n^{1/2}\right\}$ and  Lemma \ref{condition_inner_3} implies that  
\begin{align*}
             \widehat\lambda_{J+1} 
        <  
        4\mu_{p+1}
        <
         \mu_p / 4
         <
         \widehat\lambda_{J}
\end{align*} 
holds with probability at least $1-\delta/4$ where $J=N(p)$. Thus, we know that $\widehat\lambda_{J+1} 
        <  
        4\mu_{p+1}
        \leq
        \widehat{T}^{-1}
        \leq
         \mu_p / 4
         <
         \widehat\lambda_{J}$ with probability at least $1-\delta/4 - \mathfrak{C}_3 \exp \left\{ -\mathfrak{C}_4 n^{1/2}\right\}$. 
\vspace{5pt}
Let $\delta_2 = A d^{p-\gamma} \log(d)$, with $A$ a sufficiently large fixed constant.  On the critical-radius comparison event, there is a fixed constant $C_T\geq1$ such that $\widehat T^2/n\leq C_T$ and $1/\widehat T^2\leq C_T/n$. The following bounds therefore retain this factor rather than treating $\widehat T^2=n$ as an identity.
On the design-measurable event $\Omega = \left\{  
\widehat\lambda_{J+1} 
        <  
        4\mu_{p+1}
        \leq
        \widehat{T}^{-1}
        \leq
         \mu_p / 4
         <
         \widehat\lambda_{J}
\right\}$, we have 
\begin{align*}
\left\|f_{\widehat T}^{\inner} - f_{\star} \right\|_n^2 &\leq \mathbf{B}_{\widehat T}^{2} + \mathbf{V}_{\widehat T} \\
&\leq \frac{1}{\widehat{T}^{2}\widehat{\lambda}_{J} }+2\widehat \lambda_{J+1}+2\sigma^{2}\frac{J}{n}+2\kappa\sigma^{2}\frac{\widehat{T}^{2}\widehat \lambda_{J+1}}{n}+\delta_{2} \\
 &\leq C_T\left(\frac{1}{n\mu_{p}}+\mu_{p+1}+2\sigma^{2}\frac{J}{n}+2\kappa\sigma^{2}\mu_{p+1}\right)+\delta_{2}\\
    &
    \leq
    \mathfrak{C}_4\left(
    d^{p-\gamma}+d^{-p-1}+2\sigma^{2}d^{p-\gamma}+2\kappa\sigma^{2}d^{-p-1}
    \right)
    +\delta_{2}\\
&\leq \frac{3}{2}\delta_{2},
\end{align*}
holds conditionally on every design $X$ in the displayed spectral and critical-radius events, with conditional noise failure probability at most $\exp\{-c d^{p}\log d\}$. Indeed, $t=\widehat T\asymp\sqrt n$, $J\asymp d^p$, and $t^2\widehat\lambda_{J+1}\lesssim n d^{-(p+1)}\lesssim d^p$ in this regime; hence the two exponents in the preceding variance lemma are bounded below by constant multiples of $d^p\log d$ and $d^p(\log d)^2$, respectively. The design-event failures are accounted for separately in the final union bound. Here the second inequality follows from
Lemma \ref{lemma_B_t:thm:empirical_loss_modified} and Lemma \ref{lemma_V_t:thm:empirical_loss_modified} and the second last inequality follows from Lemma \ref{lemma:inner_mendelson_point_control_assist_2} with a constant $\mathfrak{C}_4$ depending only on $c_1$, $c_2$, $\mathfrak{C}_1$, and $\mathfrak{C}_2$. 
\vspace{3mm}
The remaining step is to transfer the empirical prediction error to the population norm. We use the block structure of the spherical-harmonic decomposition rather than a data-dependent covering net. By the RKHS-norm bound used in the proof of Theorem \ref{theorem:restate_norm_diff}, $\|f_{\widehat T}^{\inner}-f_\star\|_{\mathcal H}\le 3$ with the stated high probability. Applying Lemma \ref{lemma:block_transfer_inner} with $M=3$ gives
\begin{align*}
 \|f_{\widehat T}^{\inner}-f_\star\|_{L^2}^2
 &\le C\left\{\|f_{\widehat T}^{\inner}-f_\star\|_n^2+d^{-(p+1)}+n^{-1}\right\}\\
 &\le C\left\{\delta_2+d^{-(p+1)}+n^{-1}\right\}.
\end{align*}
Since $\gamma\in(2p,2p+1]$, $n\asymp d^\gamma$, and $\delta_2\asymp d^{p-\gamma}\log d$, both $d^{-(p+1)}$ and $n^{-1}$ are bounded by a constant multiple of $\delta_2$. Consequently,
\[
 \|f_{\widehat T}^{\inner}-f_\star\|_{L^2}^2 \le \mathfrak C_4 d^{p-\gamma}\log d.
\]
This holds with probability at least $1-\delta-\mathfrak C_2\exp\{-\mathfrak C_3d^p\log d\}$ after changing constants. Notice that Lemma \ref{lemma:block_transfer_inner} is uniform over the deterministic RKHS ball, so no concentration inequality is applied to a sample-dependent net.
\hfill $\square$
\paragraph*{Proof of Theorem \ref{thm:near_upper_inner_large_d} (iii)}
Suppose that $\gamma \in \bigcup_{j=0}^{\infty} (2j+1, 2j+2)$. Let $p = \lfloor \gamma/2 \rfloor$.
We apply Lemma \ref{condition_inner_3} and Lemma \ref{lemma:block_transfer_inner} with failure probability $\delta/4$ each. The exponentially small auxiliary failure probabilities are absorbed in the remaining $\delta/2$ budget for sufficiently large $d$.
By the same argument, 
$\widehat\lambda_{J+1} 
        <  
        4\mu_{p+1}
        \leq
        \widehat{T}^{-1}
        \leq
         \mu_p / 4
         <
         \widehat\lambda_{J}$
holds with probability at least $1-\delta/4 - \mathfrak{C}_3 \exp \left\{ -\mathfrak{C}_4 n^{1/2}\right\}$ where $J=N(p)$.
\vspace{5pt}
Let $\delta_2 = A d^{-(p+1)}$, with $A$ a sufficiently large fixed constant. 
On the critical-radius comparison event, there is a fixed constant $C_T\geq1$ such that $\widehat T^2/n\leq C_T$ and $1/\widehat T^2\leq C_T/n$. The following bounds therefore retain this factor rather than treating $\widehat T^2=n$ as an identity.
On the design-measurable event $\Omega = \left\{  
\widehat\lambda_{J+1} 
        <  
        4\mu_{p+1}
        \leq
        \widehat{T}^{-1}
        \leq
         \mu_p / 4
         <
         \widehat\lambda_{J}
\right\}$, we have 
\begin{align*}
\left\|f_{\widehat T}^{\inner} - f_{\star} \right\|_n^2 &\leq \mathbf{B}_{\widehat T}^{2} + \mathbf{V}_{\widehat T} \\
&\leq \frac{1}{\widehat{T}^{2}\widehat{\lambda}_{J} }+2\widehat \lambda_{J+1}+2\sigma^{2}\frac{J}{n}+2\kappa\sigma^{2}\frac{\widehat{T}^{2}\widehat \lambda_{J+1}}{n}+\delta_{2} \\
 &\leq C_T\left(\frac{1}{n\mu_{p}}+\mu_{p+1}+2\sigma^{2}\frac{J}{n}+2\kappa\sigma^{2}\mu_{p+1}\right)+\delta_{2}\\
    &
    \leq
    \mathfrak{C}_4\left(
    d^{p-\gamma}+d^{-p-1}+2\sigma^{2}d^{p-\gamma}+2\kappa\sigma^{2}d^{-p-1}
    \right)
    +\delta_{2}\\
&\leq \mathfrak{C}_5 \delta_{2},
\end{align*}
holds conditionally on every design $X$ in the displayed spectral and critical-radius events, with conditional noise failure probability at most $\exp\{-c d^{\gamma-(p+1)}\}$. Indeed, $t^2\widehat\lambda_{J+1}\lesssim n d^{-(p+1)}$, which dominates $J\asymp d^p$ when $\gamma>2p+1$; both exponents in the variance lemma are then at least of order $d^{\gamma-p-1}$. The design-event failures are accounted for separately in the final union bound. Here the second inequality follows from
Lemma \ref{lemma_B_t:thm:empirical_loss_modified} and Lemma \ref{lemma_V_t:thm:empirical_loss_modified}, the second last inequality follows from Lemma \ref{lemma:inner_mendelson_point_control_assist_2} with a constant $\mathfrak{C}_4$ depending only on $c_1$, $c_2$, $\mathfrak{C}_1$, and $\mathfrak{C}_2$, and $\mathfrak{C}_5 = \mathfrak{C}_4(1+2\sigma^2) / (2\mathfrak{C}_3) + 2$.
\vspace{3mm}
As in part (ii), we avoid a data-dependent covering net and apply the uniform blockwise transfer Lemma \ref{lemma:block_transfer_inner} to $g=f_{\widehat T}^{\inner}-f_\star$, whose RKHS norm is at most $3$ on the same high-probability event. Hence
\begin{align*}
 \|f_{\widehat T}^{\inner}-f_\star\|_{L^2}^2
 &\le C\left\{\|f_{\widehat T}^{\inner}-f_\star\|_n^2+d^{-(p+1)}+n^{-1}\right\}\\
 &\le C\left\{\delta_2+d^{-(p+1)}+n^{-1}\right\}
 \le \mathfrak C_4 d^{-(p+1)},
\end{align*}
where the last inequality uses $\gamma\in(2p+1,2p+2)$ and $\delta_2\asymp d^{-(p+1)}$. Therefore the claimed population-risk bound holds with probability at least $1-\delta-\mathfrak C_2\exp\{-\mathfrak C_3d^{\gamma-(p+1)}\}$, after changing constants.
\hfill $\square$
\vspace{10pt}
\noindent \proofname ~of Lemma \ref{condition_inner_3}: 
First, consider the case $\gamma > 1$ and prove \eqref{eqn:empirical_inner_eigenvalues_gamma_greater_than_4}.
    From Mercer's decomposition, we have the following decomposition:
    \begin{equation}\label{eqn_178_inner_decomp_emp_matrix}
    \begin{aligned}
        \frac{1}{n} K(\boldsymbol{X}, \boldsymbol{X}) &= \frac{1}{n} \boldsymbol{Y}_{\leq p+1} \boldsymbol{D}_{\leq p+1} \boldsymbol{Y}_{\leq p+1}^{\tau}
        +
        \frac{1}{n} \sum_{k=2}^{\infty} 
        \boldsymbol{Y}_{p+k} \boldsymbol{D}_{p+k} \boldsymbol{Y}_{p+k}^{\tau}\\
        &= K_{\text{main}} + K_{\text{residual}},
        \end{aligned}
    \end{equation}
where $Y_{q, j}(\cdot)$ for $j=1, \cdots, N(d, q)$ are spherical harmonic polynomials of degree $q \in \{0, 1, 2, \cdots\}$, $\boldsymbol{Y}_{q} =\left(Y_{q l}\left(\boldsymbol{x}_i\right)\right)_{i \in[n], l \in[N(d, q)]} \in \mathbb{R}^{n \times N(d, q)}$, \\
$\boldsymbol{Y}_{\leq p+1}=\left(\boldsymbol{Y}_0, \ldots, \boldsymbol{Y}_{p+1}\right) \in \mathbb{R}^{n \times N(p+1)}$,
$\boldsymbol{D}_{\leq p+1} = \text{diag}(\mu_0 \mathbf{I}_{N(d, 0)}, \cdots, \mu_{p+1} \mathbf{I}_{N(d, p+1)})$, and $\boldsymbol{D}_{p + k} = \mu_{p+k} \mathbf{I}_{N(d, p+k)}$. 
\vspace{10pt}
We replicate some results from \cite{ghorbani2021linearized} and \cite{xiao2022precise}. 
\begin{proposition}[Concatenated low-degree harmonic Gram matrix]\label{prop:lemma_11_in_montanari}
For any fixed integer $q\geq0$, let $D=N(q)$ denote the dimension of the positive-eigenvalue harmonic blocks of degrees at most $q$, and let $\boldsymbol Y_{\le q}\in\mathbb R^{n\times D}$ contain the corresponding orthonormal harmonic evaluations. Then
\[
\mathbb E\left\|\frac{\boldsymbol Y_{\le q}^{\tau}\boldsymbol Y_{\le q}}n-I_D\right\|_{\rm op}
\le C\left\{\sqrt{\frac{D\log(2D)}n}+\frac{D\log(2D)}n\right\}.
\]
In particular, if $n\gg D\log(2D)$, this expectation is $o_d(1)$.
\end{proposition}
\begin{proof}
Write $v_i\in\mathbb R^D$ for the $i$th row of $\boldsymbol Y_{\le q}$. Orthonormality across all included degrees gives $\mathbb E[v_iv_i^\tau]=I_D$, including the cross-degree terms. The spherical-harmonic addition formula gives $\|v_i\|_2^2=D$ deterministically. Thus $Z_i=v_iv_i^\tau-I_D$ are independent centered self-adjoint matrices satisfying $\|Z_i\|_{\rm op}\le D+1$ and
\[
\sum_{i=1}^n\mathbb E Z_i^2=n(D-1)I_D.
\]
Matrix Bernstein therefore yields, for every $u\ge0$,
\[
\mathbb P\left\{\left\|\frac1n\sum_{i=1}^n Z_i\right\|_{\rm op}
>C\left(\sqrt{\frac{D(\log(2D)+u)}n}+\frac{D(\log(2D)+u)}n\right)\right\}\le e^{-u}.
\]
Integration of this tail inequality gives the asserted expectation bound (with a possibly larger universal $C$). The case $D=1$ is immediate since $v_iv_i^\tau=1$.
\end{proof}
\begin{proposition}[Fixed-degree spherical block upper bound]\label{prop:eqn_72_in_montanari}
For every fixed integer $q\geq 0$ and every $u>0$, the degree-$q$ harmonic evaluation matrix satisfies, with probability at least $1-e^{-u}$,
\begin{equation}\label{eq:fixed_degree_bernstein_upper}
\left\|\frac{\mu_q}{n}\boldsymbol Y_q\boldsymbol Y_q^\tau\right\|_{\mathrm{op}}
\leq \mu_q\left[1+C\left\{\sqrt{\frac{N(d,q)(u+\log(2N(d,q)))}{n}}+\frac{N(d,q)(u+\log(2N(d,q)))}{n}\right\}\right],
\end{equation}
where $C$ is a universal numerical constant. The estimate does not require $n/d^q\to0$. In the case $q=0$, the left-hand side equals $\mu_0$ deterministically, so the assertion is immediate.
\end{proposition}
\begin{proof}
Let $N=N(d,q)$ and let $v_i\in\mathbb R^N$ be row $i$ of $\boldsymbol Y_q$. Orthonormality and the spherical-harmonic addition formula give $\mathbb E[v_iv_i^\tau]=I_N$ and $\|v_i\|_2^2=N$ almost surely. Put $Z_i=v_iv_i^\tau-I_N$. Then $\mathbb E Z_i=0$, $\|Z_i\|_{\rm op}\le N+1$, and
\[
\mathbb E Z_i^2=\mathbb E[(v_iv_i^\tau)^2]-I_N=(N-1)I_N.
\]
The self-adjoint matrix Bernstein inequality therefore implies, for $L=u+\log(2N)$,
\[
\left\|\frac1n\boldsymbol Y_q^\tau\boldsymbol Y_q-I_N\right\|_{\rm op}
\le C\left(\sqrt{\frac{NL}{n}}+\frac{NL}{n}\right)
\]
with probability at least $1-e^{-u}$. The nonzero eigenvalues of $\boldsymbol Y_q\boldsymbol Y_q^\tau$ and $\boldsymbol Y_q^\tau\boldsymbol Y_q$ coincide, yielding the claim.
\end{proof}
\begin{proposition}[High-degree spherical harmonic Gram matrix]\label{prop:prop_3_in_montanari}
Fix an integer $q\geq1$. Let $A_d\to\infty$ and suppose that
\[
 n\leq d^q\exp\{-A_d\sqrt{\log d}\}.
\]
For independent uniform points on $\mathbb S^d$, let $\boldsymbol Y_q$
be the $n\times N(d,q)$ evaluation matrix of an orthonormal degree-$q$
spherical-harmonic basis. Then
\[
 \mathbb E\left\|\frac{\boldsymbol Y_q\boldsymbol Y_q^\tau}{N(d,q)}-I_n\right\|_{\rm op}\longrightarrow0.
\]
In particular, the conclusion holds when $n\asymp d^\gamma$ with
$0<\gamma<q$ fixed.
\end{proposition}
\begin{proof}
Proposition~3 of \cite{ghorbani2021linearized} is stated for $N$
independent uniform points on the sphere of radius $\sqrt D$ in
$\mathbb R^D$, with $N\le D^q\exp\{-A_D\sqrt{\log D}\}$.
Their degree-$q$ Gegenbauer Gram matrix is normalized to have unit
diagonal, and the addition formula identifies it with
$\boldsymbol Y_q\boldsymbol Y_q^\tau/N(D-1,q)$.
Rescaling the sphere from radius $\sqrt D$ to radius one preserves
this normalized Gram matrix; setting $D=d+1$ yields the assertion.
For $n\asymp d^\gamma$ and $\gamma<q$, one may choose
$A_D=(q-\gamma)\sqrt{\log D}/2\to\infty$ for all sufficiently
large $d$, so the stated sample-size hypothesis holds.
\end{proof}
\begin{lemma}[Degree-one spherical block]\label{lemma:degree_one_spherical_block}
Let $\boldsymbol X\in\mathbb R^{n\times(d+1)}$ have rows $x_i^\tau$, where the $x_i$ are i.i.d. uniform on $\mathbb S^d$.  For the standard orthonormal basis of degree-one spherical harmonics,
\[
  \boldsymbol Y_1=\sqrt{d+1}\,\boldsymbol X.
\]
If $c_1d\le n\le c_2d$, then, for every $\delta>0$, there are constants $C=C(c_1,c_2,\delta)$ and $d_0=d_0(c_1,c_2,\delta)$ such that, for $d\ge d_0$,
\[
 \mathbb P\!\left(\left\|\frac1n\boldsymbol Y_1\boldsymbol Y_1^\tau\right\|_{\rm op}>C\right)\le\delta.
\]
Consequently,
\[
 \mathbb P\!\left(\frac1n\|\boldsymbol Y_1\boldsymbol D_1\boldsymbol Y_1^\tau\|_{\rm op}>C\mu_1\right)\le\delta.
\]
\end{lemma}
\begin{proof}
Set $z_i=\sqrt{d+1}\,x_i$.  Then $z_i$ is isotropic and has a sub-Gaussian norm bounded by an absolute constant (uniformly in $d$).  The standard sample-covariance bound for isotropic sub-Gaussian vectors yields, for every $u\ge0$,
\[
 \left\|\frac1n\sum_{i=1}^n z_i z_i^\tau-I_{d+1}\right\|_{\rm op}
 \le C_0\left(\sqrt{\frac{d+1}{n}}+\frac{d+1}{n}+\frac{u}{\sqrt n}+\frac{u^2}{n}\right)
\]
with probability at least $1-2e^{-c_0u^2}$.  Since $n\asymp d$, choosing $u$ depending only on $\delta$ proves the first claim.  The nonzero eigenvalues of $\boldsymbol Y_1\boldsymbol Y_1^\tau/n$ and $\boldsymbol Y_1^\tau\boldsymbol Y_1/n$ coincide, and $\boldsymbol D_1=\mu_1 I_{N(d,1)}$, which proves the second claim.
\end{proof}
\vspace{10pt}
The following proofs aim at bounding the eigenvalues of $K_{\text{main}}$ and $K_{\text{residual}}$. Then,
the bounds on the eigenvalues of $\frac{1}{n} K(\boldsymbol{X}, \boldsymbol{X})$ can be obtained by Weyl's inequality.
Therefore, we split the remaining proofs into three parts.
\vspace{5pt}
\paragraph*{Part I: bounding $K_{\text{main}}$}
Let $D=N(p+1)$ and take a thin singular value decomposition
$\boldsymbol Y_{\le p+1}=\sqrt n\,\boldsymbol U\boldsymbol S_\star\boldsymbol V^\tau$,
where $\boldsymbol U\in\mathbb R^{n\times D}$ has orthonormal columns,
$\boldsymbol V\in\mathbb R^{D\times D}$ is orthogonal, and
$\boldsymbol S_\star$ is the diagonal matrix of singular values of
$\boldsymbol Y_{\le p+1}/\sqrt n$. Thus
$\boldsymbol V\boldsymbol S_\star^2\boldsymbol V^\tau=\boldsymbol Y_{\le p+1}^\tau\boldsymbol Y_{\le p+1}/n$; write
$\boldsymbol S_\star=I_D+\boldsymbol\Delta_s$. 
Since $p=\lfloor\gamma/2\rfloor$ and $\gamma>1$ is not an even integer, we have $\gamma>p+1$ (including $p=0$). Consequently, for fixed $p$,
\[n\asymp d^\gamma\gg d^{p+1}\log d\asymp N(p+1)\log(2N(p+1)).\]
From  Lemma \ref{lemma:inner_mendelson_point_control_assist_2}, we further have $n  \gg N(p+1) \log (N(p+1))$.
Hence, from Proposition \ref{prop:lemma_11_in_montanari} with $q=p+1$, we have 
$\boldsymbol{Y}_{\leq p+1}^{\tau} \boldsymbol{Y}_{\leq p+1} / n=\mathbf{I}_{D}+\boldsymbol{\Delta}_{\leq p+1}$, where $\mathbb{E}\left[\|\boldsymbol{\Delta}_{\leq p+1}\|_{\mathrm{op}}\right]=o_d(1)$.
Since $|\sqrt x-1|\le |x-1|$ for $x\ge0$, the singular-value perturbation satisfies
$\|\boldsymbol\Delta_s\|_{\rm op}\le\|\boldsymbol\Delta_{\le p+1}\|_{\rm op}=o_{d,\mathbb P}(1)$.
On the event $\Omega_1=\{ \left\|\boldsymbol{\Delta}_s\right\|_{\mathrm{op}} \leq 1/4\}$, define
$B=(I_{N(p+1)}+\boldsymbol\Delta_s)V^\tau\boldsymbol D_{\le p+1}V(I_{N(p+1)}+\boldsymbol\Delta_s)^\tau$.
The nonzero eigenvalues of $K_{\rm main}$ are those of $B$. On $\Omega_1$, the singular values of $I+\boldsymbol\Delta_s$ lie in $[3/4,5/4]$. The min--max principle applied to this positive semidefinite congruence therefore gives, for every $1\le j\le N(p+1)$,
\[
\frac9{16}\lambda_j(\boldsymbol D_{\le p+1})\le\lambda_j(B)\le\frac{25}{16}\lambda_j(\boldsymbol D_{\le p+1}).
\]
In particular, the slightly weaker bounds used below are
\begin{equation}\label{eqn_182_main_part_inner}
\lambda_{N(p)}(K_{\rm main})\ge\frac7{16}\mu_p.
\end{equation}
\begin{equation}\label{eqn_183_main_part_inner}
\lambda_{N(p)+1}(K_{\rm main})\le\frac{25}{16}\mu_{p+1}.
\end{equation}
\vspace{5pt}
\paragraph*{Part II: bounding $K_{\text{residual}}$}
For each fixed $q=p+k$ with $2\le k\le p+3$, apply Proposition~\ref{prop:eqn_72_in_montanari} with $u=\log(3(p+2)/\delta)$. Since $N(d,q)\asymp d^q$, $\mu_q\asymp d^{-q}$, and $n\asymp d^\gamma$, the resulting bound is
\[
\left\|\frac{\mu_q}{n}\boldsymbol Y_q\boldsymbol Y_q^\tau\right\|_{\rm op}
\le C_{q,\delta}\left(d^{-q}+\sqrt{\frac{\log d}{nd^q}}+\frac{\log d}{n}\right)
=o(d^{-(p+1)})=o(\mu_{p+1}).
\]
Indeed $q\ge p+2$ and $\gamma>p+1$ in the present non-even regime $\gamma>1$, so each displayed term is $o(d^{-(p+1)})$. This bound applies also when $q\le\gamma$, unlike the previously cited high-degree near-identity estimate. A union bound over the $p+2$ fixed degrees gives a failure probability at most $\delta/3$ and, for sufficiently large $d$, a total contribution at most $\mu_{p+1}/3$.
For the remaining infinite tail, we do not sum the fixed-degree $o_d(1)$ statement in Proposition \ref{prop:prop_3_in_montanari}.  Instead, we use a bound that is uniform after summation.  Put $q_0=2p+4$ and
\[
 R_{\ge q_0}:=\frac1n\sum_{q\ge q_0}\mu_q\boldsymbol Y_q\boldsymbol Y_q^\tau.
\]
By the addition formula, $\sum_{j=1}^{N(d,q)}Y_{q,j}(x)^2=N(d,q)$, while orthonormality gives, for $i\ne j$,
\[
 \mathbb E\left[\left\{\frac{(\boldsymbol Y_q\boldsymbol Y_q^\tau)_{ij}}{N(d,q)}\right\}^2\right]=\frac1{N(d,q)}.
\]
Hence, with $W_q=N(d,q)^{-1}\boldsymbol Y_q\boldsymbol Y_q^\tau-I_n$.  The matrices $W_q$ need not be independent: the following argument uses only the triangle inequality, linearity of expectation for a nonnegative summable series, and then Markov's inequality.  In particular,
\[
 \mathbb E\|W_q\|_{\rm op}\le \mathbb E\|W_q\|_F\le \frac{n}{\sqrt{N(d,q)}}.
\]
Therefore
\begin{align*}
 \mathbb E\left\|R_{\ge q_0}-\frac{\kappa_{q_0}}n I_n\right\|_{\rm op}
 &\le \sum_{q\ge q_0}\frac{\mu_qN(d,q)}n\,\mathbb E\|W_q\|_{\rm op}\\
 &\le \sum_{q\ge q_0}\mu_q\sqrt{N(d,q)}\\
 &\le \frac{\kappa_{q_0}}{\sqrt{N(d,q_0)}}
 \le \frac{\kappa}{\sqrt{N(d,q_0)}}=O_{\kappa}(d^{-(p+2)}),
\end{align*}
where $\kappa_{q_0}=\sum_{q\ge q_0}\mu_qN(d,q)\le\sup_xK(x,x)\le\kappa$.  Since $q_0=2p+4$, the last display is $O(d^{-(p+2)})=o(\mu_{p+1})$, while $n^{-1}=o(\mu_{p+1})$ for $\gamma>2p$.  More precisely, for every fixed $\delta\in(0,1)$, the preceding $o(\mu_{p+1})$ expectation bound permits us to choose $d$ sufficiently large that $\mathbb E\|R_{\ge q_0}\|_{\rm op}\le (\delta^2/9)\mu_{p+1}$. Markov's inequality then gives
\[
 \mathbb P\left(\|R_{\ge q_0}\|_{\rm op}>\frac{\delta}{3}\mu_{p+1}\right)\le\frac{\delta}{3}.
\]
Combining this estimate with the finitely many degrees $p+2,\ldots,2p+3$ controlled above (whose total contribution is $o_{\mathbb P}(\mu_{p+1})$, and hence at most $(1-\delta/3)\mu_{p+1}$ with failure probability at most $\delta/3$ for all sufficiently large $d$), we obtain
\[
 \mathbb P\left(\|K_{\rm residual}\|_{\rm op}>\mu_{p+1}\right)<\frac{2\delta}{3}
\]
for all sufficiently large $d$.
\paragraph*{Part III: bounding the empirical matrix}
When $d \geq \mathfrak{C}$, where $\mathfrak{C}$ is a sufficiently large constant depending only on $c_1$, $c_2$, and $\gamma$, we have $\frac{7}{16} \mu_p
        -\mu_{p+1}
        \geq \frac{1}{4} \mu_p$.
Define the event $\Omega_2 = \left\{
\left\|K_{\text{residual}}\right\|_{\mathrm{op}} \leq \mu_{p+1}
\right\}$. 
On the event $\Omega_1 \cap \Omega_2$, we have
\begin{equation}
    \begin{aligned}
    \widehat\lambda_{N(p)} 
    =
    &~
        \lambda_{N(p)}\left(\frac{1}{n} K(\boldsymbol{X}, \boldsymbol{X})\right)\\
        \overset{\text{Weyl's inequality}}{\geq}
        &~
        \lambda_{N(p)}\left(K_{\text{main}}\right)
        -\|K_{\text{residual}}\|_{\mathrm{op}}
        \\
        \overset{\eqref{eqn_182_main_part_inner}}{\geq}
        &~
        \frac{7}{16} \mu_p
        -\mu_{p+1}
        \geq \frac{1}{4} \mu_p.
    \end{aligned}
\end{equation}
Similarly, we have
\begin{equation}
    \begin{aligned}
    \widehat\lambda_{N(p)+1} 
    =
    &~
        \lambda_{N(p)+1}\left(\frac{1}{n} K(\boldsymbol{X}, \boldsymbol{X})\right)\\
        \overset{\text{Weyl's inequality}}{\leq}
        &~
        \lambda_{N(p)+1}\left(K_{\text{main}}\right)
        +\|K_{\text{residual}}\|_{\mathrm{op}}\\
        \overset{\eqref{eqn_183_main_part_inner}}{\leq}
        &~
        4\mu_{p+1}.
    \end{aligned}
\end{equation}
For every fixed $\delta>0$, the estimate $\mathbb E\|\boldsymbol\Delta_s\|_{\mathrm{op}}=o_d(1)$ and Markov's inequality give $\mathbb P(\Omega_1^c)\le\delta/3$ for sufficiently large $d$. Part II gives $\mathbb P(\Omega_2^c)\le 2\delta/3$. Hence $\mathbb P(\Omega_1\cap\Omega_2)\ge 1-\delta$, and we obtain \eqref{eqn:empirical_inner_eigenvalues_gamma_greater_than_4}.
\vspace{10pt}
Next, consider $\gamma\in(0,1)$, so that $p=0$.  We again avoid summing fixed-degree $o_d(1)$ statements.  Split
\[
 \frac1nK_{>0}(\boldsymbol X,\boldsymbol X)
 =\frac1n\sum_{q=1}^{2}\mu_q\boldsymbol Y_q\boldsymbol Y_q^\tau
   +R_{\ge3}.
\]
For each of the two fixed degrees $q=1,2$, Proposition \ref{prop:prop_3_in_montanari} applies because $n\asymp d^\gamma\ll d^{q-\delta_q}$ for a suitable fixed $\delta_q>0$. Hence
\[
 \left\|\frac{\mu_q}{n}\boldsymbol Y_q\boldsymbol Y_q^\tau\right\|_{\rm op}
 =\frac{\mu_qN(d,q)}n\{1+o_{\mathbb P}(1)\}=O_{\mathbb P}(n^{-1}),\qquad q=1,2.
\]
For the infinite tail, the same addition-formula/Frobenius calculation used above gives
\[
 \mathbb E\left\|R_{\ge3}-\frac{\kappa_3}{n}I_n\right\|_{\rm op}
 \le \frac{\kappa_3}{\sqrt{N(d,3)}}=O(d^{-3/2})=o(n^{-1}),
\]
since $\gamma<1$.  Consequently
\begin{equation}\label{eq:high_tail_gamma_lt_one}
 \left\|\frac1nK_{>0}(\boldsymbol X,\boldsymbol X)\right\|_{\rm op}=O_{\mathbb P}(n^{-1}).
\end{equation}
Combining this with the rank-one degree-zero block and Weyl's inequality yields
\[
 \widehat\lambda_{N(0)+1}<\frac{\mathfrak C_1}{n},\qquad
 \mu_0/4<\widehat\lambda_{N(0)}.
\]
In fact, the second inequality holds deterministically: the degree-zero contribution to $K(\boldsymbol X,\boldsymbol X)/n$ is $\mu_0\boldsymbol 1\boldsymbol 1^\tau/n$, and all other Mercer blocks are positive semidefinite. Thus $\widehat\lambda_1\ge\mu_0$ by the min--max principle. The only random event required here is the upper bound for $\widehat\lambda_2$, which holds with probability at least $1-\delta$ for all sufficiently large $d$.
\vspace{10pt}
Finally, let $\gamma=1$.  Separate degrees one and two and collect degrees $q\ge3$ in $R_{\ge3}$. Proposition \ref{prop:prop_3_in_montanari} gives
\[
 \left\|\frac{\mu_2}{n}\boldsymbol Y_2\boldsymbol Y_2^\tau\right\|_{\rm op}=O_{\mathbb P}(n^{-1}),
\]
while the same summable tail estimate gives
\[
 \mathbb E\left\|R_{\ge3}-\frac{\kappa_3}{n}I_n\right\|_{\rm op}
 \le O(d^{-3/2})=o(n^{-1}).
\]
Thus the contribution of all degrees $q\ge2$ is $O_{\mathbb P}(n^{-1})$ in operator norm.  It remains only to control the degree-one block.
Furthermore, by Lemma \ref{lemma:degree_one_spherical_block}, for any $\delta>0$, there exist constants $\mathfrak{C}''$ and $\mathfrak{C}_1$ depending only on $c_1$, $c_2$, and $\delta$, such that when $d \geq \mathfrak{C}''$, we have
$$
\mathbb{P}\left( \frac{1}{n} \|\boldsymbol{Y}_{1} \boldsymbol{D}_{1} \boldsymbol{Y}_{1}^{\tau}\|_{\mathrm{op}} \geq \mathfrak{C}_1\mu_1 \right) \leq \frac{\delta}{2}.
$$
For any given $\delta>0$, 
let $d \geq \mathfrak{C}=\max\{\mathfrak{C}', \mathfrak{C}''\}$, where $\mathfrak{C}'$ is the constant (depending only on $c_{1}$ and $c_{2}$ ) introduced in Lemma \ref{lemma:inner_mendelson_point_control_assist_2} and $\mathfrak{C}''$ is the constant (depending only on $c_{1}$, $c_{2}$, and $\delta$ ) introduced as the previous paragraph. 
Since $n\le c_2d$, the fixed-degree spectral estimate gives $\mu_1\le C/d\le Cc_2/n$, where $C$ also depends on the fixed kernel coefficients.
The degree-zero block again gives the deterministic inequality $\widehat\lambda_1\ge\mu_0$. Combining the degree-one estimate with the $q\ge2$ tail bound and Weyl's inequality gives
\begin{equation}
    \begin{aligned}
        \widehat\lambda_{N(0)+1} 
        &<  
        \frac{\mathfrak{C}_3}{n}, \quad
         \mu_0 / 4
         <
         \widehat\lambda_{N(0)},
    \end{aligned}
    \end{equation}
with probability at least $1-\delta$, where $\mathfrak{C}_3$ may also depend on the fixed kernel parameters (in particular the degree-one spectral bound). The degree-one and degree-two-plus-tail events can each be assigned failure probability $\delta/2$; the degree-zero lower bound is deterministic.
\hfill $\square$
\vspace{10pt}
\noindent \proofname~of Lemma \ref{lemma:block_transfer_inner}:
Let $P_{\le p}$ and $P_{>p}$ denote the population $L^2$-orthogonal projections onto the spherical harmonics of degrees at most $p$ and greater than $p$, respectively, and write $g=g_{\le p}+g_{>p}$.
For the low-frequency part, Proposition \ref{prop:lemma_11_in_montanari} is applied once to the \emph{entire concatenated block} $\boldsymbol Y_{\le p}=(\boldsymbol Y_0,\ldots,\boldsymbol Y_p)$, not separately degree by degree. Thus its matrix $\boldsymbol Y_{\le p}^{\tau}\boldsymbol Y_{\le p}/n$ already contains all empirical cross-degree terms. Since $N(p)\asymp d^p$ and $n\asymp d^\gamma$ with $p<\gamma$, its hypothesis $n\gg N(p)\log N(p)$ holds. Hence
\[
 \mathbb E\left\|\boldsymbol Y_{\le p}^{\tau}\boldsymbol Y_{\le p}/n-I\right\|_{\rm op}=o_d(1).
\]
For the fixed failure budget $\delta>0$ in the statement, Markov's inequality therefore implies that, for all sufficiently large $d$,
\[
 \mathbb P\left(\left\|\boldsymbol Y_{\le p}^{\tau}\boldsymbol Y_{\le p}/n-I\right\|_{\rm op}>\frac12\right)\le\frac{\delta}{2}.
\]
On the complementary event, uniformly for every $h\in P_{\le p}\mathcal H$,
\begin{equation}\label{eq:low_block_isometry_transfer}
 \frac12\|h\|_{L^2}^2\le \|h\|_n^2\le \frac32\|h\|_{L^2}^2.
\end{equation}
No separate union bound over degrees and no additional control of cross terms is required, because both are already included in the concatenated Gram matrix. Notice also that Lemma \ref{lemma:block_transfer_inner} only claims probability $1-\delta$; an exponentially small failure probability is neither claimed nor needed. We allocate the remaining failure probability $\delta/2$ to the high-frequency event below.
We next control the high-frequency part uniformly. Let $K_{>p}$ be the kernel obtained by retaining only degrees $q>p$. For every $h\in P_{>p}\mathcal H$, the reproducing property of the restricted RKHS gives
\begin{equation}\label{eq:high_empirical_eval_transfer}
 \|h\|_n^2\le \lambda_{\max}\!\left(\frac1nK_{>p}(\boldsymbol X,\boldsymbol X)\right)\|h\|_{\mathcal H}^2.
\end{equation}
We claim that, with probability at least $1-\delta/2$ for all sufficiently large $d$,
\begin{equation}\label{eq:high_kernel_op_transfer}
 \lambda_{\max}\!\left(\frac1nK_{>p}(\boldsymbol X,\boldsymbol X)\right)
 \le C\left(d^{-(p+1)}+n^{-1}\right).
\end{equation}
For $\gamma>1$, split this matrix into its degree-$(p+1)$ block and the remainder $K_{\rm residual}$ in \eqref{eqn_178_inner_decomp_emp_matrix}. Proposition \ref{prop:lemma_11_in_montanari} with $q=p+1$ yields
\[
 \left\|\frac{\mu_{p+1}}n\boldsymbol Y_{p+1}\boldsymbol Y_{p+1}^{\tau}\right\|_{\rm op}
 \le C\mu_{p+1}\le Cd^{-(p+1)},
\]
and Part II of the proof of Lemma \ref{condition_inner_3} bounds the remainder by $C/n+Cd^{-(p+1)}$; in particular, its infinite tail is controlled there by a summable Frobenius estimate. For $0<\gamma<1$, the high-frequency decomposition established in the proof of Lemma \ref{condition_inner_3} gives directly $\|K_{>0}(\boldsymbol X,\boldsymbol X)/n\|_{\rm op}\le C/n$. For $\gamma=1$, the same conclusion follows from the degree-one spherical-block bound together with the high-frequency tail bound in that proof. This proves \eqref{eq:high_kernel_op_transfer}.
On the population side, Lemma \ref{lemma:monotone_of_eigenvalues_of_inner_product_kernels} and Lemma \ref{lemma:inner_edr} imply
\begin{equation}\label{eq:high_population_transfer}
 \|h\|_{L^2}^2\le Cd^{-(p+1)}\|h\|_{\mathcal H}^2,\qquad h\in P_{>p}\mathcal H.
\end{equation}
Now let $\|g\|_{\mathcal H}\le M$. Since the population spectral projections are orthogonal also in the RKHS norm, both projected RKHS norms are at most $M$. Equations \eqref{eq:high_empirical_eval_transfer}--\eqref{eq:high_kernel_op_transfer} yield
\[
 \|g_{>p}\|_n^2\le CM^2\{d^{-(p+1)}+n^{-1}\}.
\]
Using \eqref{eq:low_block_isometry_transfer} and the triangle inequality in the empirical norm,
\begin{align*}
 \|g_{\le p}\|_{L^2}^2
 &\le 2\|g_{\le p}\|_n^2
 \le 4\|g\|_n^2+4\|g_{>p}\|_n^2\\
 &\le 4\|g\|_n^2+CM^2\{d^{-(p+1)}+n^{-1}\}.
\end{align*}
Finally, \eqref{eq:high_population_transfer} and $L^2$ orthogonality give
\[
 \|g\|_{L^2}^2=\|g_{\le p}\|_{L^2}^2+\|g_{>p}\|_{L^2}^2
 \le C\left\{\|g\|_n^2+M^2d^{-(p+1)}+\frac{M^2}{n}\right\}.
\]
All events used above are uniform over deterministic subspaces or RKHS balls, so the conclusion remains valid for data-dependent $g$. \hfill $\square$
\vspace{10pt}
\noindent \proofname~of Lemma \ref{lemma_B_t:thm:empirical_loss_modified}:
Let $\Sigma=\operatorname{diag}(\widehat\lambda_1,\ldots,\widehat\lambda_n)$ and write
$U^\tau f_\star(\boldsymbol X)=\sqrt n\,\Sigma^{1/2}b$, where
$\|b\|_2\le\|f_\star\|_{\mathcal H}\le1$; zero eigenvalue
coordinates are set to zero.  The bias term satisfies
\[
\mathbf B_t^2=2\sum_{i=1}^n\widehat\lambda_i
 e^{-2t\widehat\lambda_i}b_i^2.
\]
For $i\le J$, the elementary inequality $2e^{-2u}\le u^{-2}$
($u>0$) gives
$2\widehat\lambda_i e^{-2t\widehat\lambda_i}
\le (t^2\widehat\lambda_i)^{-1}
\le(t^2\widehat\lambda_J)^{-1}$.
For $i>J$, $2\widehat\lambda_i e^{-2t\widehat\lambda_i}
\le2\widehat\lambda_{J+1}$.  Hence
\[
\mathbf B_t^2\le
\left\{(t^2\widehat\lambda_J)^{-1}+2\widehat\lambda_{J+1}\right\}\|b\|_2^2
\le(t^2\widehat\lambda_J)^{-1}+2\widehat\lambda_{J+1}.
\]
\hfill $\square$
\vspace{10pt}
\noindent \proofname ~of Lemma \ref{lemma_V_t:thm:empirical_loss_modified}:  Let $H=\left(\bm{I}-e^{-t\Sigma}\right)$ and $P=\sqrt{\frac{2}{n}}H$. Then, $\bm{V}_{t}=\boldsymbol{e}^{\tau}UP^{2}U^{\tau}\boldsymbol{e}
\overset{d}{=}
\boldsymbol{e}^{\tau}P^{2}\boldsymbol{e}$,
where $\boldsymbol{e}=(\boldsymbol{e}_1, \cdots, \boldsymbol{e}_n)^\tau$ and $\boldsymbol{e}_{i}=\bm{y}_{i}-f_{\star}(X_{i})\sim N(0,\sigma^{2})$ for any $1\leq i \leq n$. 
Conditionally on the design $\boldsymbol X$ (and for any stopping time $t$ measurable with respect to $\boldsymbol X$, so that $t$, $U$ and $P$ are fixed), applying Lemma \ref{lemma:bound_in_Wright} with $A=P^{2}$, 
$\delta=\delta_2$, and
$Q=\sum_{i, j=1}^n a_{i j}\boldsymbol{e}_i \boldsymbol{e}_j \overset{d}{=} \bm{V}_{t}$, 
we then have that
\begin{equation}\label{eqn:70_Vt_modified}
    \begin{aligned}
    |Q-\mathbb{E}[Q]| \leq 
    \delta_2,
\end{aligned}
\end{equation}
holds with probability at least $1-\exp\left(-\mathfrak{c}_{1}\min\left\{\frac{\delta_2}{\|A\|_{op}}, \frac{\delta_2^2}{\|A\|^{2}_{F}}\right\}\right)$ where $\mathfrak{c}_1$ is a constant depending only on $\sigma$, and the probability is over the noise term $\boldsymbol{e}$.
It is easy to verify that $\|H\|_{op}\leq 1$ , $\|A\|_{op} \leq \frac{2}{n}$ and 
\begin{align}
tr(H^2)\overset{(\ref{eqn:inequality_lemma_B_t:thm:empirical_loss})}{\leq}
        \sum_{j}\left(1\wedge {t}\widehat{\lambda}_{j}\right)^2
        \leq t^2\left(\frac{J}{t^2} +   \widehat{\lambda}_{J+1}\sum_{j=J+1}^{n}\widehat{\lambda}_{j} \right)
        \leq 
        t^2\left(\frac{J}{t^2} +   \kappa\widehat{\lambda}_{J+1}\right),
\end{align} 
Thus, we have
\begin{equation}
    \begin{aligned}
         \|A\|_{F}^2 &=tr(P^{4})=\frac{4}{n^{2}}tr(H^{4})\leq\frac{4}{n^{2}}tr(H^2)
        \leq   \frac{4t^2}{n^{2}}\left(\frac{J}{t^2} +   \kappa\widehat{\lambda}_{J+1}\right),
        \\
         \bbE[Q\mid\boldsymbol X] &=
         \bbE[\mathbf{V}_{t}\mid\boldsymbol X]
         =\frac{2\sigma^2}{n} tr\left(\left(\mathbf{I}-e^{-{t} \Sigma}\right)^2\right) 
        \leq 
        \frac{2\sigma^2}{n} t^2\left(\frac{J}{t^2} +   \kappa\widehat{\lambda}_{J+1}\right);
    \end{aligned}
\end{equation}
From (\ref{eqn:70_Vt_modified}), we know that there exists a constant $C$ depending on the fixed noise scale, such that we have
\begin{equation}
\begin{aligned}
\mathbf{V}_{t} \leq \mathbb{E}[Q\mid\boldsymbol X] + 
\delta_2
\leq
\frac{2\sigma^2}{n} t^2\left(\frac{J}{t^2} +   \kappa\widehat{\lambda}_{J+1}\right) +
\delta_2,
\end{aligned}
\end{equation}
with probability at least $1-\exp \left(-C \min\left\{\frac{n\delta_2}{2}, \frac{n^2\delta_2^2}{ 4t^2\left(\frac{J}{t^2} +   \kappa\widehat{\lambda}_{J+1}\right)}\right\}
\right)$.
\hfill $\square$

\section{ Optimality of kernel regression with inner product kernels in high dimensions for $\gamma=2,4,6,\cdots$}\label{section_4_ntk_certain_asymp_frameworks}
In this appendix, we use the general empirical-process and metric-entropy results of Section \ref{sec:main_results} to show that the optimal rate for inner-product kernel regression is $n^{-1/2}$ when $n\asymp d^{\gamma}$ and $\gamma=2,4,6,\ldots$.  For completeness, the auxiliary estimates and the proofs of both the upper and lower bounds are given explicitly below.
Thanks to Lemma \ref{lemma:inner_edr}, we can now use Theorem \ref{theorem:restate_norm_diff} to provide an upper bound on the excess risk of kernel regression with the inner product kernel $K^{\inner}$ in high dimensions.
\begin{theorem}[Upper bound]\label{thm:upper_bpund_inner}
    Suppose that $\calH^{\inner}$ is an RKHS associated with $K^{\inner}$ defined on $\mathbb S^{d}$. 
    Let $f_{\widehat{T}}^{\inner}$ be the function defined in \eqref{solution:gradient:flow} where $\widehat{T}^{-1}=\widehat{\varepsilon}_n^{2}$ defined in (\ref{eqn:def_empirical_mendelson_complexity}) and $K=K^{\inner}$.
    Suppose further that Assumptions ~\ref{assu:trace_class}-\ref{assu:coef_of_inner_prod_kernel} hold with $\calH = \calH^{\inner}$.  Fix $\gamma\in\{2,4,6,\ldots\}$ and constants $0<c_1\le c_2<\infty$, and assume $c_1d^{\gamma}\le n\le c_2d^{\gamma}$.
    Then there exist constants $\mathfrak{C},\mathfrak{C}_1,\mathfrak{C}_2,\mathfrak{C}_3$, depending only on $\gamma,c_1,c_2$ (and on the fixed kernel/noise constants from the standing assumptions), such that
    for any $d \geq \mathfrak{C}$,
     we have
\begin{equation}
\begin{aligned}
\left\|{f}_{\widehat{T}}^{\inner} - f_{\star}\right\|_{L^2}^2
\leq \mathfrak{C}_1  n^{-\frac{1}{2}},
\end{aligned}
\end{equation}
with probability at least $1- \mathfrak{C}_2\exp \left\{ - \mathfrak{C}_3  n^{1/2} \right\}$.
\end{theorem}
{
Recall that the eigenvalues $\lambda_j$'s in (\ref{eqn:mercer_decomp}) are of non-increasing order, while the eigenvalues $\mu_k$'s in (\ref{spherical_decomposition_of_inner_main}) are not necessarily non-increasing. However, the minimax lower bound on the excess risk with respect to the RKHS is determined by large eigenvalues. 
Therefore, the following property of the eigenvalues $\{\mu_k\}_{k\geq 0}$ is crucial to determining the minimax lower bound of large-dimensional kernel regression.
}
\begin{lemma}\label{lemma:monotone_of_eigenvalues_of_inner_product_kernels}
    Suppose that Assumptions ~\ref{assu:trace_class}-\ref{assu:coef_of_inner_prod_kernel} hold.
    Fix an integer $p \geq 0$.
    Then, for any $d \geq \mathfrak{C}$, we have
    \begin{equation*}
        \mu_j \leq \frac{\mathfrak{C}_2}{\mathfrak{C}_1} d^{-1} \mu_{p}, \quad j=p+1, p+2, \cdots,
    \end{equation*}
    where $\mathfrak{C}_1$ and $\mathfrak{C}_2$ are given in Lemma \ref{lemma:inner_edr}.
\end{lemma}
We then use Theorem \ref{restate_lower_bound_m_complexity} to show that kernel regression with $K^{\inner}$ achieves the optimal rate under specific asymptotic frameworks.
\begin{theorem}[Minimax lower bound]
\label{thm:lower_inner_large_d}
Let ${\gamma} \in \{2, 4, 6, \cdots\}$ be a fixed integer and fix $0<c_1\le c_2<\infty$. Assume $c_1d^{\gamma}\le n\le c_2d^{\gamma}$.  
    There exist constants $\mathfrak{C}$ and $\mathfrak{C}_1$, depending only on $\gamma,c_1,c_2$ and the fixed kernel/noise constants, such that
    for any $d \geq \mathfrak{C}$,
     we have
\begin{equation}\label{eqn:lower_bound_ntk}
\inf _{\hat{f}} \sup _{f_{\star} \in \mathcal B} \mathbb{E}_{(\bold{X}, \bold{y}) \sim \rho_{f_{\star}}^{\otimes n}}
\left\|\hat{f} - f_{\star}\right\|_{L^2}^2
\geq \mathfrak{C}_1 n^{-\frac{1}{2}},
\end{equation}
where $\hat{f}$ is any estimator of $f_{\star}$, measurable with respect to the observed data set $(\bm{X}, \bm{Y})$,
 $\rho_{f_{\star}}$ is the joint law of $x, y$ given by (\ref{equation:true_model}) with $f=f_{\star}$, and $\mathcal B = \{ f_{\star}\in \calH^{\inner} ~ \mid ~\|f_{\star}\|_{\calH^{\inner}} \leq 1 \}$.
\end{theorem}
Notice that Theorem \ref{thm:upper_bpund_inner} and Theorem \ref{thm:lower_inner_large_d} only show the optimality of kernel regression with $K^{\inner}$ when $n \asymp d^{\gamma}$ for $\gamma = 2, 4, 6, \cdots$. 
In the next section, we will 
modify the existing tools
for bounding the excess risk of kernel regression with $K^{\inner}$, and show the optimality of kernel regression with $K^{\inner}$ when $n \asymp d^{\gamma}$ for any  $\gamma > 0$.

\section{Proof of Theorems in Section \ref{sec:main_results}}
\subsection{Proof of Theorem \ref{theorem:restate_norm_diff}}\label{proof_3.2}
 The proof is divided into four lemmas below:
\begin{lemma}
\label{thm:empirical_loss}
Let $\widehat{\varepsilon}_n$ be the empirical Mendelson complexity defined in \eqref{eqn:def_empirical_mendelson_complexity}.
There exist positive constants $C_2$ and $C_3$, depending at most on the fixed noise scale,
such that we have
\begin{equation}\label{eqn:thm:empirical_loss}
\left\|f_{\widehat T} - f_{\star} \right\|_n^2
\leq (1+e^{-2})\widehat{\varepsilon}_n^2,
\end{equation}
with probability at least $1-C_2 \exp \left(-C_3 n \widehat{\varepsilon}_n^2\right)$, where $\|g\|^{2}_{n}=\frac{1}{n}\sum_{j \leq n}g(x_{j})^{2}$, and the probability is conditional on the design $X$, on the event $\widehat\varepsilon_n>0$.
\end{lemma}
\begin{lemma}
\label{corollary:norm_diff}
Let $\varepsilon_n$ be the population Mendelson complexity defined in \eqref{eqn:def_population_mendelson_complexity}. 
There exist positive constants $C_1$, $C_2$, and $C_3$, depending at most on the fixed kernel envelope and noise scale, such that
for any $M>0$, let $M\mathcal B := \left\{  g \in \mathcal H \mid  \|g\|_{\mathcal H} \leq M  \right\}$, then we have
\begin{align}
\left|\|g\|_n^2-\|g\|_{L^2}^2\right| \leq \frac{1}{2}\|g\|_{L^2}^2+C_1{M^2 \kappa \varepsilon_n^2} \quad \text { for all } g \in M\calB,
  \end{align}
holds with probability at least $1-C_2 e^{-C_3 n\varepsilon_n^2}$, where the probability is over $n$ samples $x_1, \cdots, x_n$.
\end{lemma}
\begin{lemma}
\label{lemma:bound_empirical_and_expected_mandelson_complexities}
Let $\widehat{\varepsilon}_n$ be the empirical Mendelson complexity defined in \eqref{eqn:def_empirical_mendelson_complexity} and $\varepsilon_n$ be the population Mendelson complexity defined in \eqref{eqn:def_population_mendelson_complexity}. 
Under the same assumptions as Theorem \ref{theorem:restate_norm_diff}, 
there exist positive constants $C_1$, $C_2$, $C_3$, $C_4$, and $\mathfrak{C}_0$, depending at most on the fixed kernel envelope and noise scale, such that for any $n \geq \mathfrak{C}_0$,
we have
\begin{equation}\label{eqn:lemma:bound_empirical_and_expected_mandelson_complexities}
    \begin{aligned}
    C_1 {\varepsilon}_n \leq
    \widehat{\varepsilon}_n
    \leq C_2 {\varepsilon}_n,
    \end{aligned}
\end{equation}
holds with probability at least $1-C_3 \exp \left\{ -C_4 n \varepsilon_n^2\right\}$, 
where the probability is over $n$ samples $x_1, \cdots, x_n$.
\end{lemma}
\begin{lemma}\label{lemma: empirical_loss_bound}
There exist positive constants $C_1,c_1$, depending at most on fixed model parameters, such that 
\begin{equation}\label{eqn:lemma: empirical_loss_bound}
	    \begin{aligned}
\|{f}_{\widehat{T}} - f_{\star}\|_{\mathcal H} & \leq 3,
	    \end{aligned}
\end{equation}
holds with probability at least $1-C_1\exp \left\{ - c_1 n \varepsilon_n^2 \right\}$ over the joint design and noise law, provided the critical-radius comparison assumptions hold. Conditionally on the design, the RKHS-norm bound has failure probability at most $\exp(-c_1 n\widehat\varepsilon_n^2)$.
\end{lemma}
The required constants may depend on fixed model parameters. We defer the details of the proofs
to Appendix \ref{appendix_thm_3.3_lemmas}.
Now let's begin the proof of Theorem \ref{theorem:restate_norm_diff}. We work on the intersection of the design-only critical-radius comparison event from Lemma \ref{lemma:bound_empirical_and_expected_mandelson_complexities}, the design-only uniform norm event from Lemma \ref{corollary:norm_diff}, and the empirical-loss and RKHS-norm events from Lemmas \ref{thm:empirical_loss} and \ref{lemma: empirical_loss_bound}. The latter noise-dependent events are combined conditionally on the design and then integrated over the design law. The resulting failure probability is bounded by a finite sum of terms of order $\exp\{-c n\varepsilon_n^2\}$, with constants depending at most on the fixed model parameters.
\begin{itemize}
    \item[a)] Lemma  \ref{thm:empirical_loss} and \ref{lemma:bound_empirical_and_expected_mandelson_complexities} imply that $\|f_{\widehat{T}}-f_{\star}\|_{n}^2\leq (1+e^{-2})\widehat{\varepsilon}_{n}^2\leq C_{1}\varepsilon_{n}^2 $.
    \item[b)] Lemma \ref{lemma: empirical_loss_bound} guarantees 
$
\frac{1}{3}\left( {f}_{\widehat{T}} - f_{\star} \right) \in \mathcal B
$.
\item[c)] Lemma \ref{corollary:norm_diff} then guarantees
    \begin{align}\label{eqn:47_norm_diff}
    \frac{1}{2}\|{f}_{\widehat{T}} - f_{\star}\|_{L^2}^2
    \leq
\|{f}_{\widehat{T}} - f_{\star}\|_n^2 + {9C_1 \kappa\varepsilon_n^2},
  \end{align}
\end{itemize}
On the intersection of the four events corresponding to \eqref{eqn:thm:empirical_loss}, \eqref{eqn:lemma:bound_empirical_and_expected_mandelson_complexities}, \eqref{eqn:lemma: empirical_loss_bound}, and \eqref{eqn:47_norm_diff}, the preceding estimates hold simultaneously. The design-only events are combined with the conditional noise bounds as described above.
We obtain, writing $C_{\rm emp}$ for the constant in part~(a) and $C_{\rm norm}$ for the constant in Lemma~\ref{corollary:norm_diff},
\begin{align*}
    \|f_{\widehat{T}}-f_{\star}\|^{2}_{L^{2}}
    &\leq 2\|{f}_{\widehat{T}} - f_{\star}\|_n^2 +18C_{\rm norm}\kappa\varepsilon_n^2\\
    &\leq (2C_{\rm emp}+18C_{\rm norm}\kappa)\varepsilon_n^2.
\end{align*}
For clarity, let $E_{\rm cr}$ and $E_{\rm unif}$ denote the two design-only events, and let $E_{\rm emp}$ and $E_{\rm RKHS}$ denote the noise-dependent events. On $E_{\rm cr}$, write $a\varepsilon_n\leq\widehat\varepsilon_n\leq b\varepsilon_n$ with fixed $a,b>0$. If their individual failure probabilities are bounded by $A_j\exp(-c_j n\varepsilon_n^2)$ for the design events and, conditional on $X\in E_{\rm cr}$, by $B_j\exp(-d_j n\widehat\varepsilon_n^2)$ for the noise events, then the union bound and conditional expectation give
\[
\mathbb P\bigl((E_{\rm cr}\cap E_{\rm unif}\cap E_{\rm emp}\cap E_{\rm RKHS})^c\bigr)
\leq (A_1+A_2+B_1+B_2)\exp(-c n\varepsilon_n^2),
\]
where $c=\min\{c_1,c_2,d_1a^2,d_2a^2\}>0$. This yields the asserted bound after renaming fixed constants.
\hfill $\square$

\subsection{Proof of Proposition \ref{prop:rc_violated_in_large_dimensions}}
Fix $0<\alpha<1$ and $\delta>0$.  Put
$L_\alpha=\log(144/\alpha^2)$ and choose
$\tilde\varepsilon_d=13\sqrt{\mu_2(d)}/\alpha$.
For all sufficiently large $d$, the fixed-degree spectral estimates give
$\mu_2\asymp d^{-2}$, $\mu_1\asymp d^{-1}$, and
$\mu_0\asymp1$.  In particular $\tilde\varepsilon_d\to0$,
$\tilde\varepsilon_d^2<\mu_1$, and
$(\alpha\tilde\varepsilon_d/12)^2>\mu_2$.
Consequently, the only eigenvalue blocks contributing to either
$K(\tilde\varepsilon_d)$ or
$K(\alpha\tilde\varepsilon_d/12)$ are degrees zero and one,
with total multiplicity $d+2$.  Hence
\[
 K(\alpha\tilde\varepsilon_d/12)
 =K(\tilde\varepsilon_d)+\frac{d+2}{2}L_\alpha.
\]
Moreover,
\[
 K(\tilde\varepsilon_d)
 \geq\frac{d+1}{2}\log\frac{\mu_1}{\tilde\varepsilon_d^2},
 \qquad
 \log\frac{\mu_1}{\tilde\varepsilon_d^2}\longrightarrow\infty.
\]
By Lemmas \ref{lemma_M_2_and_V_2} and
\ref{lemma_entropy_of_RKHS}, including the $\log2$ remainder
in the entropy-number upper bound,
\[
 \frac{M_2(\alpha\tilde\varepsilon_d,\calB)}
 {M_2(\tilde\varepsilon_d,\calB)}
 \leq
 \frac{V_2(\alpha\tilde\varepsilon_d/2,\calB)}
 {V_2(\tilde\varepsilon_d,\calB)}
 \leq
 \frac{K(\alpha\tilde\varepsilon_d/12)+\log2}
 {K(\tilde\varepsilon_d)}
 =1+\frac{(d+2)L_\alpha/2+\log2}
 {K(\tilde\varepsilon_d)}\longrightarrow1.
\]
The ratio is always at least one because packing entropy is
non-increasing in its radius.  Thus its limit along this sequence is
one, which proves the stated liminf bound for every $\delta>0$.
\hfill $\square$
\subsection{Proof of Theorem \ref{restate_lower_bound_m_complexity}}
    Fix an admissible radius $\bar\varepsilon_n$ as in the theorem, with $n\bar\varepsilon_n^2\geq2\log2$.
    We first establish the information-theoretic inequality used below,
    including its metric conventions. Fix $r,s>0$ and choose a finite
    $r$-separated subset $f_1,\ldots,f_M$ of $\calB$ of maximal
    cardinality (or take a finite subset and pass to the supremum).
    Give its index $J$ the uniform distribution and let $Z$ denote the
    $n$ observations generated under $f_J$. An $s$-net of the model
    class in the single-observation square-root KL metric has
    $N=\exp\{V_K(s,\calD)\}$ centers $Q_1,\ldots,Q_N$.
    Assign each $P_{f_j}$ to a center at KL distance at most $s^2$.
    With $Q=N^{-1}\sum_{k=1}^N Q_k^{\otimes n}$, the log-sum
    inequality yields
    \[
      I(J;Z)\leq M^{-1}\sum_{j=1}^M
       D_{\rm KL}(P_{f_j}^{\otimes n}\|Q)
      \leq \log N+ns^2.
    \]
    (The same bound follows for centers within the model class, as in
    Definition \ref{def:covering_entropy}.) By Fano's inequality,
    every decoder of $J$ errs with probability at least
    $1-(I(J;Z)+\log2)/\log M$.
    The nearest-neighbor decoder induced by an estimator $\hat f$
    is correct whenever $\|\hat f-f_J\|_{L^2}<r/2$.
    Since $\log M=M_2(r,\calB)\geq V_2(r,\calB)$,
    we obtain (whenever the entropies are finite and the denominator
    is positive)
    \begin{equation}\label{eqn:yang_direct_form}
    \inf_{\hat f}\sup_{f_\star\in\calB}
    \mathbb P_{f_\star}\!\left(\|\hat f-f_\star\|_{L^2}^2
    \geq r^2/4\right)
    \geq 1-\frac{V_K(s,\calD)+ns^2+\log2}{V_2(r,\calB)}.
    \end{equation}
    If a maximal packing does not attain its supremum, the same
    conclusion follows by finite packings approaching the entropy.
    Set
    \[
        r=\frac{\mathfrak c_2}{6}\bar\varepsilon_n,
        \qquad s=\bar\varepsilon_n.
    \]
    By condition \eqref{eqn:lower_condition_24},
    \[
        V_2(\mathfrak c_2\bar\varepsilon_n,\calB)
        \geq 5n\bar\varepsilon_n^2.
    \]
    Since covering entropy is non-increasing in its radius and $r\leq \mathfrak c_2\bar\varepsilon_n$, it follows that
    \begin{equation}\label{eqn:entropy_at_direct_fano_radius}
        V_2(r,\calB)\geq 5n\bar\varepsilon_n^2.
    \end{equation}
    On the other hand, by the definition of $\bar\varepsilon_n$ and the assumption $n\bar\varepsilon_n^2\geq2\log2$,
    \begin{align*}
        V_K(s,\calD)+ns^2+\log2
        &\leq 2n\bar\varepsilon_n^2+\log2\\
        &\leq \frac52 n\bar\varepsilon_n^2.
    \end{align*}
    Combining this with \eqref{eqn:entropy_at_direct_fano_radius} in \eqref{eqn:yang_direct_form} gives
    \[
    \inf_{\hat f}\sup_{f_\star\in\calB}
    \mathbb P_{(\bold X,\bold y)\sim\rho_{f_\star}^{\otimes n}}
    \left(\|\hat f-f_\star\|_{L^2}^2
    \geq \left(\frac{\mathfrak c_2}{12}\right)^2\bar\varepsilon_n^2\right)
    \geq \frac12.
    \]
    Therefore,
    \[
    \inf_{\hat f}\sup_{f_\star\in\calB}
    \mathbb E_{(\bold X,\bold y)\sim\rho_{f_\star}^{\otimes n}}
    \|\hat f-f_\star\|_{L^2}^2
    \geq
    \frac12\left(\frac{\mathfrak c_2}{12}\right)^2\bar\varepsilon_n^2,
    \]
    which is the desired result.
\hfill $\square$
\subsubsection{Properties of the metric entropy of $\calB$}\label{subsec:properties:metric:entropy}
It is clear that  $V_{2}(\varepsilon, \mathcal B)$ is also the logarithm of the $\varepsilon$-covering number of 
    $\{(\sqrt{\lambda_i} a_i)_{i\geq 1} \mid \sum_i a_i^2 \leq 1 \}\subset \ell^{2}$ (with respect to the $\ell^{2}$ distance), where $\lambda_i$'s are given in (\ref{eqn:mercer_decomp}). Then we have the following lemmas about the metric entropy of $\calB$.
\begin{lemma}
\label{lemma_entropy_of_RKHS}
For any $\varepsilon>0$, let $K(\varepsilon)=1/2\sum_{j: \lambda_j > \varepsilon^2} \log\left({\lambda_j}/{\varepsilon^2}\right)$.  We have
\begin{equation}\label{eqn:137}
	\begin{aligned}
 V_{2}(6\varepsilon, \mathcal B)\leq K(\varepsilon)+\log 2, \qquad K(\varepsilon) \leq V_{2}(\varepsilon, \mathcal B).
	\end{aligned}
\end{equation}
\end{lemma}
\proof
We need the following lemma:
\begin{lemma}[Proposition 1.3.2 in \cite{carl_stephani_1990}] \label{lemma:number}
For a non-increasing sequence $\{\lambda_{i}\}$ of non-negative numbers, let $S$ be an operator from $\ell^{2}$ to itself which is given by
\begin{equation}
	\begin{aligned}
S &: \ell^2 \to \ell^2, \quad (a_i)_{i\geq 1} \to (\sqrt{\lambda_i} a_i)_{i\geq 1}.
	\end{aligned}
\end{equation}
 Let us denote the unit ball in $\ell^{2}$ by $U_{E}$. Then,  we have
    \begin{equation}
        \begin{aligned}
            \sup_{1\leq k < \infty} \left(\frac{\prod_{j=1}^k \sqrt{\lambda_j}}{q}\right)^{\frac{1}{k}} \leq
            \varepsilon_q(S)
            \leq 
            6 \sup_{1\leq k < \infty} \left(\frac{\prod_{j=1}^k \sqrt{\lambda_j}}{q}\right)^{\frac{1}{k}},
        \end{aligned}
    \end{equation}
    where $q$ is a positive integer and
    \begin{equation*}
        \begin{aligned}
            \varepsilon_q(S) = \inf \left\{
            \varepsilon>0 \mid
            \text{ there exist } q \text{ points } a_1, \cdots, a_q \in \ell^2 \text { such that } S(U_E) \subset \cup_{j=1}^{q} B(a_j, \varepsilon)
            \right\}.
        \end{aligned}
    \end{equation*}
\end{lemma}
Proposition 1.3.2 of Carl and Stephani is stated for non-negative,
non-increasing diagonal coefficients, including coefficients equal to zero.
Consequently its real $\ell^2$ version applies directly to finite-rank
and rank-zero diagonal maps; no positive-tail approximation is needed.
The original $q$ denotes the number of covering centers (not a dyadic
entropy index). Its closed-ball covering convention is used here.

The entropy numbers in Lemma \ref{lemma:number} permit centers in the
ambient Hilbert space, whereas Definition \ref{def:covering_entropy}
requires centers in the covered set. These conventions agree for the
ellipsoid $S(U_E)$: it is a closed convex subset of $\ell^2$ (the
diagonal operator $S$ is compact, so its image of the weakly compact
unit ball is norm compact). Its metric projection $P_{S(U_E)}$ is
nonexpansive, and for every $x\in S(U_E)$ and ambient center $z$,
$\|x-P_{S(U_E)}z\|_2\leq\|x-z\|_2$. Thus projecting all centers of
an ambient covering into the ellipsoid preserves the radius and the
number of balls. The two covering numbers are equal, including at
finite rank, so no radius inflation is needed. Moreover, because the ellipsoid is compact,
the minimum covering radius with any fixed finite number of centers is
attained: project centers into the ellipsoid and pass to a convergent
subsequence of center tuples. Thus the entropy-number upper bound at
$6\varepsilon$ yields an actual closed-ball covering at that radius,
not merely coverings at all larger radii.

Now let's begin to prove Lemma \ref{lemma_entropy_of_RKHS}.
For any $\varepsilon>0$, let $m=\#\{j:\lambda_j>\varepsilon^2\}$ and set
$q_\varepsilon=\prod_{j=1}^m(\sqrt{\lambda_j}/\varepsilon)=\exp\{K(\varepsilon)\}$.
Here an empty product equals one. This product is a positive real number, not necessarily an integer covering
cardinality.  For every $k\geq1$, monotonicity of the semiaxes implies
$\prod_{j=1}^k(\sqrt{\lambda_j}/\varepsilon)\leq q_\varepsilon$:
for $k\leq m$ this follows by extending the product with factors
larger than one, and for $k>m$ by extending it with factors at most one.
Consequently, the upper entropy-number bound in Lemma
\ref{lemma:number} is at most $6\varepsilon$ when its cardinality
parameter is at least $q_\varepsilon$. Applying that lemma with the integer
$Q=\lceil q_\varepsilon\rceil$ and using
$q_\varepsilon\leq Q\leq 2q_\varepsilon$ gives
\[
 V_2(6\varepsilon,\calB)\leq\log Q\leq K(\varepsilon)+\log 2.
\]
For completeness, the lower bound does not require rounding an entropy
number.  Let $P_m$ be the orthogonal projection onto the first $m$
coordinates.  The projected RKHS unit ball is the $m$-dimensional
ellipsoid with semiaxes $\sqrt{\lambda_1},\ldots,\sqrt{\lambda_m}$.
Any covering of $\calB$ by $N$ $L^2$-balls of radius $\varepsilon$
projects to a covering of this ellipsoid by $N$ Euclidean balls of
radius $\varepsilon$.  Comparing $m$-dimensional volumes gives
\[
N\,\operatorname{vol}(B_2^m)\varepsilon^m
\geq \operatorname{vol}(B_2^m)
\prod_{j=1}^m\sqrt{\lambda_j},
\]
so that $\log N\geq K(\varepsilon)$.  If $m=0$ then
$K(\varepsilon)=0$ and the same conclusion is immediate.
Thus the bounds needed below are
\[
 V_2(6\varepsilon,\calB)\leq K(\varepsilon)+\log 2,
 \qquad K(\varepsilon)\leq V_2(\varepsilon,\calB).
\]
\hfill $\square$
\begin{remark}[Spectral ellipsoid identification]
For the unit ball $\calB$ of the positive-spectral RKHS, Mercer
orthogonality identifies $f=\sum_j a_j\sqrt{\lambda_j}\phi_j$
with $(a_j)_j\in\ell^2$, $\sum_j a_j^2\leq1$. Under this
identification the $L^2$ distance is
$\|f-g\|_{L^2}^2=\sum_j\lambda_j(a_j-b_j)^2$.
Thus $\calB$ is isometric (as a metric space) to the image of the
unit $\ell^2$ ball under the diagonal map
$S(a_j)=(\sqrt{\lambda_j}a_j)$. Zero population eigenvalues are
excluded from this coordinate representation; they contribute no
nonzero functions to the RKHS. Since the kernel is trace class,
$\lambda_j\to0$, and the diagonal image is totally bounded.
The entropy comparisons above are therefore valid in finite rank as
well, with empty products interpreted as one.
\end{remark}
\begin{lemma}\label{lemma_M_2_and_V_2} 
For any $\varepsilon>0$, we have
$
    M_{2}(2\varepsilon,\calB) \leq V_{2}(\varepsilon,\calB) \leq M_{2}(\varepsilon,\calB)
    $.
\proof 
Suppose $E=\left\{f_1, \ldots, f_M\right\}$ is a maximal $\varepsilon$-packing. Then for all $f \in \mathcal B \backslash E$, we can find $f_i$, such that $\left\|f-f_i\right\| \leq \varepsilon$. Hence $E$ is an $\varepsilon$-net. Therefore, we have $V_{2}(\varepsilon, \calB) \leq M_{2}(\varepsilon, \calB)$. 
On the other side, suppose there exists a $2 \varepsilon$-packing $\left\{f_1, \ldots, f_M\right\}$ and an $\varepsilon$-net $\left\{g_1, \ldots, g_N\right\}$ such that $M \geq N+1$. Then we must have $f_i$ and $f_j$ belonging to the same $\varepsilon$-ball $B\left(g_k, \varepsilon\right)$ for some $i \neq j$ and $k$. This means that we have $\left\|f_i-f_j\right\| \leq 2 \varepsilon$, which leads to a contradiction. Therefore, we have $M_{2}(2\varepsilon, \calB) \leq V_{2}(\varepsilon, \calB)$.
\hfill $\square$
\end{lemma}
\begin{lemma}\label{claim:d_K_and_d_2}
$
    V_{2}\left(\varepsilon, \calB \right) = V_{K}\left(\frac{\varepsilon}{\sqrt{2}\sigma}, \calD\right)
    $.
\proof 
Denote the p.d.f. of $x$ as $\mu(x)$, and the p.d.f. of $y$ given $x$ as $\rho(y \mid x)$.
Since $y \mid x \sim N(f(x), \sigma^2)$, we then have
\begin{equation}
    \begin{aligned}
    \rho(y \mid x) = \frac{1}{\sqrt{2\pi \sigma^2}} \exp\left\{-\frac{(y-f(x))^2}{2\sigma^2}\right\};
    \end{aligned}
\end{equation}
Therefore, for any $f, g \in \mathcal H$, we have 
\begin{equation}
    \begin{aligned}
    \ell_{K}^2(f,g)
    =&\ \int \rho_f(x,y) \log \frac{\rho_f(x,y)}{\rho_g(x,y)} \mathsf{d} x \mathsf{d} y\\
    =&\ \int \rho_f(x,y) \frac{y}{\sigma^2}[f(x)-g(x)] \mathsf{d} x \mathsf{d} y
    +
    \int \rho_f(x,y) \frac{1}{2\sigma^2}[g^2(x)-f^2(x)] \mathsf{d} x \mathsf{d} y\\
    =&\ 
    \int 
    \left(\int 
    y \rho(y \mid x)
    dy\right)
     \frac{f(x)-g(x)}{\sigma^2} \mu(x) \mathsf{d} x 
    +
    \int  \frac{g^2(x)-f^2(x)}{2\sigma^2} \mu(x) \mathsf{d} x\\
    =&\
    \frac{1}{2\sigma^2} \int 
    \left( 2f^2(x)-2f(x)g(x) + g^2(x)-f^2(x) \right)
    \mu(x) \mathsf{d} x
    = \frac{1}{2\sigma^2} \ell_{2}^2(f,g).
    \end{aligned}
\end{equation}
Therefore, from the definition of $V_2$ and $V_K$, we have $V_{2}\left(\varepsilon, \calB \right) = V_{K}\left(\frac{\varepsilon}{\sqrt{2}\sigma}, \calD\right)$. 
\hfill $\square$
\end{lemma}

\section{Auxiliary proofs for inner-product kernels}
\subsection{Proof of Lemma \ref{lemma:inner_edr}}
Fix an integer $p\geq0$ and write $D=d+1$. Let $T=\langle x,X\rangle$, where $X$ is uniform on $\mathbb S^d$ and $x$ is fixed. Let $P_{d,k}$ be the degree-$k$ zonal Gegenbauer polynomial normalized by $P_{d,k}(1)=1$. By the Funk--Hecke formula, $\mu_k=\mathbb E[\Phi(T)P_{d,k}(T)]$. For every fixed integer $k\geq1$, the Rodrigues identity and $k$ integrations by parts yield the exact identity
\[
\mu_k=\frac{\mathbb E[(1-T^2)^k\Phi^{(k)}(T)]}{\prod_{j=0}^{k-1}(d+2j)}.
\]
Indeed, the density of $T$ is proportional to $(1-t^2)^{(d-2)/2}$ on $(-1,1)$, and the boundary terms vanish for sufficiently large $d$ (with $k$ fixed). The normalization constant in the Rodrigues formula is determined by $P_{d,k}(1)=1$; equivalently, the identity can be checked on the monomial $t^k$ and then extended by integration by parts. For $k=0$, $\mu_0=\mathbb E\Phi(T)$.
Since $\mathbb E T^2=1/(d+1)$, we have $T\to0$ in probability. The function $\Phi$ belongs to $C^\infty([-1,1])$, so its fixed-order derivatives are bounded and continuous. Dominated convergence gives
\[
\mathbb E[(1-T^2)^k\Phi^{(k)}(T)]\longrightarrow\Phi^{(k)}(0),
\qquad d^k\Big/\prod_{j=0}^{k-1}(d+2j)\longrightarrow1.
\]
Consequently $d^k\mu_k\to\Phi^{(k)}(0)=k!a_k>0$ for each fixed $k\geq0$. Because $\{0,\ldots,p+1\}$ is finite, for all sufficiently large $d$ (depending on $p$ and $\Phi$), simultaneously for $0\leq k\leq p+1$,
\[
\frac{\Phi^{(k)}(0)}{2d^k}\leq\mu_k\leq\frac{2\Phi^{(k)}(0)}{d^k}.
\]
Here $d^0=1$. Thus one may take $\mathfrak C_1=\frac12\min_{0\leq k\leq p+1}k!a_k>0$ and $\mathfrak C_2=2\max_{0\leq k\leq p+1}k!a_k<\infty$.
For completeness, the exact dimension of the degree-$k$ spherical harmonics on $\mathbb{S}^d$ is
\[
N(d,k)=\binom{d+k}{k}-\binom{d+k-2}{k-2},
\]
where the second binomial coefficient is interpreted as zero for $k=0,1$.
For fixed $k\ge 1$, the first term is $d^k/k!+O_k(d^{k-1})$, whereas the second is $O_k(d^{k-2})$; hence
$N(d,k)=d^k/k!+O_k(d^{k-1})$.  For $k=0$, $N(d,0)=1$.
In particular, the estimates $N(d,k)\asymp_k d^k$ hold simultaneously for $0\le k\le p+1$ because this index set is finite.
\hfill $\square$
\subsection{Proof of Theorem \ref{thm:upper_bpund_inner}}
Let $\varepsilon_n$ be the population Mendelson complexity defined in \eqref{eqn:def_population_mendelson_complexity} with $K=K^{\inner}$.
We need the following lemmas.
{
\begin{lemma}[Restate Lemma  \ref{lemma:inner_edr}]\label{lemma:inner_mendelson_point_control_assist_2}
    Suppose that $p$ is a fixed nonnegative integer and $k\in\{0,1,\ldots,p+1\}$. For $k=0$, use $N(d,0)=1$ and the fixed-degree bound $\mu_0\asymp1$; for positive $k$, the bounds are those of Lemma~\ref{lemma:inner_edr}. There exist constants $\mathfrak{C}_1$ and $\mathfrak{C}_2$ depending on $p$ and the fixed kernel coefficients, such that for any $d \geq \mathfrak{C}$, a sufficiently large constant depending on $p$ and the fixed kernel coefficients, we have
\begin{equation}\label{eqn:lemma:inner_mendelson_point_control_assist_2}
\begin{aligned}
\frac{\mathfrak{C}_1}{d^{k}} &\leq \mu_{k} \leq \frac{\mathfrak{C}_2}{d^{k}},\\
\mathfrak{C}_1 d^k
&\leq
N(d, k) 
\leq
\mathfrak{C}_2 d^k.
\end{aligned}
\end{equation}
\end{lemma}
}
\begin{lemma}\label{lemma:theorem_upper_inner_assist_summation}
    Suppose that $q \in \{1, 2, 3, \cdots \}$. There exists a constant $\mathfrak{C}_3$ depending only on $q$, such that for any $d \geq \mathfrak{C}$, a sufficiently large constant depending only on $q$, we have
\begin{equation}
\begin{aligned}
\mathfrak{C}_3
\leq
\sum_{k=0}^{\infty} N(d, k) \min\{\mu_k, \mu_q\}
\leq
\kappa.
\end{aligned}
\end{equation}
\proof From Assumption ~\ref{assu:trace_class} we have $\sum_{k} N(d, k) \min\{\mu_k, \mu_q\} \leq \sum_{k} N(d, k) \mu_k \leq \kappa$; from Lemma \ref{lemma:inner_mendelson_point_control_assist_2} we have $\sum_{k} N(d, k) \min\{\mu_k, \mu_q\} \geq N(d, q) \mu_q \geq \mathfrak{C}_1^2$.
\hfill $\square$
\end{lemma}
{
\begin{lemma}\label{lemma:inner_mendelson_interval_control}
    Suppose that $\gamma >0$ is a real number and $p$ is an integer satisfying that $\gamma \in [2p,2p+2)$. 
    Then, 
    there exist constants $\mathfrak{C}_1$ and $\mathfrak{C}_2$ depending only on $p$ satisfying that
    for any constants $0<c_1\leq c_2<\infty$, 
    there exists a sufficiently large constant $\mathfrak{C}$ depending only on $\gamma$, $c_1$, and $c_2$, such that
    for any $d \geq \mathfrak{C}$ and any $n\in [c_1 d^{\gamma}, c_2 d^{\gamma}]$, we have
    \begin{equation}
\begin{aligned}
\mathfrak{C}_1 n^{-1/2}   
\leq 
\varepsilon_n^2 
\leq 
\mathfrak{C}_2 n^{-1/2}.
\end{aligned}
\end{equation}
\end{lemma}
}
Now we begin to prove Theorem \ref{thm:upper_bpund_inner} by using Theorem \ref{theorem:restate_norm_diff}.
\begin{itemize}
    \item From Lemma \ref{lemma:inner_mendelson_interval_control}, it is easy to check that for any absolute constant $C$, there exists a sufficiently large constant $\mathfrak{C}$ depending only on $c_1$, $c_2$, and $\gamma$, such that for any $d \geq \mathfrak{C}$, we have $C^2\varepsilon_n^2 \geq 1/n$.
    \item It is assumed that $K^{\inner}$ satisfies Assumption ~\ref{assu:trace_class}.
    \item Therefore, all requirements in Theorem \ref{theorem:restate_norm_diff} are satisfied.
\end{itemize}
From Lemma \ref{lemma:inner_mendelson_interval_control}, we get the desired results.
\hfill $\square$
\subsubsection{Proof of Lemma \ref{lemma:inner_mendelson_interval_control}}\label{sec:proof_lemma:inner_mendelson_interval_control}
We need to apply the Lemma \ref{lemma:sup_lemma1_for_lemma:bound_empirical_and_expected_mandelson_complexities}
and Remark \ref{remark_how_to_use_mendelson_complexity}. 
Suppose that $\gamma\in [2p, 2p+2)$ for some integer $p$.
 Let $
C(p) =\left[{4e^2 \sigma^2\mathfrak{C}_3 } / {\mathfrak{C}_2^2} \right]^{1/2}$, where $\mathfrak{C}_2$ and $\mathfrak{C}_3$ are two constants depending only on $p$ given in Lemma \ref{lemma:inner_mendelson_point_control_assist_2} and the Lemma \ref{lemma:theorem_upper_inner_assist_summation} respectively. 
 It is clear that
\begin{equation}
\begin{aligned}
\varepsilon_{low}^2 
\triangleq C(p) \mu_p \sqrt{d^{2p}/n}
\overset{Lemma \ref{lemma:inner_mendelson_point_control_assist_2}}{\geq}
C(p) \mathfrak{C}_1 \sqrt{\frac{1}{n}},
\end{aligned}
\end{equation}
and  for any $d \geq \mathfrak{C}$, a sufficiently large constant depending only on $p$ and $c_2$, we have
\begin{equation}\label{eqn:eps_low_larger_than_mu}
\begin{aligned}
\frac{\varepsilon_{low}^2}{\mu_{p+1}} 
&\overset{Lemma \ref{lemma:inner_mendelson_point_control_assist_2}}{\geq}
 \frac{\mathfrak{C}_1}{\mathfrak{C}_2} C(p) \sqrt{\frac{d^{2p+2}}{c_2 d^\gamma}}\geq 1.
\end{aligned}
\end{equation}
Therefore, we have
\begin{equation}
\begin{aligned}
&~
\sum_{k=0}^{\infty} N(d, k) \min\{\mu_k, \varepsilon_{low}^2\} 
\overset{(\ref{eqn:eps_low_larger_than_mu})}{\geq}
\sum_{k=0}^{\infty} N(d, k) \min\{\mu_k, \mu_{p+1}\} \\
\overset{Lemma \ref{lemma:theorem_upper_inner_assist_summation}}{\geq}
&~
\mathfrak{C}_3
\overset{\text{Definition of }  C(p)}{=}
\frac{[C(p)]^2}{4e^2\sigma^2} \mathfrak{C}_2^2\\
\overset{Lemma \ref{lemma:inner_mendelson_point_control_assist_2}}{\geq}
&~
\frac{[C(p)]^2}{4e^2\sigma^2} \mu_p^2 d^{2p}
=\frac{n\varepsilon_{low}^4}{4e^2\sigma^2}.
\end{aligned}
\end{equation}
Thus, we know that
$\varepsilon_n^2 
\geq 
\varepsilon_{low}^2 \geq
C(p) \mathfrak{C}_1 \sqrt{1/n}$.
We next prove the upper bound without introducing a neighboring spectral block.  By Assumption~\ref{assu:trace_class},
\[
\sum_{k=0}^{\infty}N(d,k)\min\{\mu_k,t^2\}
\leq \sum_{k=0}^{\infty}N(d,k)\mu_k\leq\kappa
\qquad\text{for every }t>0.
\]
Set
\[
\varepsilon_{upp}^2:=2e\sigma\sqrt{\kappa/n}.
\]
Then
\[
\sum_{k=0}^{\infty}N(d,k)\min\{\mu_k,\varepsilon_{upp}^2\}
\leq\kappa
=\frac{n\varepsilon_{upp}^4}{4e^2\sigma^2}.
\]
By the definition (equivalently, the generalized-inverse characterization) of the population Mendelson critical radius, this implies
\[
\varepsilon_n^2\leq\varepsilon_{upp}^2
=2e\sigma\sqrt{\kappa/n}.
\]
Together with the lower bound above, this proves $\varepsilon_n^2\asymp n^{-1/2}$ uniformly over the stated polynomial regime. \hfill $\square$

\subsection{Proof of Lemma \ref{lemma:monotone_of_eigenvalues_of_inner_product_kernels}}
We need the following lemma:
\begin{lemma}\label{lemma:assist_1_monotone_of_eigenvalues_of_inner_product_kernels}
    For any integer $q \geq 0$, we have $\mu_{q} > \mu_{q+2}$.
    \proof Deferred to the end of this subsection.
\end{lemma}
Now let's begin to prove Lemma \ref{lemma:monotone_of_eigenvalues_of_inner_product_kernels}. 
From Lemma \ref{lemma:inner_edr}, for any $d \geq \mathfrak{C}$, a sufficiently large constant depending on $p$ and the fixed kernel parameters, we have
\begin{equation*}
        \mu_{p+1} \leq \frac{\mathfrak{C}_2}{\mathfrak{C}_1} d^{-1} \mu_{p}, \quad \mu_{p+2} \leq \frac{\mathfrak{C}_2}{\mathfrak{C}_1} d^{-2} \mu_{p}.
    \end{equation*}
Then, from Lemma \ref{lemma:assist_1_monotone_of_eigenvalues_of_inner_product_kernels}, we further have
    \begin{equation*}
        \mu_j \leq \max\{\mu_{p+1}, \mu_{p+2}\} \leq  \frac{\mathfrak{C}_2}{\mathfrak{C}_1} d^{-1} \mu_{p}, \quad j=p+1, p+2, \cdots.
    \end{equation*}
\hfill $\square$
\vspace{10pt}
\noindent \proofname ~of Lemma \ref{lemma:assist_1_monotone_of_eigenvalues_of_inner_product_kernels}:
We give a direct proof using the Funk--Hecke formula. The sphere in this paper is $\mathbb S^d\subset\mathbb R^{d+1}$, equipped with the uniform probability measure. For integers $k,q\geq0$, put
\[
 b_{k,q}(d):=\frac{(k+2q)!\,\Gamma((d+1)/2)}{2^{k+2q}q!\,\Gamma((d+1)/2+k+q)}.
\]
The degree-$k$ eigenvalue of the monomial kernel $\langle x,y\rangle^{k+2q}$ equals $b_{k,q}(d)$; a monomial of smaller degree or opposite parity contributes zero. To see the normalization directly, the Rodrigues--Funk--Hecke identity used in the proof of Lemma \ref{lemma:inner_edr}, applied to $\Phi(t)=t^{k+2q}$, gives
\[
 \frac{(k+2q)!}{(2q)!}\frac{\mathbb E[T^{2q}(1-T^2)^k]}{\prod_{j=0}^{k-1}(d+2j)}=b_{k,q}(d),
\]
where $T$ is the first coordinate of a uniform point on $\mathbb S^d$; the equality follows from the beta integral. In particular $b_{k,q}(d)\geq0$ and $b_{k,0}(d)=k!/[ (d+1)(d+3)\cdots(d+2k-1)]$.
The kernel power series has nonnegative coefficients and converges uniformly on $[-1,1]$ because $\sum_{j\geq0}a_j=\Phi(1)<\infty$. Since each monomial integral operator has norm at most one, termwise application of Funk--Hecke is justified and
\[
 \mu_k=\sum_{q\geq0}a_{k+2q}b_{k,q}(d).
\]
For $q\geq1$, gamma recurrence yields
\[
 \frac{b_{k+2,q-1}(d)}{b_{k,q}(d)}=\frac{q}{(d+1)/2+k+q}<1.
\]
Consequently,
\[
 \mu_{k+2}=\sum_{q\geq1}a_{k+2q}b_{k,q}(d)\frac{q}{(d+1)/2+k+q}\leq\mu_k.
\]
If $\mu_k>0$, at least one nonnegative summand in its expansion is positive and is either omitted or strictly decreased, so $\mu_{k+2}<\mu_k$. The standing strictly positive coefficient assumption ensures $\mu_k>0$ for every $k$. Moreover the $q=0$ summand gives
\[
 \mu_k\geq a_k\frac{k!}{(d+1)(d+3)\cdots(d+2k-1)},
\]
with the empty product equal to one. This proof does not invoke an external eigenvalue formula. \hfill $\square$
\vspace{10pt}
\subsection{Proof of Theorem \ref{thm:lower_inner_large_d}}
Let 
$\bar\varepsilon_n$ be the generalized-inverse covering critical radius defined in Remark \ref{remark:classical_entropy_radii}, for the class associated with $\calH^{\inner}$.
We need the following lemma.
\begin{lemma}\label{lemma:bound134_2}
    Suppose that $\gamma \in \{2, 4, 6, \cdots\}$ is an integer and $p=\gamma/2$. 
    Then, 
    for any constants $0<c_1\leq c_2<\infty$, 
    there exist constants $\mathfrak{C}$, $\mathfrak{C}_1$, and $\mathfrak{C}_2$ depending only on $\gamma$, $c_1$, $c_2$, and the fixed kernel/noise parameters, such that
    for any $d \geq \mathfrak{C}$ and any $n\in [c_1 d^{\gamma}, c_2 d^{\gamma}]$, we have
\begin{equation}\label{eqn:lemma:bound134_2}
\begin{aligned} 
\mathfrak{C}_1 \sqrt{\frac{1}{n}}
&<
\bar\varepsilon_n^2
<
\mathfrak{C}_2 \sqrt{\frac{1}{n}}.
\end{aligned}
\end{equation}
\proof Deferred to the end of this subsection.
\end{lemma}
Now we prove Theorem \ref{thm:lower_inner_large_d} directly with an
admissible upper trial radius, avoiding any need to assume that a
covering-entropy generalized inverse satisfies an equality.  Fix
$\gamma=2p$, $p\ge1$, and define
\[
 r_n^2:=\widetilde C(p)\mu_p\sqrt{d^{2p}/n},
\]
where $\widetilde C(p)=\tilde C(p)$ is the fixed positive constant used in the
upper-trial construction in the proof of Lemma \ref{lemma:bound134_2}.
That construction (Claim \ref{claim_1} below, including the $\log2$
remainder in the entropy upper bound) proves, uniformly for
$c_1d^{2p}\le n\le c_2d^{2p}$ and all sufficiently large $d$, that
\[
 V_K(r_n,\calD)=V_2(\sqrt2\sigma r_n,\calB)\le nr_n^2.
\]
Moreover, the fixed-degree bounds yield
$r_n^2\asymp n^{-1/2}$ and $nr_n^2\asymp\sqrt n\to\infty$.
Thus $r_n$ is an admissible choice of the radius in Theorem
\ref{restate_lower_bound_m_complexity}, whether or not the entropy
crossing is attained.
It remains to verify the strengthened covering-entropy condition required by Theorem \ref{restate_lower_bound_m_complexity}.  Write $b_p>0$ for a fixed constant such that
$N(d,p)\ge b_pd^p$ for large $d$.  Also let $B_p>0$ satisfy
$nr_n^2\le B_pd^p$ uniformly in the prescribed sampling regime.
Choose a fixed $0<\mathfrak c_2<1$ so small that
\[
 b_p\log\left(\frac{\sqrt{c_1}}
 {\mathfrak c_2^2\widetilde C(p)}\right)>10B_p.
\]
Since $r_n^2=\widetilde C(p)\mu_p\sqrt{d^{2p}/n}$,
we have, more simply, $\mu_p/r_n^2=\sqrt{n/d^{2p}}/\widetilde C(p)
\ge\sqrt{c_1}/\widetilde C(p)$; hence, decreasing
$\mathfrak c_2$ if necessary, the degree-$p$ term alone gives
\[
 \sum_{k:\mu_k>\mathfrak c_2^2r_n^2}N(d,k)
 \log\frac{\mu_k}{\mathfrak c_2^2r_n^2}
 \ge b_pd^p\log\frac{\sqrt{c_1}}
 {\mathfrak c_2^2\widetilde C(p)}
 \ge10nr_n^2.
\]
The entropy lower bound in Lemma \ref{lemma_entropy_of_RKHS}
therefore implies
$V_2(\mathfrak c_2r_n,\calB)\ge5nr_n^2$.
Theorem \ref{restate_lower_bound_m_complexity} now yields a minimax
risk at least $(\mathfrak c_2/12)^2r_n^2/2\gtrsim n^{-1/2}$.
Here $nr_n^2\asymp\sqrt n$ also verifies the theorem's condition
$nr_n^2\geq2\log2$ for all sufficiently large $d$.
\hfill $\square$
\vspace{10pt}
\noindent \proofname ~of Lemma \ref{lemma:bound134_2}: 
Throughout this proof $p=\gamma/2\geq1$ is fixed, as required in the lemma.
Let $C(p) = \min\{ \sqrt{c_1}/(4\sigma^2),$
$\frac{1}{2} \mathfrak{C}_1 \log\left(2\right) 
/\left( \sqrt{c_2} \mathfrak{C}_2 \right)\}$. It is clear that
\begin{equation}
\begin{aligned}
\bar\varepsilon_{low}^2
\triangleq
C(p) \mu_p \sqrt{d^{2p} / n }
< \mu_p / (2\sigma^2).
\end{aligned}
\end{equation}
Therefore, for any $d \geq \mathfrak{C}$, where $\mathfrak{C}$ is a sufficiently large constant depending on $p,c_1,c_2$ and the fixed kernel/noise parameters, we have
\begin{equation}
\begin{aligned}
&~
V_{2}(\sqrt{2}\sigma \bar\varepsilon_{low}, \calB)
\overset{Lemma \ref{lemma_entropy_of_RKHS}}{\geq} 
K\left( \sqrt{2}\sigma \bar\varepsilon_{low} \right)
\geq
\frac{1}{2}N(d,p)\log\left(\frac{\mu_p}{2\sigma^2\bar\varepsilon_{low}^2}\right)\\ \overset{\text{Lemma } \ref{lemma:inner_mendelson_point_control_assist_2} \text{ and Definition of } \bar\varepsilon_{low}^2}{\geq}
&~
\frac{1}{2} \mathfrak{C}_1 d^p \log\left(\frac{\sqrt{c_1}}{2\sigma^2 C(p)}\right) 
\overset{\text{Definition of } C(p)}{\geq} 
\frac{1}{2} \mathfrak{C}_1 d^p \log\left(2\right) \\
\overset{\text{Definition of } C(p)}{\geq}
&~
C(p) \sqrt{c_2} \mathfrak{C}_2 d^{p}
\overset{\text{Lemma } \ref{lemma:inner_mendelson_point_control_assist_2}}{\geq}
C(p) \mu_p \sqrt{n d^{2p}} = n\bar\varepsilon_{low}^2.
\end{aligned}
\end{equation}
By Lemma \ref{claim:d_K_and_d_2}, the preceding display states that the entropy curve is still above the quadratic curve at $\bar\varepsilon_{low}$. Under the generalized-inverse convention for $\bar\varepsilon_n$, this implies $\bar\varepsilon_n^2\geq\bar\varepsilon_{low}^2$.
On the other hand, let 
$\tilde C(p) = \max\{ 72\sqrt{c_2}/(2\sigma^2),$
$p \mathfrak{C}_2 \log\left(2\right) 
/\left( \sqrt{c_1} \mathfrak{C}_1 \right)\}$.
It is clear that
\begin{equation}\label{eqn:def_upp_for_lower_bound_inner}
\begin{aligned}
\bar\varepsilon_{upp}^2
\triangleq
\tilde C(p) \mu_p \sqrt{d^{2p} / n }
\geq36\mu_p / (2\sigma^2).
\end{aligned}
\end{equation}
Furthermore, from Lemma \ref{lemma:inner_mendelson_point_control_assist_2} and Lemma \ref{lemma:monotone_of_eigenvalues_of_inner_product_kernels}, one can check the following claim:
\begin{claim}\label{claim_1}
For any fixed $p\geq1$, there exist constants $C$ and $\mathfrak C$, depending only on $p,c_1,c_2$ and the fixed kernel/noise constants, such that for all $d\geq\mathfrak C$,
\begin{equation}
K\left( \sqrt{2}\sigma \bar\varepsilon_{upp} / 6 \right)
\leq C d^{p-1}\log d.
\end{equation}
Consequently, $K(\sqrt{2}\sigma \bar\varepsilon_{upp}/6)\le n\bar\varepsilon_{upp}^2$ for all sufficiently large $d$.
\proof
Set $r^2=2\sigma^2\bar\varepsilon_{upp}^2/36$.  By (\ref{eqn:def_upp_for_lower_bound_inner}), $r^2\geq\mu_p$.  Lemma \ref{lemma:monotone_of_eigenvalues_of_inner_product_kernels} therefore implies that every degree $k\ge p$ is absent from the sum defining $K(\sqrt{2}\sigma\bar\varepsilon_{upp}/6)$.  Hence
\[
K\left(\sqrt{2}\sigma\bar\varepsilon_{upp}/6\right)
=\frac12\sum_{k:\mu_k>r^2}N(d,k)\log\frac{\mu_k}{r^2}
\leq \frac12\sum_{k=0}^{p-1}N(d,k)\log\frac{\mu_k}{r^2}.
\]
For fixed $p$, Lemma \ref{lemma:inner_mendelson_point_control_assist_2} gives $N(d,k)\le C d^k$ for $k\le p-1$, while $\mu_k/r^2\le C d^{p-k}$ because $r^2\asymp\bar\varepsilon_{upp}^2\asymp\mu_p\asymp d^{-p}$.  Therefore each logarithm is $O(\log d)$ and
\[
K\left(\sqrt{2}\sigma\bar\varepsilon_{upp}/6\right)
\le C\sum_{k=0}^{p-1}d^k\log d
\le C d^{p-1}\log d.
\]
On the other hand, $n\bar\varepsilon_{upp}^2\asymp d^{2p}d^{-p}=d^p$.  Since $d^{p-1}\log d=o(d^p)$, the last assertion follows. \hfill $\square$
\end{claim}
Therefore, by Lemma \ref{lemma_entropy_of_RKHS} and Claim \ref{claim_1}, for all sufficiently large $d$,
\[
V_2(\sqrt2\sigma\bar\varepsilon_{upp},\calB)
\le K(\sqrt2\sigma\bar\varepsilon_{upp}/6)+\log 2
\le C d^{p-1}\log d+\log 2
\le n\bar\varepsilon_{upp}^2.
\]
Under the generalized-inverse convention for the entropy critical radius in Remark \ref{remark:classical_entropy_radii}, this shows $\bar\varepsilon_n^2\le\bar\varepsilon_{upp}^2$.
\hfill $\square$
\vspace{10pt}

\section{Properties of the inner product kernels}\label{app:eigenvalues_of_kernels}
\subsection{Mercer decomposition of the inner product kernels on the sphere}\label{append:intro_sphere_eigenvalues}
For inner product kernels on the sphere, Mercer's decomposition (\ref{eqn:mercer_decomp}) can be expressed in the basis of spherical harmonics \cite{scholkopf2002learning, smola2000regularization}. This allows for the eigenvalues of such kernels to be computed. In this subsection, we will briefly review the Mercer decomposition corresponding to inner product kernels on the sphere. See \cite{gallier2009notes, Bietti_on_2019} for references.
Let $\rho_{\calX}$ be the
uniform measure on $\mathbb S^{d}$, 
and let's assume that $K^{\inner}$ is an {\it inner product kernel} defined on $\mathbb S^{d}$, that is , there exists a function $\Phi: [-1,1] \to \mathbb{R}$, such that for any $\boldsymbol{x}, \boldsymbol{x}^\prime \in \mathbb S^{d}$, we have $K^{\inner}(\boldsymbol{x}, \boldsymbol{x}^\prime) = \Phi(\left\langle \boldsymbol{x}, \boldsymbol{x}^\prime \right\rangle)$.
Similar to (\ref{eqn:mercer_decomp}), Mercer's decomposition for the inner product kernel ${K}^{\inner}$ is given in the basis of spherical
harmonics :
\begin{equation}\label{spherical_decomposition_of_inner_product_kernel}
\begin{aligned}
{K}^{\inner}(\boldsymbol{x},\boldsymbol{x}^\prime) = \sum_{k=0}^{\infty} \mu_{k} \sum_{j=1}^{N(d, k)} Y_{k, j}(\boldsymbol{x}) Y_{k, j}\left(\boldsymbol{x}^\prime\right),
\end{aligned}
\end{equation}
where $Y_{k, j}$ for $j=1, \cdots, N(d, k)$ are spherical harmonic polynomials of degree $k$, $\mu_{k}$'s are the eigenvalues of   $K^{\inner}$ with multiplicity $N(d, k)$, where $N(d, 0)=1$, and $N(d, k)
     = \frac{2k+d-1}{k} \cdot \frac{(k+d-2)!}{(d-1)!(k-1)!}$ for any $k = 1, \cdots$. 
By known results on spherical harmonics, the eigenvalues $\mu_{k}$'s have the following explicit expression \cite{Bietti_on_2019}:
\begin{equation}\label{eqn:explicit_eigenvalues_inner_product_kernel}
\begin{aligned}
\mu_k=\frac{\omega_{d-1}}{\omega_{d}} \int_{-1}^1 \Phi(t) P_k(t)\left(1-t^2\right)^{(d-2) / 2} ~\mathsf{d} t,
\end{aligned}
\end{equation}
where $P_k$ is the $k$-th Legendre polynomial in dimension $d+1$, $\omega_{d}
$ denotes the surface of the sphere $\mathbb{S}^{d}$.

\section{Supplementary proofs of Theorem \ref{theorem:restate_norm_diff}}
\subsection{An elementary lemma}
\begin{lemma}\label{lemma:sup_lemma1_for_lemma:bound_empirical_and_expected_mandelson_complexities}
Let $\varepsilon_n$ be the population critical radius defined in \eqref{eqn:def_population_mendelson_complexity}. Then we have
\begin{itemize}
    \item[i)] For any $\varepsilon$ satisfying $R_{K}(\varepsilon)\geq \frac{\varepsilon^{2}}{2e\sigma}$, we have $\varepsilon_{n}\geq \varepsilon$.
    \item[ii)] For any $\varepsilon$ satisfying $R_{K}(\varepsilon)\leq \frac{\varepsilon^{2}}{2e\sigma}$, we have $\varepsilon_{n}\leq \varepsilon$.
\end{itemize}
Similarly, let $\widehat{\varepsilon}_n$ be the empirical critical radius defined in \eqref{eqn:def_empirical_mendelson_complexity}. Then we have
\begin{itemize}
    \item[i)] For any $\varepsilon>0$, the inequality $\varepsilon\leq \widehat{\varepsilon}_{n}$ holds if the following event occurs:
    \begin{align}
    \Omega_{1}(\varepsilon)=\left\{\omega ~\big\vert~ \widehat{\mathcal{R}}_{{K}} \left( \varepsilon\right) \geq  \frac{ \varepsilon^2}{2 e \sigma} \right\}.
    \end{align}
    \item[ii)] For any $\varepsilon>0$, the inequality $\widehat{\varepsilon}_{n}\leq \varepsilon$ holds if the following event occurs:
    \begin{align}
        \Omega_{2}(\varepsilon)=\left\{\omega ~\big\vert~ \widehat{\mathcal{R}}_{{K}} \left( \varepsilon\right) \leq \frac{ \varepsilon^2}{2 e \sigma} \right\}.
    \end{align}
\end{itemize}
\end{lemma}
\proof It is clear that $R_{K}(\varepsilon)/\varepsilon=\left(\sum_{i}\frac{\lambda_{i}}{\varepsilon^{2}}\wedge 1\right)^{1/2}$ is a non-increasing function and $\frac{\varepsilon}{2e\sigma}$ is a strictly increasing function.  
If $R_{K}(\varepsilon)\geq \frac{\varepsilon^{2}}{2e\sigma}$, for any $\delta<\varepsilon$, we have
\begin{align}
   \frac{R_{K}(\delta)}{\delta}\geq \frac{R_{K}(\varepsilon)}{\varepsilon}\geq\frac{\varepsilon}{2e\sigma}>\frac{\delta}{2e\sigma}.
\end{align}
Thus, we have $\varepsilon_{n}\geq \varepsilon$.
If $R_{K}(\varepsilon)\leq \frac{\varepsilon^{2}}{2e\sigma}$, for any $\delta>\varepsilon$, we have
\begin{align}
   \frac{R_{K}(\delta)}{\delta}\leq \frac{R_{K}(\varepsilon)}{\varepsilon}\leq \frac{\varepsilon}{2e\sigma}<\frac{\delta}{2e\sigma}.
\end{align}
Thus, we have $\varepsilon_{n}\leq \varepsilon$.
The empirical version can be proved similarly.
\hfill $\square$
\begin{remark}\label{remark_how_to_use_mendelson_complexity}
The Lemma \ref{lemma:sup_lemma1_for_lemma:bound_empirical_and_expected_mandelson_complexities} provides us an easy way to bound the Mendelson complexity $\varepsilon_{n}$ and the empirical Mendelson complexity $\widehat{\varepsilon}_{n}$. For example, if we can find $\varepsilon_{low}$ and $\varepsilon_{upp}$ satisfying that
\begin{align}
    R_{K}(\varepsilon_{low})\geq \frac{\varepsilon_{low}^{2}}{2e\sigma}~\mbox{ and }~    R_{K}(\varepsilon_{upp})\leq \frac{\varepsilon_{upp}^{2}}{2e\sigma},
\end{align}
then we have $\varepsilon_{low}\leq \varepsilon_n\leq \varepsilon_{upp}$.
\end{remark}
\subsection{Detailed proofs of the Lemmas \ref{thm:empirical_loss}, \ref{corollary:norm_diff}, \ref{lemma:bound_empirical_and_expected_mandelson_complexities} and \ref{lemma: empirical_loss_bound}}\label{appendix_thm_3.3_lemmas}
The purpose of these proofs is to track the constants in Lemmas \ref{thm:empirical_loss}, \ref{corollary:norm_diff}, \ref{lemma:bound_empirical_and_expected_mandelson_complexities}, and \ref{lemma: empirical_loss_bound}; some constants depend on fixed model parameters such as $\kappa$ and $\sigma$. We included them here for self-contained.  
\subsubsection{Proof of Lemma \ref{thm:empirical_loss}}
\proof
From (\ref{ntk:f:flow}) we have
\begin{equation}
\begin{aligned}
f_t(\bm{X})-f_{\star}(\bm{X})
&= \left(\mathbf{I}-e^{-\frac{1}{n}t K(\bm{X}, \bm{X})}\right)\bm{y} -f_{\star}(\bm{X})\\
&=-e^{-\frac{1}{n}t K(\bm{X}, \bm{X})}f_{\star}(\bm{X}) +\left(\mathbf{I}-e^{-\frac{1}{n}t K(\bm{X}, \bm{X})}\right) (\bm{y} -f_{\star}(\bm{X})).
\end{aligned}
\end{equation}
Let $\frac{1}{n} K(\bm{X}, \bm{X})=U \Sigma U^\tau$, where $U$ is an orthogonal matrix, and $\Sigma=diag\{ \widehat\lambda_{1}, \cdots, \widehat \lambda_{n} \}$.
Let $g_t=U^\tau f_t(\bm{X})$, $g^*=U^\tau f_{\star}(\bm{X})$, and $\boldsymbol{e} = \bm{y} -f_{\star}(\bm{X})$, then we have
\begin{equation}\label{eqn:66}
\begin{aligned}
\left\|f_{t} - f_{\star} \right\|_n^2
&= \frac{1}{n} \left\|g_{t} - g^* \right\|^2=  \frac{1}{n} \left\|-e^{-t \Sigma}g^* +\left(\mathbf{I}-e^{-t \Sigma}\right) U^\tau \boldsymbol{e} \right\|^2\\
&\leq \frac{2}{n} \left\| e^{-t \Sigma} g^* \right\|^2 + \frac{2}{n} \left\| \left(\mathbf{I}-e^{-t \Sigma}\right) U^\tau \boldsymbol{e} \right\|^2:= \mathbf{B}_t^2 + \mathbf{V}_t.
\end{aligned}
\end{equation}
We then bound the terms $\mathbf{B}_t^2$ and $\mathbf{V}_t$ based on the proof of Theorem 1 in \cite{raskutti2014early}. We need the following two lemmas:
\begin{lemma}\label{lemma_B_t:thm:empirical_loss}
    For any $t > 0$, we have
\begin{equation}
\begin{aligned}
\mathbf{B}_t^2 \leq \frac{1}{t}.
\end{aligned}
\end{equation}
\proof Deferred to the end of this subsection.
\end{lemma}
On the nondegenerate empirical-Gram event, recall that $\widehat{T}=(\widehat{\varepsilon}_{n})^{-2}$ where $\widehat{\varepsilon}_{n}$ is the empirical Mendelson complexity defined by \eqref{eqn:def_empirical_mendelson_complexity}.
\begin{lemma}\label{lemma_V_t:thm:empirical_loss}
    There exists a positive constant $C$ depending at most on the fixed noise scale, 
    such that for $\widehat T=(\widehat{\varepsilon}_{n})^{-2}$, we have
\begin{equation}
\begin{aligned}
\mathbf{V}_{\widehat T} \leq \frac{\widehat{\varepsilon}_n^2}{e^2},
\end{aligned}
\end{equation}
with conditional probability at least $1-2\exp \left(-C n \widehat{\varepsilon}_n^2\right)$ given the design, where the probability is over the noise term $\boldsymbol{e}$. 
\proof Deferred to the end of this subsection.
\end{lemma}
From the above lemmas, when $t=\widehat T$ ( which is $(\widehat{\varepsilon}_{n})^{-2}$ ), there exist positive constants $C_2$ and $C_3$, depending at most on the fixed noise scale, 
such that we have
\begin{equation}
\left\|f_{\widehat T} - f_{\star} \right\|_n^2
\leq \widehat{\varepsilon}_n^2 + \frac{\widehat{\varepsilon}_n^2}{e^2} \leq (1+e^{-2})\widehat{\varepsilon}_n^2,
\end{equation}
with conditional probability at least $1-C_2 \exp \left(-C_3 n \widehat{\varepsilon}_n^2\right)$ given the design, on the nondegenerate empirical-Gram event.
\hfill $\square$
\vspace{10pt}
\noindent \proofname ~of Lemma \ref{lemma_B_t:thm:empirical_loss}: 
For any $t, x \geq 0$, we have the following inequality:
\begin{equation}\label{eqn:inequality_lemma_B_t:thm:empirical_loss}
\begin{aligned}
e^{-t x} \leq \frac{1}{t x} \quad (t,x>0),\qquad (1 \wedge t x) / 2 \leq 1-e^{-t x} \leq 1 \wedge t x.
\end{aligned}
\end{equation}
Define
\begin{equation}\label{eqn:63_phi}
\begin{aligned}
\Phi_{\bm{X}}^{\top} = (\phi_i(x_j))_{i\geq 1, j=1, \cdots, n} 
\end{aligned}
\end{equation}
Similarly, we define a (diagonal) linear operator $D: \ell^2 \rightarrow \ell^2$ with entries $[D]_{j k}=\lambda_j \delta_{jk}$. Then we have $f_{\star}(\bm{X})=\Phi_{\bm{X}} D^{1 / 2} a$ for some sequence $a \in \ell^2$.
By Mercer's decomposition, we have
\begin{equation}\label{eqn:63.5_phi}
\begin{aligned}
\frac{1}{n}K(\bm{X}, \bm{X})=U \Sigma U^\tau=\frac{1}{n} \Phi_{\bm{X}} D \Phi_{\bm{X}}^\tau, 
\end{aligned}
\end{equation}
and hence there exists an operator
$\Psi: \mathbb{R}^n \mapsto \ell^2$ such that
\begin{equation}\label{eqn:64_phi}
\begin{aligned}
\frac{1}{\sqrt{n}} \Phi_{\bm{X}} D^{1 / 2}=U \Sigma^{1 / 2} \Psi^* \text { and } \Psi^* \circ \Psi=P_{\operatorname{ran}(\Sigma)} .
\end{aligned}
\end{equation}
Denote
\begin{equation}\label{eqn:65_phi}
\begin{aligned}
(z_1, \cdots, z_n)^\tau = U^\tau f_{\star}(\bm{X}) = U^\tau \Phi_{\bm{X}} D^{1 / 2} a = \sqrt{n} U^\tau U \Sigma^{1 / 2} \Psi^* a = \sqrt{n} \Sigma^{1 / 2} \Psi^* a,
\end{aligned}
\end{equation}
then from (\ref{eqn:inequality_lemma_B_t:thm:empirical_loss}) we have
\begin{equation}\label{eqn:66_phi}
\begin{aligned}
\mathbf{B}_t^2 
&\leq
\frac{2}{n} \sum_{i:\widehat\lambda_i>0} \frac{[U^\tau f_{\star}(\bm{X})]_i^2}{2t\widehat \lambda_i}
=
\frac{1}{nt} \sum_{i:\widehat\lambda_i>0} \frac{z_i^2}{\widehat \lambda_i}
=
\frac{1}{t} \sum_{i:\widehat\lambda_i>0} [\Psi^* a]_i^2\leq \frac{1}{t}\|a\|_2^2\leq \frac{1}{t}.
\end{aligned}
\end{equation}
\hfill $\square$
\vspace{10pt}
\noindent \proofname ~of Lemma \ref{lemma_V_t:thm:empirical_loss}:  Let $H=\left(\bm{I}-e^{-\widehat T\Sigma}\right)$ and $P=\sqrt{\frac{2}{n}}H$. Then, $\bm{V}_{\widehat T}=\boldsymbol{e}^{\tau}UP^{2}U^{\tau}\boldsymbol{e}
\overset{d}{=}
\boldsymbol{e}^{\tau}P^{2}\boldsymbol{e}$,
where $\boldsymbol{e}=(\boldsymbol{e}_1, \cdots, \boldsymbol{e}_n)^\tau$ and $\boldsymbol{e}_{i}=\bm{y}_{i}-f_{\star}(X_{i})\sim N(0,\sigma^{2})$ for any $1\leq i \leq n$. Applying Lemma \ref{lemma:bound_in_Wright} with $A=P^{2}$,  $\delta=\widehat{\varepsilon}_n^2/(2e^2)$, and
$Q=\sum_{i, j=1}^n a_{i j}\boldsymbol{e}_i \boldsymbol{e}_j \overset{d}{=} \bm{V}_{{\widehat T}}$, 
we then have that
\begin{equation}\label{eqn:70_Vt}
    \begin{aligned}
    |Q-\mathbb{E}[Q]| \leq \frac{\widehat{\varepsilon}_n^2}{2e^2}
\end{aligned}
\end{equation}
holds with probability at least $1-2\exp\left(-\mathfrak{c}_{1}\min\left\{\frac{ \widehat{\varepsilon}_n^2}{2e^2\sigma^2\|A\|_{op}}, \frac{\widehat{\varepsilon}_n^4}{4e^4\sigma^4\|A\|^{2}_{F}}\right\}\right)$ where $\mathfrak{c}_1>0$ is a universal constant, and the probability is conditional on the design and over the noise term $\boldsymbol{e}$.
It is easy to verify that $\|H\|_{op}\leq 1$ , $\|A\|_{op} \leq \frac{2}{n}$ and 
\begin{align}
tr(H)\overset{(\ref{eqn:inequality_lemma_B_t:thm:empirical_loss})}{\leq}
        \sum_{j}\left(1\wedge {\widehat T}\widehat{\lambda}_{j}\right)
        =n\widehat T\left(\widehat{\mathcal{R}}_{{K}}(1 / \sqrt{{\widehat T}})\right)^2=  \frac{n\widehat{\varepsilon}_n^2}{4e^2 \sigma^2 }.
\end{align} 
Thus, we have
\begin{equation}
    \begin{aligned}
         \|A\|_{F}^2 &=tr(P^{4})=\frac{4}{n^{2}}tr(H^{4})\leq\frac{4}{n^{2}}tr(H)
        \leq   \frac{\widehat{\varepsilon}_n^2}{e^2 \sigma^2 n},
        \\
         \bbE[Q] &=
         \bbE[\mathbf{V}_{\widehat T}]
         =\frac{2\sigma^2}{n} tr\left(\left(\mathbf{I}-e^{-{\widehat T} \Sigma}\right)^2\right) 
        \leq 
        \frac{2\sigma^2}{n} tr\left(\bm{I}-e^{-{\widehat T}\Sigma}\right)
        \leq
        \frac{\widehat{\varepsilon}_n^2}{2e^2};
    \end{aligned}
\end{equation}
From (\ref{eqn:70_Vt}), we know that there exists an absolute constant $C$, such that we have
\begin{equation}
\begin{aligned}
\mathbf{V}_{\widehat T} \leq \mathbb{E}[Q] + \frac{\widehat{\varepsilon}_n^2}{2e^2} \leq \frac{\widehat{\varepsilon}_n^2}{e^2},
\end{aligned}
\end{equation}
with probability at least $1-2\exp \left(-C n \widehat{\varepsilon}_n^2\right)$.
\hfill $\square$
\subsubsection{Proof of Lemma \ref{corollary:norm_diff}}
\proof
Write $b=\sqrt{\kappa}$ and $c_*=\sqrt{2}/(2e\sigma)$. The RKHS reproducing property gives $\|f\|_\infty\le b$ for every $f\in\mathcal B$. By the population complexity comparison and the critical-radius identity,
\[
 Q_n(\varepsilon_n;\mathcal B)\le\sqrt{2}R_K(\varepsilon_n)=c_*\varepsilon_n^2.
\]
Since $\mathcal B$ is star-shaped, $Q_n(r;\mathcal B)/r$ is nonincreasing. Choose a fixed $A\ge\max\{1,bc_*\}$. Then
\[
 Q_n(A\varepsilon_n;\mathcal B)\le A Q_n(\varepsilon_n;\mathcal B)
 \le Ac_*\varepsilon_n^2\le (A\varepsilon_n)^2/b.
\]
Applying Lemma~\ref{lemma:wainwright14_1} to $\mathcal B$ with $t=A\varepsilon_n$ yields, simultaneously for every $f\in\mathcal B$,
\[
 \big|\|f\|_n^2-\|f\|_{L^2}^2\big|
 \le\tfrac12\|f\|_{L^2}^2+A^2\varepsilon_n^2
\]
with probability at least $1-c_1\exp(-c_2nA^2\varepsilon_n^2/\kappa)$. On this same event, for any fixed $M>0$ and $g\in M\mathcal B$, apply the preceding inequality to $f=g/M$ and multiply by $M^2$ to obtain
\[
 \big|\|g\|_n^2-\|g\|_{L^2}^2\big|
 \le\tfrac12\|g\|_{L^2}^2+M^2A^2\varepsilon_n^2.
\]
The event is independent of $M$, so no additional union bound is required. Constants may depend on the fixed parameters $\kappa$ and $\sigma$.
\hfill $\square$
\subsubsection{Proof of Lemma \ref{lemma:bound_empirical_and_expected_mandelson_complexities}}
\proof
Before we start the proof, we need the following three lemmas.
\begin{lemma}\label{lemma:relation_between_Z_and_hat_Z}
Suppose that $w_{i}$ are i.i.d. Rademacher random variables independent of $x_{i}$ and let 
\begin{align*}
   \hat Z_n(w, t):=&\sup _{\substack{g \in {\mathcal B} \\\|g\|_{n} \leq t}}\left|\frac{1}{n} \sum_{i=1}^n w_i g\left(x_i\right)\right| \mbox{~and~}
   Z_n(w, t):=\mathbb{E}_{x_{1},\ldots,x_{n}\,\stackrel{\mathrm{i.i.d.}}{\sim}\,\mu}\left[\sup _{\substack{g \in {\mathcal B}  \\\|g\|_{L^2} \leq t}}\left|\frac{1}{n} \sum_{i=1}^n w_i g\left(x_i\right)\right|\right]
\end{align*}
 where $\|g\|^{2}_{n}=\frac{1}{n}\sum_{j}g(x_{j})^{2}$ and $\mathcal B = \left\{  g \in \mathcal H \mid  \|g\|_{\mathcal H} \leq 1  \right\}$.
For any $c>0$, let $C_0$ denote the fixed-parameter constant in Lemma \ref{corollary:norm_diff} (with the kernel-diagonal bound absorbed into it), and define the event 
\begin{equation}
 \Omega_{3}(c)= \left\{   \begin{aligned}\label{eqn:lemma:relation_between_Z_and_hat_Z}
\widehat{Z}_n(w, c\varepsilon_n)
 &\leq \frac{3}{2} \mathcal Q_n(L_0 \max\{c, 1\}\varepsilon_n) + C \varepsilon_n^2,\\
 \widehat{Z}_n(w, c\varepsilon_n)
  &\geq a_c\left\{\frac{1}{2} \mathcal Q_n(\varepsilon_n)-\eta\varepsilon_n^2\right\},\qquad
  a_c:=\min\left\{1,\frac{c}{\sqrt{3/2+C_0}}\right\}
\end{aligned}\right\}
\end{equation}
where $C>0$ may depend on the fixed kernel envelope and $\eta>0$ may be any fixed constant. For every fixed $c,\eta>0$, this event occurs with probability at least $1-C_{c,\eta}\exp\{-c_{c,\eta}n\varepsilon_n^2\}$.
\proof Deferred to the end of this subsection.
\end{lemma}
\begin{lemma}\label{lemma:talagrand_inequ_rademacher}
Suppose that $w_{i}$ are i.i.d. Rademacher random variables independent of $x_{i}$ and let
\begin{align*}
    \widehat Q_{n}(t)=\mathbb E_{w} [\widehat Z_{n}(w,t)].
\end{align*}
    There exist absolute constants $C_4$, $C_5$, such that for any
     $t, t_0 >0$, the event
\begin{equation}\label{eqn:concentration_ledoux}
\Omega_{4}(t, t_0)=\left\{\left|\widehat{Z}_n(w, t)-\widehat{\mathcal{Q}}_n(t)\right| \leq t_0\right\}
\end{equation}
occurs with probability at least $1-C_4 \exp \left\{-C_5\frac{n t_0^2}{ t^2}\right\}$.
\proof Deferred to the end of this subsection.
\end{lemma}
\begin{lemma}\label{lemma:eqn:example_7_of_kol}
Let $S(t)=\sum_j\min\{\lambda_j,t^2\}$ and
$\kappa=\sup_x K(x,x)$. At any radius $t>0$ satisfying
$nS(t)\ge\kappa$, one has
\[
 \frac12\mathcal R_K(t)\le\mathcal Q_n(t)
 \le\sqrt2\,\mathcal R_K(t).
\]
For the empirical localized complexities, at every $t>0$, almost surely,
\begin{equation}\label{eqn:example_7_of_kol}
 \frac1{\sqrt3}\widehat{\mathcal R}_K(t)
 \le\widehat{\mathcal Q}_n(t)
 \le\sqrt2\,\widehat{\mathcal R}_K(t).
\end{equation}
The population lower comparison is used below only at fixed multiples
of $\varepsilon_n$, where $nS(t)\ge\kappa$ for all sufficiently large $n$.
\proof The empirical assertion is proved in Appendix
\ref{appendix:mendelson_complexity}. The population lower comparison
under the displayed nondegeneracy condition is established at the end
of this subsection; its upper comparison follows from the standard
ellipsoid bound.\hfill $\square$
\end{lemma}
Set $c_{\mathrm{cmp}}=1/2$, the constant in the population lower comparison of Lemma \ref{lemma:eqn:example_7_of_kol}. The symbol $c_{\mathrm{cmp}}$ is distinct from the tail constant $C_3$ used in the preceding empirical-loss proof.
Thanks to Lemma \ref{lemma:sup_lemma1_for_lemma:bound_empirical_and_expected_mandelson_complexities} and Remark \ref{remark_how_to_use_mendelson_complexity}, we only need to prove that, there exist positive constants $C_1$ and $C_2$, depending at most on the fixed kernel envelope and noise scale, 
such that the event 
\begin{equation}\label{eqn:41_lemma:bound_empirical_and_expected_mandelson_complexities}
 \Omega_{1}(C_1 \varepsilon_n) \cap \Omega_{2}(C_2 \varepsilon_n)=   \begin{aligned}
\left\{ \omega ~\big\vert~
\widehat{\mathcal{R}}_{{K}} \left(C_1 \varepsilon_n\right) \geq \frac{C_1^2 \varepsilon_n^2}{2 e \sigma} \mbox{~~and~ }
\widehat{\mathcal{R}}_{{K}} \left(C_2 \varepsilon_n\right) \leq \frac{C_2^2 \varepsilon_n^2}{2 e \sigma}
\right\}
    \end{aligned}
\end{equation}
occurs with high probability.
For each fixed $C>0$, the fixed-point equation and $n\varepsilon_n^2\to\infty$ imply $nS(C\varepsilon_n)\to\infty$. Thus the population lower comparison in Lemma \ref{lemma:eqn:example_7_of_kol} applies at every fixed trial radius used below, for all sufficiently large $n$.
For any fixed constant $C_2 \geq 1$, 
conditioning on the event $\Omega_{4}(C_2 \varepsilon_n, \sqrt{2}C_2 \varepsilon_n^2) \cap \Omega_{3}(C_2)$
we have
\begin{equation}\label{eqn:long_inequa_define_kappa}
    \begin{aligned}
\widehat{\mathcal{R}}_{{K}} \left(C_2 \varepsilon_n\right) 
&\leq \frac{1}{c_{\mathrm{cmp}}} \widehat{\mathcal{Q}}_n(C_2 \varepsilon_n) \quad ((\ref{eqn:example_7_of_kol}))\\
&\leq 
\frac{1}{c_{\mathrm{cmp}}} \widehat{Z}_n(w, C_2 \varepsilon_n) + \frac{\sqrt{2}C_2}{c_{\mathrm{cmp}}} \varepsilon_n^2 \quad
\left((\ref{eqn:concentration_ledoux}), \text{ let } t=C_2 \varepsilon_n, t_0 = \sqrt{2}C_2 \varepsilon_n^2 \right)
\\
&\leq 
\frac{3}{2c_{\mathrm{cmp}}} {\mathcal{Q}}_n(L_0C_2\varepsilon_n) + \frac{\sqrt{2}C_2+C}{c_{\mathrm{cmp}}} 
 \varepsilon_n^2 \quad (\text{Lemma } \ref{lemma:relation_between_Z_and_hat_Z})\\
&\leq 
\frac{3\sqrt{2}}{2c_{\mathrm{cmp}}} \mathcal{R}_{{K}}(L_0C_2\varepsilon_n) + \frac{\sqrt{2}C_2+C}{c_{\mathrm{cmp}}} \varepsilon_n^2
\quad ((\ref{eqn:example_7_of_kol}))\\
&\leq \frac{C_2^2 \varepsilon_n^2}{2 e \sigma}. \quad (\text{we can choose } C_2 \text{ large enough, see the remarks below.})
    \end{aligned}
\end{equation}
Therefore, there exist positive constants $C_2, C_6$, and $C_7$, depending at most on the fixed model parameters, such that
\begin{equation}
    \begin{aligned}
    \mathbb{P}\left(\Omega_{2}\left(C_2 \varepsilon_n\right)\right) 
    &\geq 
    \mathbb{P}\left(
    \Omega_{4}\left(C_2 \varepsilon_n, \sqrt{2}C_2 \varepsilon_n^2\right) 
\cap 
\Omega_{3}\left(C_2\right)
    \right)\\
    &\geq
    1-C_6\exp \left\{-C_7 n \varepsilon_n^2\right\}.
    \end{aligned}
\end{equation}
For completeness, we spell out the lower comparison, since this is where the additive term in Lemma \ref{corollary:norm_diff} matters. Choose
\[
 \eta=\frac{c_{\mathrm{cmp}}}{8e\sigma}
\]
in Lemma \ref{lemma:relation_between_Z_and_hat_Z}. For a fixed sufficiently small $C_1>0$, put
\[
 a_1:=\frac{C_1}{\sqrt{3/2+C_0}},
 \qquad
 t_0:=\frac{a_1c_{\mathrm{cmp}}}{32e\sigma}\varepsilon_n^2,
\]
where $C_0$ is the constant in that lemma. Taking $C_1<\sqrt{3/2+C_0}$ ensures that $a_{C_1}=a_1$. On the event
$\Omega_3(C_1)\cap\Omega_4(C_1\varepsilon_n,t_0)$, Lemmas
\ref{lemma:eqn:example_7_of_kol}, \ref{lemma:talagrand_inequ_rademacher}, and
\ref{lemma:relation_between_Z_and_hat_Z} imply
\begin{align*}
 \widehat{\mathcal R}_K(C_1\varepsilon_n)
 &\ge \frac1{\sqrt2}\widehat{\mathcal Q}_n(C_1\varepsilon_n)\\
 &\ge \frac1{\sqrt2}\left\{\widehat Z_n(w,C_1\varepsilon_n)-t_0\right\}\\
 &\ge \frac1{\sqrt2}\left[
 a_1\left\{\frac12\mathcal Q_n(\varepsilon_n)-\eta\varepsilon_n^2\right\}-t_0
 \right]\\
 &\ge \frac1{\sqrt2}\left[
 a_1\left\{\frac{c_{\mathrm{cmp}}}{4e\sigma}-\frac{c_{\mathrm{cmp}}}{8e\sigma}\right\}
 -\frac{a_1c_{\mathrm{cmp}}}{32e\sigma}
 \right]\varepsilon_n^2\\
 &=\frac{3a_1c_{\mathrm{cmp}}}{32\sqrt2 e\sigma}\varepsilon_n^2.
\end{align*}
Here we used $\mathcal Q_n(\varepsilon_n)\ge c_{\mathrm{cmp}}\mathcal R_K(\varepsilon_n)
=c_{\mathrm{cmp}}\varepsilon_n^2/(2e\sigma)$. Since $a_1$ is a fixed positive multiple of $C_1$, the last display is strictly larger than
$C_1^2\varepsilon_n^2/(2e\sigma)$ once $C_1$ is chosen sufficiently small. Moreover,
$n t_0^2/(C_1^2\varepsilon_n^2)\asymp n\varepsilon_n^2$, so the two events used above have probability at least
$1-C_8e^{-C_9n\varepsilon_n^2}$. Consequently,
\[
 \mathbb P\left(\Omega_1(C_1\varepsilon_n)\right)
 \ge 1-C_8e^{-C_9n\varepsilon_n^2}.
\]
Together with the upper comparison this proves the desired result. More explicitly, let
$E_{\rm low}=\Omega_3(C_1)\cap\Omega_4(C_1\varepsilon_n,t_0)$ and
$E_{\rm upp}=\Omega_3(C_2)\cap\Omega_4(C_2\varepsilon_n,\sqrt2 C_2\varepsilon_n^2)$.
Both are events on the joint design--Rademacher probability space. The deterministic
trial-radius comparisons show that on $E_{\rm low}\cap E_{\rm upp}$,
$C_1\varepsilon_n\le\widehat\varepsilon_n(X)\le C_2\varepsilon_n$.
Since this last event is measurable with respect to $X$ alone, its failure probability
under the design law equals its failure probability under the joint law. For clarity, the four failure events are bounded separately: the two
$\Omega_3$ failures are each at most $C_{c,\eta}e^{-c_{c,\eta}n\varepsilon_n^2}$,
whereas the two $\Omega_4$ failures are at most
$C_4e^{-C_5nt_0^2/(C_1^2\varepsilon_n^2)}$ and
$C_4e^{-2C_5n\varepsilon_n^2}$, respectively.
Here $t_0=a_1c_{\mathrm{cmp}}\varepsilon_n^2/(32e\sigma)$, so the first
$\Omega_4$ exponent is a positive fixed multiple of
$n\varepsilon_n^2$. The union bound,
without any independence assumption between the four events, gives
$\mathbb P_X\{C_1\varepsilon_n\le\widehat\varepsilon_n\le C_2\varepsilon_n\}
\ge 1-C\exp(-c n\varepsilon_n^2)$ for fixed positive $C,c$.
\hfill $\square$
\begin{remark}
    Here we give a detailed discussion of the last inequality in (\ref{eqn:long_inequa_define_kappa}). Suppose $C_2 \geq 1$.
    Since for any $j \geq 1$, $\min \left\{\lambda_j, L_0^2C_2^2\varepsilon_n^2\right\} \leq \max\{L_0^2C_2^2, 1\} \min \left\{\lambda_j, \varepsilon_n^2\right\}= L_0^2C_2^2 \min \left\{\lambda_j, \varepsilon_n^2\right\}$, we have
    \begin{equation}\label{eqn:C_2_kappa_def_0}
        \mathcal{R}_{{K}}(L_0C_2\varepsilon_n) \leq 
        {L_0C_2}\mathcal{R}_{{K}}(\varepsilon_n).
    \end{equation}
    If $C_{2}$ is sufficiently large such that
    \begin{align}
        \label{eqn:C_2_kappa_def_1} C_2^2 &\geq \frac{ 6 \sqrt{2}L_0C_2}{c_{\mathrm{cmp}}}, \\
        \label{eqn:C_2_kappa_def_2} C_2^2 &\geq 4e\sigma\frac{\sqrt{2}C_2+C}{c_{\mathrm{cmp}}},
    \end{align}
    then we have
    \begin{equation}
        \frac{3\sqrt{2}}{2c_{\mathrm{cmp}}} \mathcal{R}_{{K}}(L_0C_2\varepsilon_n) 
        \overset{(\ref{eqn:C_2_kappa_def_0})}{\leq}
        \frac{3\sqrt{2}L_0C_2}{c_{\mathrm{cmp}}} \mathcal{R}_{{K}}(\varepsilon_n)
        =
        \frac{3\sqrt{2}L_0C_2}{2e\sigma c_{\mathrm{cmp}}} \varepsilon_n^2
    \overset{(\ref{eqn:C_2_kappa_def_1})}{\leq}
        \frac{C_2^2 \varepsilon_n^2}{4 e \sigma},
    \end{equation}
    and
    \begin{equation}
    \frac{\sqrt{2}C_2+C}{c_{\mathrm{cmp}}} \varepsilon_n^2
    \overset{(\ref{eqn:C_2_kappa_def_2})}{\leq}
        \frac{C_2^2 \varepsilon_n^2}{4 e \sigma}.
    \end{equation}
\end{remark}
\vspace{4mm}
\noindent \proofname ~of Lemma \ref{lemma:relation_between_Z_and_hat_Z}:
Let $C_0$ be the constant in Lemma \ref{corollary:norm_diff}, which may depend on the fixed kernel envelope and noise scale. On an event $\mathcal E_0$, depending only on $X=(x_1,\ldots,x_n)$, with probability at least $1-Ce^{-cn\varepsilon_n^2}$,
\begin{equation}\label{eq:new-uniform-norm-comparison}
 \big|\|g\|_n^2-\|g\|_{L^2}^2\big|
 \le \frac12\|g\|_{L^2}^2+C_0\varepsilon_n^2,
 \qquad g\in\mathcal B.
\end{equation}
Consequently, on $\mathcal E_0$,
\[
 \{g\in\mathcal B:\|g\|_n\le c\varepsilon_n\}
 \subset
 \{g\in\mathcal B:\|g\|_{L^2}\le L_0\max\{c,1\}\varepsilon_n\},
\]
where $L_0:=\sqrt{2(1+C_0)}$. This follows by rearranging \eqref{eq:new-uniform-norm-comparison}.
For $t>0$ define
\[
 H_n^w(t):=\sup_{\substack{g\in\mathcal B\\ \|g\|_{L^2}\le t}}
 \left|\frac1n\sum_{i=1}^nw_i g(x_i)\right|.
\]
Regard $(x_i,w_i)$ as i.i.d. observations from $\rho_{\mathcal X}\otimes\mathrm{Rad}$. For $h_g(x,w)=wg(x)$, multiplication by $w$ preserves both the envelope and the $L^2$ norm, and
\[
 \mathbb E_{X,w}H_n^w(t)=\mathcal Q_n(t).
\]
The envelope of $h_g$ is $b=\max\{1,\sqrt\kappa\}$, not necessarily one. Apply Lemma \ref{lemma:thm3_mas} to $h_g/b$, and multiply the resulting bounds by $b$. The upper and lower deviations for the normalized averages are then respectively bounded by
\[
 2t\sqrt{2u/n}+66.5bu/n,\qquad
 t\sqrt{10.8u/n}+88.9bu/n,
\]
with failure probability at most $e^{-u}$ for each tail. At the finitely many trial radii $t=a\varepsilon_n$ used below, choose $u=c_{a,\eta}n\varepsilon_n^2$ with a sufficiently small positive fixed coefficient. This yields an additive error bounded by $\eta\varepsilon_n^2$ in the lower tail and by $C\varepsilon_n^2$ in the upper tail, where the constants may depend on the fixed envelope $b$. Thus, for every fixed $\eta>0$ and the values of $t$ used below,
\begin{align}
 H_n^w(t)&\le \frac32\mathcal Q_n(t)+C\varepsilon_n^2,
 \label{eq:new-upper-signed-process}\\
 H_n^w(\varepsilon_n)&\ge \frac12\mathcal Q_n(\varepsilon_n)-\eta\varepsilon_n^2,
 \label{eq:new-lower-signed-process}
\end{align}
with probability at least $1-C_{c,\eta}\exp\{-c_{c,\eta}n\varepsilon_n^2\}$. For the lower inequality the deviation parameter in Lemma \ref{lemma:thm3_mas} is chosen as a sufficiently small fixed multiple of $n\varepsilon_n^2$; hence $\eta$ can be chosen arbitrarily small while retaining an exponential probability of the same order. This product-space argument avoids any identification of a signed process with an unsigned empirical process.
The set inclusion above and \eqref{eq:new-upper-signed-process}, with $t=L_0\max\{c,1\}\varepsilon_n$, prove the upper comparison.
For the lower comparison, the additive $C_0\varepsilon_n^2$ term in \eqref{eq:new-uniform-norm-comparison} cannot be absorbed uniformly for arbitrary $c$. Instead use that the RKHS ball is star-shaped. Put
\[
 a_c:=\min\left\{1,\frac{c}{\sqrt{3/2+C_0}}\right\}.
\]
If $g\in\mathcal B$ and $\|g\|_{L^2}\le\varepsilon_n$, then $a_cg\in\mathcal B$ and, on $\mathcal E_0$,
\[
 \|a_cg\|_n^2
 \le \frac32a_c^2\|g\|_{L^2}^2+C_0a_c^2\varepsilon_n^2
 \le c^2\varepsilon_n^2.
\]
Hence, pointwise in $(X,w)$,
\[
 \widehat Z_n(w,c\varepsilon_n)\ge a_c H_n^w(\varepsilon_n).
\]
Combining this with \eqref{eq:new-lower-signed-process} proves the lower comparison. \hfill $\square$
\vspace{10pt}
\noindent \proofname ~of Lemma \ref{lemma:talagrand_inequ_rademacher}:
Fix the design $X=(x_1,\ldots,x_n)$. The set
$\{(g(x_1),\ldots,g(x_n)):g\in\mathcal B,\ \|g\|_n\le t\}$
is symmetric and contained in the Euclidean ball of radius $\sqrt n\,t$.
Consequently the function
\[
F_X(w)=\frac{\sqrt n}{t}\widehat Z_n(w,t)
=\sup_{\substack{g\in\mathcal B\\\|g\|_n\le t}}
\frac{1}{\sqrt n\,t}\sum_{i=1}^n w_i g(x_i)
\]
is convex and 1-Lipschitz on $\mathbb R^n$.
Applying Lemma \ref{lemma:talagrand_thn_9} conditionally on $X$
to the independent Rademacher vector $w$ gives
\[
\mathbb P_w\left(\left|\widehat Z_n(w,t)-
\mathbb E_w[\widehat Z_n(w,t)\mid X]\right|>t_0\,\middle|\,X\right)
\le C_4\exp\left(-C_5\frac{nt_0^2}{t^2}\right).
\]
Since the right-hand side is deterministic, integrating over $X$
proves the claimed bound. The assertion concerns only the empirical
localized process; no concentration statement for the population-localized
process $Z_n(w,t)$ is required in the subsequent critical-radius comparison.
\hfill $\square$
\paragraph{Population complexity lower bound at the critical radius.}
The population comparison $\mathcal Q_n(t)\ge\mathcal R_K(t)/2$
under $nS(t)\ge\kappa$ is proved directly in Appendix
\ref{appendix:mendelson_complexity}, Lemma~\ref{theorem:1}.
At $t=\varepsilon_n$, the fixed-point identity gives
$nS(\varepsilon_n)=(n\varepsilon_n^2)^2/(4e^2\sigma^2)$,
so this condition holds for all sufficiently large $n$ whenever
$n\varepsilon_n^2\to\infty$. The same conclusion at each fixed
positive multiple of $\varepsilon_n$ follows from monotonicity of
$S$ and the elementary inequality $S(ct)\ge\min\{1,c^2\}S(t)$.
\subsubsection{Proof of Lemma \ref{lemma: empirical_loss_bound}}
\proof
The empirical prediction-error assertion follows from Lemma~\ref{thm:empirical_loss} and the empirical/population critical-radius comparison. We prove the RKHS-norm assertion separately. Condition on the design $\boldsymbol X$ and write $\widehat T=\widehat\varepsilon_n^{-2}$. Let $\widehat K=K(\boldsymbol X,\boldsymbol X)/n=U\Sigma U^\tau$, where $\Sigma=\operatorname{diag}(\widehat\lambda_j)$. All inverse-eigenvalue expressions below are set to zero at $\widehat\lambda_j=0$.
By linearity of kernel gradient flow in the responses, write $f_{\widehat T}=f_{\widehat T}^{\rm sig}+\xi_{\widehat T}$, with signal and noise fitted separately. The signal operator is a positive contraction on the RKHS, so $\|f_{\widehat T}^{\rm sig}\|_{\mathcal H}\le\|f_\star\|_{\mathcal H}\le1$. For the noise part, the representer identity gives
\[
\|\xi_{\widehat T}\|_{\mathcal H}^2
=\frac1n\boldsymbol e^\tau U P U^\tau\boldsymbol e,
\qquad P=\operatorname{diag}\!\left(\frac{(1-e^{-\widehat T\widehat\lambda_j})^2}{\widehat\lambda_j}\right).
\]
Since $(1-e^{-u})^2/u\le\min\{u,1\}$ for $u\ge0$, each positive empirical eigenvalue satisfies $P_{jj}\le\widehat T^2\min\{\widehat\lambda_j,\widehat T^{-1}\}$, while $P_{jj}=0$ on the nullspace. Also $\|P\|_{\rm op}\le\widehat T$. Hence, by the definition of the empirical complexity and its positive critical-root identity,
\[
\mathbb E[\|\xi_{\widehat T}\|_{\mathcal H}^2\mid\boldsymbol X]
=\frac{\sigma^2}{n}\operatorname{tr}P
\le\sigma^2\widehat T^2\widehat{\mathcal R}_K(\widehat T^{-1/2})^2
=\frac1{4e^2}.
\]
Conditional on $\boldsymbol X$, $\|\xi_{\widehat T}\|_{\mathcal H}$ is a Lipschitz function of a standard Gaussian vector with Lipschitz constant at most $\sigma\sqrt{\widehat T/n}$. Gaussian concentration and Jensen's inequality imply
\[
\mathbb P\bigl(\|\xi_{\widehat T}\|_{\mathcal H}>1\mid\boldsymbol X\bigr)
\le\exp\!\left\{-\frac{(1-1/(2e))^2n}{2\sigma^2\widehat T}\right\}.
\]
The stopping time is measurable with respect to the design, so this conditional Gaussian argument requires no union bound over stopping times. On the complementary event, $\|f_{\widehat T}-f_\star\|_{\mathcal H}\le\|f_{\widehat T}^{\rm sig}\|_{\mathcal H}+\|\xi_{\widehat T}\|_{\mathcal H}+\|f_\star\|_{\mathcal H}\le3$. Finally, on the critical-radius comparison event $\widehat\varepsilon_n\ge c_0\varepsilon_n$, the conditional failure probability is bounded by $\exp(-c n\varepsilon_n^2)$. The complement of that comparison event has probability at most $C\exp(-c'n\varepsilon_n^2)$ by Lemma~\ref{lemma:bound_empirical_and_expected_mandelson_complexities}. Taking expectation over $\boldsymbol X$ and combining these events yields the stated unconditional bound. \hfill $\square$

\section{Assisting Lemmas}\label{appendix:assist_lemmas}
\subsection{Local Rademacher complexity}\label{appendix:mendelson_complexity}
Suppose that $K$ is a kernel defined on $\mathcal X\subset \mathbb R^{d+1}$ and $\mathcal H$ is the RKHS associated to the kernel $K$. Let 
\begin{align}
K(x,y)=\sum_{j}\lambda_{j}\phi_{j}(x)\phi_{j}(y)
\end{align}
be the Mercer's decomposition of $K$ where $\lambda_{1}\geq \lambda_{2}\geq ...\geq 0$ is non-increasing non-negative real numbers and $\{\phi_{j}\}$ are orthonormal functions in $L^{2}(\mathcal X, \rho_{\calX})$.
Let $\Phi(x)^{\tau}=(\sqrt{\lambda_{1}}\phi_{1}(x),\sqrt{\lambda_{2}}\phi_{2}(x),....)$. Then we introduce an isometric coordinate map $i: \ell^{2}(J_+) \to \mathcal H$, where $J_+=\{j:\lambda_j>0\}$ given by
\begin{align}
    a=(a_{1},a_{2},\cdots) \mapsto a^{\tau}\Phi=\sum_{j}a_{j}\sqrt{\lambda_{j}}\phi_{j}(x).
\end{align}
\subsubsection{Population version}
We introduce the following quantities:
\begin{equation}\label{eqn:def_population_mendelson}
  \begin{aligned}
    R_{K}(t)=\left[\frac{1}{n}\sum_{j\in J_+} \min\{\lambda_{j},t^{2}\}\right]^{1/2}, \quad Q_{n}(t)=\mathbb E_{w} [Z_{n}(w,t)]
\end{aligned}  
\end{equation}
where $Z_n(w, t):=\mathbb{E}_{x_{1},\ldots,x_n\sim \rho_{\calX}}\left[\sup _{\substack{g \in {\mathcal B}  \\\|g\|_{L^2} \leq t}}\left|\frac{1}{n} \sum_{i=1}^n w_i g\left(x_i\right)\right|\right]$, $w_{i}$ are i.i.d. Rademacher random variables independent of $x_{i}$, and $\mathcal B = \left\{  g \in \mathcal H \mid  \|g\|_{\mathcal H} \leq 1  \right\}$.
The following comparison is proved directly below. In particular, the
population lower bound is stated with its explicit nondegeneracy
condition, rather than relying on an unverified all-kernel version of
an external first-moment inequality.
\begin{lemma}
\label{theorem:1} 
For any $t>0$, we have
\begin{align}
    Q_{n}(t)\leq \sqrt{2} R_{K}(t).
\end{align}
Furthermore, put $S(t)=\sum_{j\in J_+}\min\{\lambda_j,t^2\}$ and
$\kappa=\sup_{x\in\mathcal X}K(x,x)<\infty$. Whenever
$nS(t)\ge\kappa$, one has
\begin{align}
    Q_n(t)\ge\frac12R_K(t).
\end{align}
\end{lemma}
\proof    Let $T(t)=\sup _{\substack{g \in {\mathcal B} \\\|g\|_{L^2} \leq t}}\left|\sum_{i=1}^n w_i g\left(x_i\right)\right|$.  We need the following two lemmas:
\begin{lemma}[Lemma 42 in \cite{Mendelson_Geometric_2002}]\label{lemma:42}
For any $t>0$, we have
\begin{equation}\label{rademacher_popu_lemma_235}
    \begin{aligned}
      n^{2}R_{K}^{2}(t) \leq \mathbb E_{w,x_{1},..,x_{n}}T(t)^{2}\leq 2n^{2}R_{K}^{2}(t).
   \end{aligned}
\end{equation}
\proof  
   Denote $\calF(t)=\left\{f \in {\mathcal B}  \mid \|f\|_{L^{2}}\leq t\right\}$. Since there exists $\beta\in \ell^{2}(J_+)$ such that $f(x)=\beta^{\tau}\Phi(x)$, we know that $\calF(t)=i\left(\left\{ \beta ~\vert ~\sum \beta_{j}^{2}\leq 1 \text{ and } \sum \beta_{j}^{2}\lambda_{j}\leq t^{2} \right\}\right)$.
   Throughout the spectral-coordinate arguments below, sums and reciprocal eigenvalues are taken only over strictly positive eigenvalues; null-space coordinates are omitted. Any inverse of a non-invertible empirical kernel operator is understood as its Moore--Penrose inverse on its range.
For any $s>0$, let $\calE(s)=\left\{~\beta ~\mid ~ \sum_{j}\beta^{2}_{j}\mu_{j}\leq s~\right\}$ where  $\mu_{j}:= \mu_{j}(t)=\left(\min\{1,t^{2}/\lambda_{j}\}\right)^{-1} = \max\{1, \lambda_j / t^2\}$. Then \[i\left(\calE(1)\right)\subset \calF(t) \subset i\left(\calE(2)\right).\] Thus we have
   \begin{align*}
       \mathbb E T(t)^{2}\leq \bbE \sup_{
       \sum \beta^{2}_{j}\mu_{j}\leq 2
       }
       \langle\beta, \sum_{i=1}^{n}w_{i}\Phi(x_{i})\rangle^{2}=2n \sum_{i\in J_+} \frac{\lambda_i}{\mu_i} = 2n \sum_{i\in J_+} \min \left\{\lambda_i, t^2 \right\}=2n^{2}R_{K}(t)^{2}.
   \end{align*}
   For the reverse inequality, the ellipsoid $\mathcal E(1)$ is contained in
   $i^{-1}(\mathcal F(t))$. By Hilbert-space duality, for
   $U=\sum_{i=1}^n w_i\Phi(x_i)$,
   \[
   \sup_{\beta\in\mathcal E(1)}|\langle\beta,U\rangle|^2
     =\sum_{j\in J_+}\frac{|U_j|^2}{\mu_j(t)}.
   \]
   Since $\mathbb E_{w,x}|U_j|^2=n\lambda_j$, Tonelli's theorem
   (all terms are nonnegative) yields
   \[
   \mathbb E_{w,x}T(t)^2\geq
   n\sum_{j\in J_+}\frac{\lambda_j}{\mu_j(t)}
   =n\sum_{j\in J_+}\min\{\lambda_j,t^2\}
   =n^2R_K(t)^2.
   \]
   The argument includes finite-rank kernels and zero eigenvalues by
   restricting the coordinate sum to $J_+$.
   \hfill $\square$
\end{lemma}
For the first-moment lower bound, let
\[
 v_t(x)=\bigl(\sqrt{\min\{\lambda_j,t^2\}}\,\phi_j(x)\bigr)_{j\in J_+},
 \qquad U=\sum_{i=1}^n w_iv_t(x_i).
\]
For every $\|b\|_2\le1$, the function
$g_b(x)=\langle b,v_t(x)\rangle$ belongs to $\mathcal B$
and has $L^2$ norm at most $t$. Thus $T(t)\ge\|U\|_2$.
The filtered feature map obeys $\|v_t(x)\|_2^2\le K(x,x)\le\kappa$.
Independence and the Hilbert-space Rademacher fourth-moment bound yield
\[
 \mathbb E\|U\|_2^2=nS(t),\qquad
 \mathbb E\|U\|_2^4\le n\kappa S(t)+3n^2S(t)^2.
\]
If $nS(t)\ge\kappa$, the last expression is at most $4n^2S(t)^2$.
Moment interpolation gives
\[
 \mathbb E\|U\|_2\ge
 \frac{(\mathbb E\|U\|_2^2)^{3/2}}
 {(\mathbb E\|U\|_2^4)^{1/2}}
 \ge\tfrac12\sqrt{nS(t)}.
\]
Consequently $Q_n(t)=n^{-1}\mathbb ET(t)\ge\tfrac12R_K(t)$.
The upper bound follows from Jensen's inequality and
\eqref{rademacher_popu_lemma_235}. This proof applies equally to
finite-rank and infinite-rank kernels. \hfill $\square$
\subsubsection{Empirical version}
Suppose that we have $n$ i.i.d. design points $x_i\sim\rho_{\calX}$, $1\le i\le n$. Let $\widehat \lambda_{1}\geq ...\geq \widehat \lambda_{n}$ be the eigenvalues of $\frac{1}{n}K(X,X)$. We then introduce the empirical version of the aforementioned quantities:
\begin{equation}\label{eqn:def_empirical_mendelson}
    \begin{aligned}
    \widehat R_{K}(t)=\left[\frac{1}{n}\sum_{j} \min\{\widehat \lambda_{j},t^{2}\}\right]^{1/2}, \quad \widehat Q_{n}(t)=\mathbb E_{w} [\widehat Z_{n}(w,t)]
\end{aligned}
\end{equation}
where $\hat Z_n(w, t):=sup _{\substack{g \in {\mathcal B} \\\|g\|_{n} \leq t}}\left|\frac{1}{n} \sum_{i=1}^n w_i g\left(x_i\right)\right|$, $w_{i}$ are i.i.d. Rademacher random variables independent of $x_{i}$, $\|g\|^{2}_{n}=\frac{1}{n}\sum_{j}g(x_{j})^{2}$, and $\mathcal B = \left\{  g \in \mathcal H \mid  \|g\|_{\mathcal H} \leq 1  \right\}$.
\begin{lemma}\label{theorem:example7_Koltchinskii}
For any $t>0$, we have
\begin{align}\label{eqn:example7_Koltchinskii}
    \widehat Q_{n}(t)\leq \sqrt{2}\widehat R_{K}(t).
\end{align}
Furthermore, for every $t>0$, the following lower bound holds (also for singular empirical Gram matrices):
\begin{align}\label{eqn:example7_Koltchinskii_lower}
    \widehat Q_{n}(t) \geq 3^{-1/2}\widehat R_{K}(t).
\end{align}
\end{lemma}
\begin{remark}
We notice that \cite{Koltchinskii_Local_2006} claimed that (\ref{eqn:example7_Koltchinskii}) and (\ref{eqn:example7_Koltchinskii_lower}) held without proving it, and \cite{bartlett2005local} only gave the proof of the upper bound of $\widehat Q_{n}(t)$.
\end{remark}
\proof Introduce the operator $\hat{C}_n$ on $\mathcal{H}$ defined by
$$
\left(\hat{C}_n f\right)(x)=\frac{1}{n} \sum_{i=1}^n f\left(X_i\right) K\left(X_i, x\right),
$$
then we have the following lemma:
\begin{lemma}\label{lemma:eigenvalue_of_cn}
    The strictly positive eigenvalues of $\hat C_n$, counted with multiplicity, coincide with those of $n^{-1}K(X,X)$. Both operators have rank $r\leq n$; all their other spectral values are zero (with the multiplicity determined by their respective ambient spaces).
    \proof Deferred to the end of this subsection.
\end{lemma}
Note that $\hat C_{n}$ is an operator with rank $\leq n$. Thus its null space is the orthogonal complement of its finite-dimensional range (and is infinite-dimensional only when $\mathcal H$ is infinite-dimensional). Write $r=\operatorname{rank}(\hat C_n)$ and index an orthonormal eigenbasis by $J$, placing the $r$ positive eigenvalues first. 
Let $\left(\hat \phi_i\right)_{i \in J}$ be an orthonormal basis of $\mathcal{H}$ of eigen-functions of $\hat{C}_n$ (such that $\hat \phi_i$ is associated with $\hat{\lambda}_i$ ). The zero-eigenvalue eigenvectors are not unique, and $\hat\lambda_i=0$ for $i\in J$ beyond the positive spectral range.
Here $J=\{1,\ldots,\dim\mathcal H\}$ in the finite-dimensional case and $J=\mathbb N$ otherwise; zero eigenvalues are appended as needed. For any $f \in \mathcal{H}$, we have the following decomposition:
\begin{equation}\label{eqn:men_empirical_decomp}
    \begin{aligned}
f
&=\sum_{i \in J} \left\langle f, \hat \phi_i\right\rangle_{\mathcal H} \hat \phi_i \\
\hat{C}_n f
&=\sum_{i \in J} \hat{\lambda}_i\left\langle f, \hat \phi_i\right\rangle_{\mathcal H} \hat \phi_i 
    \end{aligned}
\end{equation}
We need the following three lemmas:
\begin{lemma}\label{lemma:men_empirical_upper}
    For any $t>0$, we have
\begin{align}\label{eqn:lemma:men_empirical_upper}
   \left(\bbE_w \widehat Z_{n}(w,t)\right)^{2} \leq \frac{2}{n}\sum_{j} \min\{\widehat \lambda_{j},t^{2}\}.
\end{align}
\proof Deferred to the end of this subsection.
\end{lemma}
\begin{lemma}\label{lemma:men_empirical_lower}
For any $t>0$, the empirical localized complexity satisfies
\[
\widehat Q_n(t)\ge 3^{-1/2}\widehat R_K(t).
\]
\proof The Rademacher fourth-moment argument below establishes this bound directly; it does not require a lower bound growing linearly in $t$.\hfill $\square$
\end{lemma}
Note that $\widehat Q_{n}(t)=\bbE_w \widehat Z_{n}(w,t)$. 
For any $t>0$, we have
\begin{align}
  \widehat Q_{n}(t)
  \overset{(\ref{eqn:lemma:men_empirical_upper})}{\leq} \sqrt{2}\widehat R_{K}(t). 
\end{align}
Furthermore, for every $t>0$, the ellipsoid
$\{b:\sum_j b_j^2/\min(1,t^2/\widehat\lambda_j)\le1\}$
(with zero-eigenvalue directions omitted) is contained in the empirical
localized RKHS ball. Consequently, conditional on the design,
\[
\widehat Z_n(w,t)\ge \frac1n\left\|\sum_{i=1}^n w_i v_t(X_i)\right\|_{\ell^2},
\quad
v_t(X_i)_j=\frac{\widehat\phi_j(X_i)}{\sqrt{\widehat\mu_j}},
\quad \widehat\mu_j=\max\{1,\widehat\lambda_j/t^2\}.
\]
For any deterministic Hilbert-space vectors $v_i$, the Rademacher sum
$S=\sum_i w_i v_i$ satisfies
$\mathbb E_w\|S\|^4\le 3(\mathbb E_w\|S\|^2)^2$.
By interpolation of its first, second and fourth moments,
$\mathbb E_w\|S\|\ge (\mathbb E_w\|S\|^2)^{1/2}/\sqrt3$.
Using the empirical covariance eigenvalue identity gives
\[
\widehat Q_n(t)\ge\frac1{\sqrt3}\left(\frac1n
\sum_j\min\{\widehat\lambda_j,t^2\}\right)^{1/2}
=\frac1{\sqrt3}\widehat R_K(t).
\]
This proof does not require an unrestricted lower bound proportional to $t$.
\hfill $\square$
\noindent \proofname ~of Lemma \ref{lemma:eigenvalue_of_cn}:
Define the bounded sampling operator $S:\mathcal H\to\mathbb R^n$ by
$(Sf)_i=n^{-1/2}f(X_i)$. Its adjoint is
$S^*v=n^{-1/2}\sum_{i=1}^n v_iK(X_i,\cdot)$.
Consequently,
\[
 S^*S=\hat C_n,\qquad SS^*=n^{-1}K(X,X).
\]
For every $\lambda>0$, $S$ maps the $\lambda$-eigenspace
of $S^*S$ injectively into the $\lambda$-eigenspace of $SS^*$:
indeed $\|Sf\|^2=\lambda\|f\|_{\mathcal H}^2$, and
$SS^*(Sf)=\lambda Sf$. Conversely, $S^*$ maps the latter
eigenspace injectively into the former, since
$\|S^*v\|^2=\lambda\|v\|^2$ and
$S^*S(S^*v)=\lambda S^*v$. These maps are inverses up to the
factor $\lambda$, so the positive eigenspaces have equal dimensions.
The zero eigenspaces are $\ker S$ and $\ker S^*$, respectively,
and need not have the same dimension. This proves the claim,
including the cases of finite-dimensional $\mathcal H$ and singular
Gram matrices.\hfill $\square$
\vspace{10pt}
\noindent \proofname ~of Lemma \ref{lemma:men_empirical_upper}:
Fix an integer $0 \leq h \leq r$, where $r=\operatorname{rank}(\hat C_n)$. All inverse eigenvalues below have indices $1\leq j\leq h$, hence are strictly positive.
For any $f \in \mathcal{H}$ satisfying $\|f\|_{\mathcal H} \leq 1$ and 
\begin{equation}\label{eqn:men_empirical_14}
    \begin{aligned}
\|f\|_n^2 &=\left\langle f, \hat{C}_n f\right\rangle_{\mathcal H}
=\sum_{i \in J} \hat{\lambda}_i\left\langle f, \hat \phi_i\right\rangle_{\mathcal H}^2 \leq t^2, 
    \end{aligned}
\end{equation}
we have
\begin{equation}
\begin{aligned}
\sum_{i=1}^n w_i f\left(X_i\right)
= &\ \left\langle f, \sum_{i=1}^n w_i K\left(X_i, \cdot\right)\right\rangle_{\mathcal H} \\
= &\ \left\langle\sum_{j=1}^h \sqrt{\hat{\lambda}_j}\left\langle f, \hat \phi_j\right\rangle_{\mathcal H} \hat \phi_j, \sum_{j=1}^h \frac{1}{\sqrt{\hat{\lambda}_j}}\left\langle\sum_{i=1}^n w_i K\left(X_i, \cdot\right), \hat \phi_j\right\rangle_{\mathcal H} \hat \phi_j\right\rangle_{\mathcal H} \\
&\ +\left\langle f, \sum_{j>h}\left\langle\sum_{i=1}^n w_i K\left(X_i, \cdot\right), \hat \phi_j\right\rangle_{\mathcal H} \hat \phi_j\right\rangle_{\mathcal H}\\
\leq &\ \sqrt{t^2 \cdot \sum_{j=1}^h \frac{1}{\hat{\lambda}_j}\left\langle\sum_{i=1}^n w_i K\left(X_i, \cdot\right), \hat\phi_j\right\rangle_{\mathcal H}^2} 
 +\sqrt{ 1^2 \cdot \sum_{j>h}\left\langle\sum_{i=1}^n w_i K\left(X_i, \cdot\right), \hat\phi_j\right\rangle_{\mathcal H}^2}.
\end{aligned}
\end{equation}
By Jensen's inequality that $\mathbb E \sqrt{Z} \leq \sqrt{\mathbb E Z}$, we have
\begin{equation}\label{eqn:men_empirical_15}
\begin{aligned}
n\bbE_w \widehat Z_{n}(w,t) = &\
\mathbb{E}_w \sup _{\substack{f \in {\mathcal B} \\\|f\|_n \leq t}}\left|\sum_{i=1}^n w_i f\left(x_i\right)\right|\\
\leq &\ t \mathbb{E}_w \sqrt{\sum_{j=1}^h \frac{1}{\hat{\lambda}_j}\left\langle\sum_{i=1}^n w_i K\left(X_i, \cdot\right), \hat\phi_j\right\rangle_{\mathcal H}^2}  +\mathbb{E}_w \sqrt{ \sum_{j>h}\left\langle\sum_{i=1}^n w_i K\left(X_i, \cdot\right), \hat\phi_j\right\rangle_{\mathcal H}^2}\\
\leq &\ t \sqrt{\sum_{j=1}^h \frac{1}{\hat{\lambda}_j} \mathbb{E}_w \left\langle\sum_{i=1}^n w_i K\left(X_i, \cdot\right), \hat\phi_j\right\rangle_{\mathcal H}^2}  +\sqrt{ \sum_{j>h} \mathbb{E}_w \left\langle\sum_{i=1}^n w_i K\left(X_i, \cdot\right), \hat\phi_j\right\rangle_{\mathcal H}^2}\\
= &\ t \sqrt{\sum_{j=1}^h \frac{1}{\hat{\lambda}_j} \mathbf{H}_j}  +\sqrt{ \sum_{j>h} \mathbf{H}_j},
\end{aligned}
\end{equation}
where $\mathbf{H}_j = \mathbb{E}_w \left\langle\sum_{i=1}^n w_i K\left(X_i, \cdot\right), \hat\phi_j\right\rangle_{\mathcal H}^2$.
For any $j \geq 1$, we have
\begin{equation}\label{eqn:men_empirical_16}
\begin{aligned}
\mathbf{H}_j 
& =\mathbb{E}_w \sum_{i, \ell=1}^n w_i w_{\ell}\left\langle K\left(X_i, \cdot\right), \hat \phi_j\right\rangle_{\mathcal H} \left\langle K\left(X_l, \cdot\right), \hat\phi_j\right\rangle_{\mathcal H} \\
& = \sum_{i=1}^n\left\langle K\left(X_i, \cdot\right), \hat \phi_j\right\rangle_{\mathcal H}^2
= \sum_{i=1}^n \hat\phi_j(X_i)^2 = n \| \hat\phi_j \|_n^2
\\
& \overset{(\ref{eqn:men_empirical_14})}{=} n\left\langle \hat \phi_j, \hat C_n \hat \phi_j\right\rangle_{\mathcal H}\\
& = n\hat \lambda_j.
\end{aligned}
\end{equation}
Combining (\ref{eqn:men_empirical_15}) and (\ref{eqn:men_empirical_16}) gives the upper bound for $\widehat Q_n(t)$:
\begin{equation}
\begin{aligned}
\left(\bbE_w \widehat Z_{n}(w,t)\right)^{2}
& \leq \frac{1}{n^2} \min_{0\leq h\leq r} \left( t\sqrt{nh}  + \sqrt{n \sum_{j>h} \hat \lambda_j  }\right)^2\\
& \leq \frac{2}{n} \min_{0\leq h\leq r} \left( t^2 h  +  \sum_{j>h} \hat \lambda_j  \right)\\
& = \frac{2}{n} \sum_{j=1}^{n} \min\{\hat \lambda_j, t^2  \}.
\end{aligned}
\end{equation}
\hfill $\square$
\vspace{10pt}
\noindent \proofname ~of Lemma \ref{lemma:men_empirical_lower}:
Choose a complete orthonormal eigenbasis $\{\hat \phi_j:j\in J\}$, including a basis of $\ker(\hat C_n)$. Its coefficient map is an isometric isomorphism $\widehat i:  \ell^{2}(J) \to \mathcal H$ given by
\begin{align}
    b=(b_{1}, b_{2},\cdots) \mapsto \sum_{j}b_{j}\hat\phi_{j}(\cdot).
\end{align}
Denote $\calF(t)=\{f \in {\mathcal B}  \mid \|f\|_{n}\leq t\}$. Since there exists $b\in \ell^{2}(J)$ such that $f(x)=\sum_{j}b_{j}\hat\phi_{j}(x)$, we know that $\calF(t)=\widehat i\left(\{\sum b_{j}^{2}\leq 1 \text{ and } \sum b_{j}^{2} \hat \lambda_{j}\leq t^{2} \}\right)$.
Let $\widehat{\mathcal{E}} = \left\{b \mid \sum_{i\in J} \hat \mu_i b_i^2 \leq 1\right\}$, where $\hat \mu_i=\max\{1,\hat\lambda_i/t^2\}$, including the convention $\hat\mu_i=1$ when $\hat\lambda_i=0$.
Then $\widehat i\left(\widehat{\mathcal{E}}\right)\subset \calF(t)$. The following second-moment calculation is an identity for the ellipsoid supremum, but a first-moment lower bound additionally requires the Rademacher fourth-moment argument established above. Thus we have
\begin{equation}
\begin{aligned}
\bbE_w\left(\widehat Z_{n}(w,t)\right)^{2}
&\geq
\mathbb{E}_w \sup _{f \in \widehat i(\widehat{\mathcal{E}})}\left|\frac{1}{n}\sum_{i=1}^n w_i f\left(X_i\right)\right|^2\\
&=
\mathbb{E}_w \sup _{\{\hat \beta_i\mid \sum\hat\mu_i \hat\beta_i^2 \leq 1\}}\left|\frac{1}{n}\sum_{i=1}^n w_i \sum_{j\in J} \hat\beta_j \hat \phi_j\left(X_i\right)\right|^2\\
&= \frac{1}{n^2} \cdot \mathbb{E}_w \sum_{j\in J} \frac{1}{\hat\mu_j} \left(\sum_{i=1}^n w_i \hat\phi_j\left(X_i\right)\right)^2\\
& \overset{(\ref{eqn:men_empirical_16})}{=} \frac{1}{n} \cdot  \sum_{j\in J} \frac{\hat \lambda_j}{\hat\mu_j}
 = \frac{1}{n} \sum_{i=1}^{n} \min \left\{\hat \lambda_i, t^2 \right\}.
\end{aligned}
\end{equation}
\hfill $\square$
\subsection{Concentration bounds}
The following quadratic concentration inequality is introduced in \cite{Wright_bound_1973}.
\begin{lemma}[The main theorem in \cite{Wright_bound_1973}]\label{lemma:bound_in_Wright}
Suppose we have $\boldsymbol{e}_1, \cdots, \boldsymbol{e}_n \sim_{i.i.d.} N(0, \sigma^2)$. 
For any real symmetric matrix $A=\left\{a_{i j}\right\}_{i, j=1}^n$, denote $Q=\sum_{i, j=1}^n a_{i j}\boldsymbol{e}_i \boldsymbol{e}_j$, then we have 
\begin{equation}
\begin{aligned}
\mathbb{P}[|Q-\mathbb{E}[Q]| \geq \delta] \leq 2\exp \left(-c \min \left\{\frac{\delta}{\sigma^2\|A\|_{\mathrm{op}}}, \frac{\delta^2}{\sigma^4\|A\|_{\mathrm{F}}^2}\right\}\right) \quad \text { for all } \delta>0,
\end{aligned}
\end{equation}
where $c>0$ is a universal numerical constant, and $\left(\|A\|_{\mathrm{op}},\|A\|_{\mathrm{F}}\right)$ are (respectively) the operator and Frobenius norms of the matrix $A$.
\end{lemma}
\begin{proof}
Diagonalize $A=V\operatorname{diag}(a_1,\ldots,a_n)V^\tau$ and write
$Q-\mathbb EQ=\sigma^2\sum_j a_j(Z_j^2-1)$ for independent standard
Gaussian $Z_j$. For $|2s\sigma^2a_j|<1$,
\[
\log\mathbb E\exp\{s\sigma^2a_j(Z_j^2-1)\}
=-s\sigma^2a_j-\tfrac12\log(1-2s\sigma^2a_j).
\]
The elementary bound $-\log(1-u)-u\le u^2/[2(1-|u|)]$
for $|u|<1$ gives
\[
\log\mathbb E e^{s(Q-\mathbb EQ)}
\le \frac{s^2\sigma^4\|A\|_{\mathrm F}^2}
{1-2|s|\sigma^2\|A\|_{\mathrm{op}}}.
\]
Chernoff's inequality, applied to both signs and optimized over
$0<|s|<(2\sigma^2\|A\|_{\mathrm{op}})^{-1}$, yields the stated
bound for a universal numerical $c>0$. The case $A=0$ is immediate.
\end{proof}
\begin{lemma}[Theorem 9 in \cite{tao_concentration_2010}]\label{lemma:talagrand_thn_9}
    Let $w_1, \ldots, w_n$ be independent real random variables with $|w_i|\leq K$ for all $1\leq i\leq n$, where $K>0$. Let $G: \mathbb{R}^n \rightarrow \mathbb{R}$ be a 1-Lipschitz convex function (for the Euclidean norm). Then for any $\delta>0$ one has
\begin{equation}
\begin{aligned}
\mathbb{P}\left(|G(w_1, \ldots, w_n)-\mathbb{E} G(w_1, \ldots, w_n)| \geq \delta \right) \leq C_1 \exp \left(-C_2 \delta^2/K^2\right)
\end{aligned}
\end{equation}
for some absolute constants $C_1, C_2>0$.
\end{lemma}
\begin{lemma}[Theorem 14.1 in \cite{wainwright2019high}]\label{lemma:wainwright14_1}
Let $\mathcal F$ be a star-shaped class with $\sup_{f\in\mathcal F}\|f\|_\infty\le b$, where $b>0$. Suppose that $r>0$ satisfies $Q_n(r;\mathcal F)\le r^2/b$, where $Q_n$ is the population localized Rademacher complexity. Then for every deterministic $t\ge r$, with probability at least $1-c_1\exp(-c_2nt^2/b^2)$,
\[
 \big|\|f\|_n^2-\|f\|_{L^2}^2\big|\le \tfrac12\|f\|_{L^2}^2+t^2
 \qquad\text{for all }f\in\mathcal F.
\]
Here $c_1,c_2>0$ are universal constants. This is the two-sided squared-norm bound (14.5a) of the cited theorem, not the separate norm bound (14.5b).
\end{lemma}
\begin{lemma}[Theorem 3 in \cite{10.1214/aop/1019160263}]\label{lemma:thm3_mas}
Consider $n$ independent random variables $x_1, \cdots, x_n$ sampling from $\rho_{\calX}$ on $\calX$. For any $t >0$, let $\calF$ be some countable family of real-valued measurable functions in $L^2$ spaces, such that $\mathbb E[f(x_1)]=0$, $\|f\|_{L^2} \leq t$, and $\|f\|_{\infty} \leq 1$ for every $f \in \calF$. Let $Z$ denote
$
\sup _{f \in \mathcal{F}}\left|\sum_{i=1}^n f\left(x_i\right)\right|
$.
Let
$
\sigma^2 :=nt^2 \geq \sup _{f \in \calF} \sum_{i=1}^n \operatorname{Var}\left(f\left(x_i\right)\right)
$,
then, for any positive real number $\delta$,
\begin{equation}
\begin{aligned}
\mathbb{P} \left(Z \geq \frac{3}{2} \mathbb{E}[Z]+2\sigma \sqrt{2 \delta}+66.5 \delta \right) \leq \exp \{-\delta\}.
\end{aligned}
\end{equation}
Moreover, one also has
\begin{equation}
\begin{aligned}
\mathbb{P}\left(Z \leq \frac{1}{2} \mathbb{E}[Z]-\sigma \sqrt{10.8 \delta}- 88.9 \delta\right) \leq \exp \{-\delta\}.
\end{aligned}
\end{equation}
The coefficients follow from the original Theorem 3, equations (11)--(12), by taking its auxiliary parameter $\epsilon=1/2$: the upper coefficients are $\sqrt{2\cdot4}=2\sqrt2$ and $2.5+32/(1/2)=66.5$, while the lower coefficients are $\sqrt{2\cdot5.4}=\sqrt{10.8}$ and $2.5+43.2/(1/2)=88.9$. For the absolute supremum use the symmetric class $\mathcal F\cup(-\mathcal F)$; the centered form of the original theorem applies. The variance proxy $nt^2$ upper bounds the actual variance parameter.
\end{lemma}
\begin{remark}
From Corollary 3.6 of \cite{steinwart2012mercer} and Assumption \ref{assu:trace_class}, we know that the RKHS $\calH$ as well as its unit ball $\calB$ are separable. To justify the localized supremum, choose a countable dense subset $D$ of the RKHS unit ball in the RKHS norm and let $D_0$ be the countable set of its rational contractions. For every $g\in\mathcal B$ with $\|g\|_{L^2}\le t$, approximate $g$ in RKHS norm by members of $D$, then contract the approximants by factors in $[0,1]$ tending to one so that their $L^2$ norms are at most $t$. These contractions can be approximated by rational factors from below. The reproducing property implies pointwise convergence at each sampled $x_i$, so the supremum over the localized ball agrees with the supremum over the resulting countable localized subset. Hence the separability requirement in Lemma \ref{lemma:thm3_mas} is satisfied. For $\kappa>1$, apply that lemma to the normalized functions $f/\max\{1,\sqrt\kappa\}$ and rescale the conclusion; the deviation terms then carry the envelope factor described in the proof of Lemma \ref{lemma:relation_between_Z_and_hat_Z}.
\end{remark}

\end{document}